%% file: Main_SEA.tex
\documentclass{article}
\usepackage{authblk}
\usepackage{algorithm}
\usepackage{algorithmic}
\usepackage{microtype}
\usepackage{graphicx}
\usepackage{subfigure}
\usepackage{booktabs} % for professional tables

\usepackage{hyperref}

\newcommand{\alglinelabel}{%
  \addtocounter{ALC@line}{-1}% Reduce line counter by 1
  \refstepcounter{ALC@line}% Increment line counter with reference capability
  \label% Regular \label
}
\makeatletter
\def\theHALC@line{\thealgorithm-\theALC@line}
\def\theHALC@rem{\thealgorithm-\theALC@rem}
\makeatother

\usepackage{amsmath}
\usepackage{amssymb}
\usepackage{mathtools}
\usepackage{amsthm}

\usepackage[capitalize,noabbrev]{cleveref}

\theoremstyle{plain}
\newtheorem{theorem}{Theorem}[section]
\newtheorem{proposition}[theorem]{Proposition}
\newtheorem{lemma}[theorem]{Lemma}
\newtheorem{corollary}[theorem]{Corollary}
\theoremstyle{definition}

\theoremstyle{remark}

\usepackage[textsize=tiny]{todonotes}

\usepackage{Macros} % File with notations of the paper
\usepackage{stmaryrd} % For integer brackets
\usepackage{enumitem} % To cite enumerate
\usepackage{pdfpages} % Add pdf to the document
\usepackage{soul}  % Pour barrer du texte
\usepackage[numbers]{natbib} %  If you rely on the LATEX bibliographic facility, use natbib.sty and icml2023.bst
\usepackage{wrapfig}  % Pour wrapper les figures
\usepackage{xspace}
\usepackage{multicol}  % POur le multicolonne
\usepackage{tikz}
\usetikzlibrary{matrix, positioning}
\usepackage{svg}  % Pour les plots

\usetikzlibrary{arrows}
\usetikzlibrary{decorations.pathreplacing}
\tikzstyle{data}=[rectangle,draw,fill=white]
\tikzstyle{op}=[circle,draw]
\tikzstyle{fleche}=[->,>=stealth',thick,rounded corners=4pt]

\definecolor{drawGreen}{RGB}{124,179,107}
\definecolor{drawRed}{RGB}{191,84,82}
\definecolor{drawBlue}{RGB}{102,142,189}
\definecolor{drawPurple}{RGB}{153,115,164}
\definecolor{drawOrange}{RGB}{213,154,0}
\usepackage{geometry}
\SetSymbolFont{stmry}{bold}{U}{stmry}{b}{n}

\begin{document}

\title{Straight-Through meets Sparse Recovery: the Support Exploration Algorithm}




\author[1,2]{Mimoun Mohamed}
\author[4]{François Malgouyres}
\author[1]{Valentin Emiya}
\author[3]{Caroline Chaux}
%\icmlauthor{François Malgouyres}{imt}
%\icmlauthor{Valentin Emiya}{afflis}
%\icmlauthor{Caroline Chaux}{affipal}

\affil[1]{Aix Marseille Univ, CNRS, LIS, Marseille, France}
\affil[2]{Aix Marseille Univ, CNRS, I2M, Marseille, France}
\affil[3]{CNRS, IPAL, Singapour}
\affil[4]{Institut de Mathématiques de Toulouse ; UMR5219 , Université de Toulouse ; CNRS , UPS IMT F-31062 Toulouse Cedex 9, France}
%\icmlaffiliation{sch}{School of ZZZ, Institute of WWW, Location, Country}

% Mimoun + François? Fr : Mimoun suffit
%\icmlcorrespondingauthor{Mimoun Mohamed}{mimoun.mohamed@lis-lab.fr}

%\icmlkeywords{model selection, support recovery, restricted isometry property, straight-through estimator, sparsity, greedy algorithm}

\vskip 0.3in



\date{}
\maketitle

%%%%%%%%%%%%%%%%%%%%%%%%%%%%%%%%%%
% Ci-dessous, j'ai juste changé le style bibtex et décommenté les acknowledgement

\begin{abstract}
The {\it straight-through estimator} (STE) is commonly used to optimize quantized neural networks, yet its contexts of effective performance are still unclear despite empirical successes.
To make a step forward in this comprehension, we apply STE to a well-understood problem: {\it sparse support recovery}. 
We introduce the {\it Support Exploration Algorithm} (SEA), a novel algorithm promoting sparsity, and we analyze its performance in support recovery (a.k.a. model selection) problems. 
SEA explores more supports than the state-of-the-art, leading to superior performance in experiments, especially when the columns of $A$ are strongly coherent.
The theoretical analysis considers recovery guarantees when the linear measurements matrix $A$ satisfies the {\it Restricted Isometry Property} (RIP).
The sufficient conditions of recovery are comparable but more stringent than those of the state-of-the-art in sparse support recovery. Their significance lies mainly in their applicability to an instance of the STE.
 \end{abstract}

\input{Introduction}

\input{Related_works}

\input{Method}

\input{Theoretical_analysis}

\input{Experimental_analysis}

\input{Discussion}

\input{Acknowledgement}

\bibliography{abbr,bibliography}
\bibliographystyle{plain}

%%%%%%%%%%%%%%%%%%%%%%%%%%%%%%%%%%%%%%%%%%%%%%%%%%%%%%%%%%%%%%%%%%%%%%%%%%%%%%%
%%%%%%%%%%%%%%%%%%%%%%%%%%%%%%%%%%%%%%%%%%%%%%%%%%%%%%%%%%%%%%%%%%%%%%%%%%%%%%%
% APPENDIX
%%%%%%%%%%%%%%%%%%%%%%%%%%%%%%%%%%%%%%%%%%%%%%%%%%%%%%%%%%%%%%%%%%%%%%%%%%%%%%%
%%%%%%%%%%%%%%%%%%%%%%%%%%%%%%%%%%%%%%%%%%%%%%%%%%%%%%%%%%%%%%%%%%%%%%%%%%%%%%%
\newpage
\appendix

\input{Problem_statement}
\input{Algorithms}

\clearpage

\input{RIP_Proof}

\input{DT_experiments}

\input{Deconv_experiment}

\clearpage

\input{ML_experiments}

\clearpage

%%%%%%%%%%%%%%%%%%%%%%%%%%%%%%%%%%%%%%%%%%%%%%%%%%%%%%%%%%%%%%%%%%%%%%%%%%%%%%%
%%%%%%%%%%%%%%%%%%%%%%%%%%%%%%%%%%%%%%%%%%%%%%%%%%%%%%%%%%%%%%%%%%%%%%%%%%%%%%%

\end{document}

%% file: Introduction.tex
\section{Introduction}
\paragraph{Straight-through estimator.}
The use of quantized neural networks spares memory, energy, and computing resources during inference, making them essential for embedding neural networks \cite{yuan2023comprehensive, Sayed2023ASL}. An effective strategy is to learn the quantized weights. Seminal works \cite{courbariaux2015binaryconnect, courbariaux16BNN} rely on full-precision weights $w$ that evolve in the parameter space, while quantized weights $w_q=H_q(w)$ are obtained by applying a piecewise-constant quantization operator $H_q$.

Denoting by $F$ the computational chain from $w_q$ to the loss, the learning procedure $\underset{w}{\minimize} \, F\left(H_q\left(w\right)\right)$ relies on the computational graph:
\begin{center}
\begin{tikzpicture}
\node[data] (X) at (0,0) {$w$};
\node[op] (H) at (1.5,0) {$H_q$};
\node[data] (x) at (3,0) {$w_q$};
\node[op] (F) at (4.5,0) {$F$};
\node[data] (loss) at (6,0) {\textit{loss}};
\draw[fleche] (X) -- (H);
\draw[fleche] (H) -- (x);
\draw[fleche] (x) -- (F);
\draw[fleche] (F) -- (loss);
\end{tikzpicture}
\end{center}
Given a step-size $\eta$, the update of $w$ is performed as $w \leftarrow w - \eta \frac{\partial F}{\partial w_q}|_{w_q}$.
The motivation, as explained in \cite{hinton2012nnml, bengio2013estimating}, is that since $\frac{\partial H_q}{\partial w}|_w$ is either undefined or $0$, we cannot backpropagate using the chain rule $\frac{\partial F\circ H_q}{\partial w}|_w=\frac{\partial F}{\partial w_q}|_{w_q}\frac{\partial H_q}{\partial w}|_{w}$. The STE makes the coarse approximation $\frac{\partial F\circ H_q}{\partial w}|_{w}\approx \frac{\partial F}{\partial w_q}|_{w_q}$ to backpropagate the gradient through the piecewise-constant operator $H_q$. Many subsequent works improve these methods in various aspects \cite{yuan2023comprehensive, Sayed2023ASL}.

Although STE achieves state-of-the-art performance in training quantized weights for neural networks, it is poorly understood and has not been investigated beyond the context of quantization. We introduce an STE principle for sparsification, leading to a novel algorithm named the Support Exploration Algorithm (SEA) and present experimental evidence of its benefits in challenging, coherent settings such as spike deconvolution, as well as in systematic experiments like the phase transition diagram, see \cref{fig:dcv_mass:intro}.
Additionally, we establish theoretical guarantees for the STE-based algorithm.
% Although STE achieves state-of-the-art performance in training quantized weights for neural networks, it is poorly understood and has not been investigated beyond the context of quantization. We introduce an STE principle for sparsification, leading to a novel algorithm named the Support Exploration Algorithm (SEA) and present experimental evidence of its benefits in challenging, coherent settings such as spike deconvolution, see \cref{fig:dcv_mass:intro}~(right), and in systematic experiments like the phase transition diagram, see \cref{fig:dcv_mass:intro}~(left). Additionally, we establish theoretical guarantees for the STE-based algorithm.

%Although STE achieves state-of-the-art for training quantized weights of neural networks, it is poorly understood and has not been explored beyond the context of quantization. We introduce an STE principle for sparsification and provide experimental evidence of its benefits in challenging, coherent settings such as spike deconvolution, see \cref{fig:dcv_mass:intro}, and in systematic experiments such as the phase transition diagram, see \fm{FIGURE 2}. We also establish theoretical guarantees for the STE-based algorithm.
%Lots of ensuing works enhance these methods in many  aspects \cite{yuan2023comprehensive,Sayed2023ASL}. 

\begin{figure}[tbh]
    \centering
    % \hspace{-0.02\linewidth}
            \includegraphics[width=0.35\linewidth]{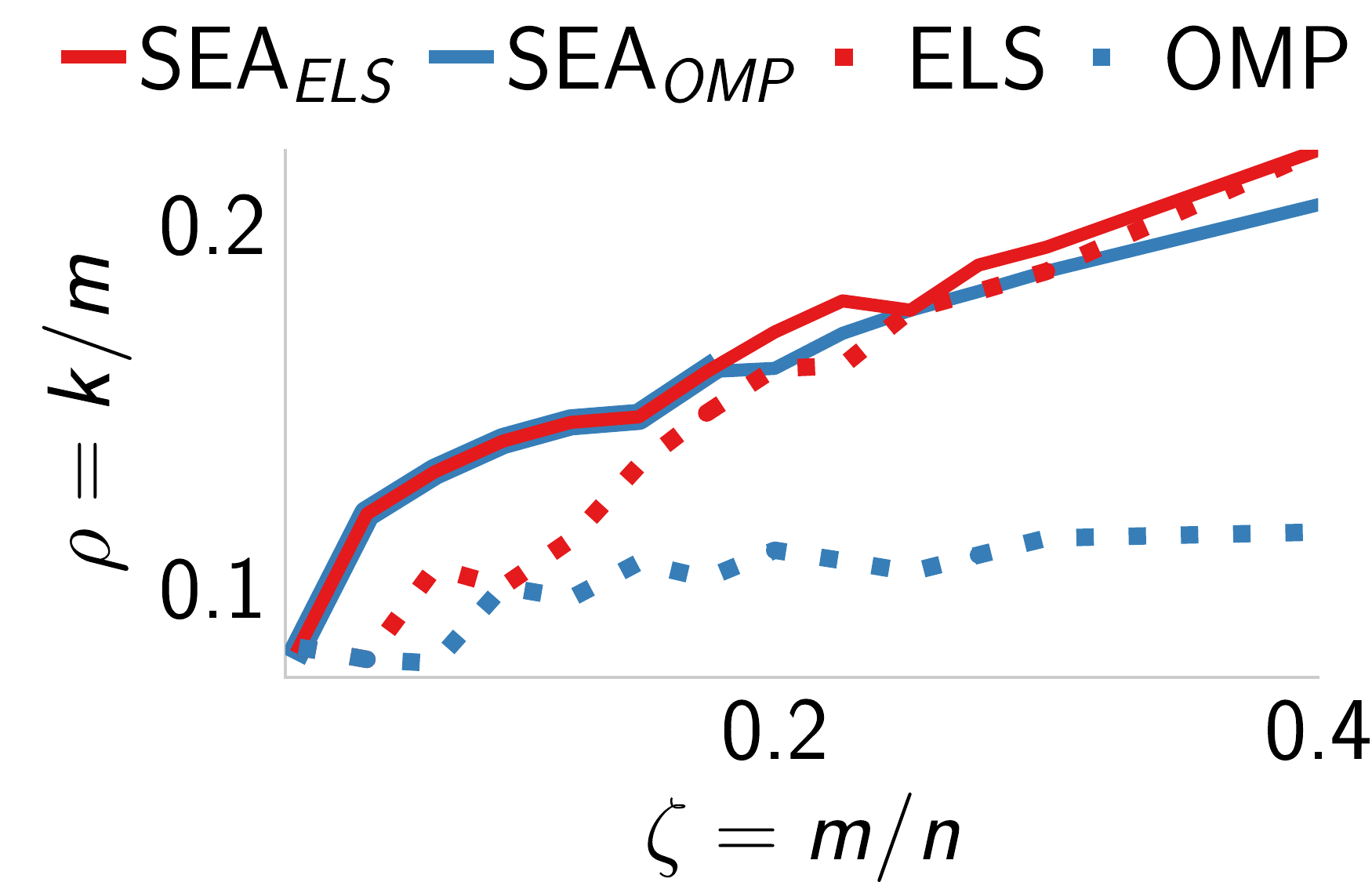}
            \hspace{2cm}
            \includegraphics[width=0.35\linewidth]{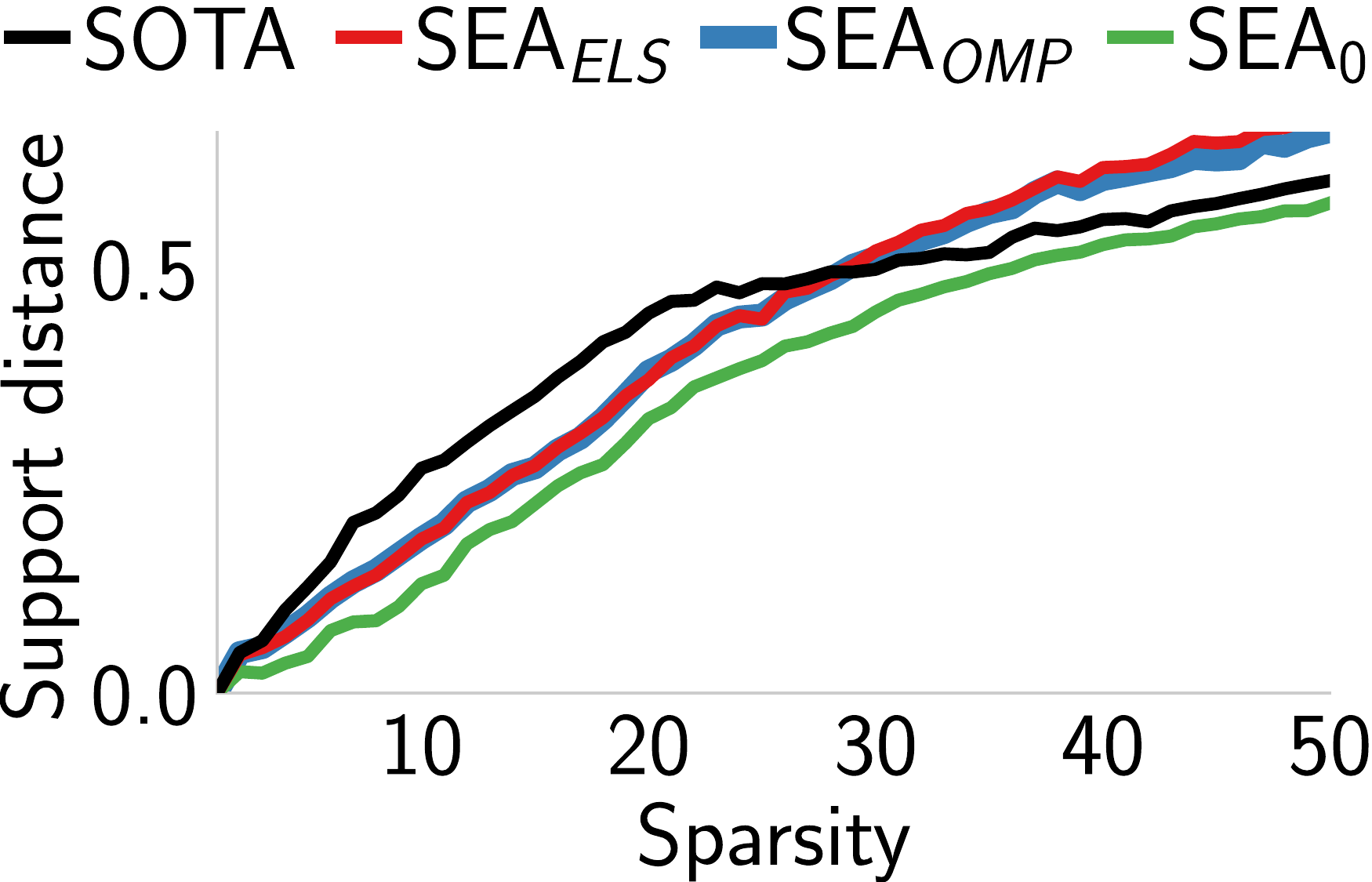}
    \caption{Overview of the main results.
    Left: phase transition diagram showing the recovery limits in dimension $n=500$ while sparsity $k$ and number of observations $m$ varies (the higher, the better, see details in Section~\ref{dt-sec}). Right: spike deconvolution in dimension $m=n=500$ - Average distance between the supports of the solution $\sol$ and the estimations obtained from various algorithms, plotted against the sparsity level $\spars$ (the lower, the better, see details in Section~\ref{deconv-sec}).}
    \label{fig:dcv_mass:intro}
\end{figure}

\paragraph{Sparse support recovery.}

For a sparsity $\spars \in \NN$, we assume $\sol \in \sigspace$ is an unknown sparse vector of unknown support $\suppsol = \SUPP{\sol}$, of sparsity $|\suppsol| \leq \spars$, $\mat \in \matspace $ is a known matrix, and 
\[\obs = \mat \sol + \noise \in \obsspace 
\]
%\begin{equation}\label{datamodel}
%\obs = \mat \sol + \noise \in \obsspace 
%\end{equation}
is a linear observation of $\sol$, contaminated with an additive error/noise $\noise\in\obsspace$.
% \begin{equation}
% \label{datamodel}
%     \obs = \mat \sol + \noise.
% \end{equation}

The support recovery objective\footnote{The adaptation of the article to ``signed support recovery'' is possible and is straightforward. We chose to simplify the presentation and not discuss sign recovery.}, also coined variable or model selection, searches for a support $S$ with cardinality at most $\spars$ such that $S^*\subseteq S$.
We say that {\it an algorithm recovers $S^*$} if it finds such an $S$.

Using a least-square criterion 
\[F\left(\sig\right) = \frac{1}{2}\normdd{\mat\sig - \obs},
\]
a famous model for support recovery is the optimization problem with sparsity constraint:
\begin{equation}
%\underset{\sig \in \sigspace,\, \normz{\sig} \leq \spars}{\minimize} F\left(\sig\right) := \frac{1}{2}\normdd{\mat\sig - \obs}~.
\underset{\sig \in \sigspace}{\minimize} \, F\left(\sig\right) \text{ s.t. } \normz{\sig} \leq \spars,
\label{eq:constrained_sparse_problem}
\end{equation}
where $\|x\|_0$ is the $\ell^0$ pseudo-norm of $x$. Problem \eqref{eq:constrained_sparse_problem} is known to be NP-hard \cite{davis1997adaptive} and to recover the correct support under mild conditions \citep[Chapter 5.2.2]{Elad2010}.

\paragraph{Proposed STE-based approach for sparse recovery.}
Let us define an unconstrained optimization problem which is equivalent to  problem~\eqref{eq:constrained_sparse_problem} and whose structure is compatible with the STE. We set
\begin{equation}
\underset{\explo \in \sigspace}{\minimize} \, F\left(H\left(\explo\right)\right) %= \frac{1}{2}\normdd{\mat H\left(\explo\right) - \obs}
\label{eq:sea_problem}
\end{equation}
where $H$ is the sparsification operator\footnote{In scenarios of interest, the minimization problem \eqref{eq:sparsification_operator} has a unique and easy to compute solution.}
\begin{align}
    H\left(\explo\right)  \in %\underset{\sig \in \sigspace, \, \SUPP{\sig} \subseteq \largest_\spars\left(\explo\right) }{\argmin}\normdd{\mat\sig - \obs}
\underset{\substack{\sig \in \sigspace \\ \SUPP{\sig} \subseteq \largest_\spars\left(\explo\right) }}{\argmin} \frac{1}{2}\normdd{\mat\sig - \obs}.
\label{eq:sparsification_operator}
\end{align}
% which is naturally related to the computational graph:
% \begin{center}
% \begin{tikzpicture}
% \node[data] (X) at (0,0) {$\explo$};
% \node[op] (Hk) at (1.5,0) {$H_k$};
% \node[data] (x) at (3,0) {$\sig$};
% \node[op] (F) at (4.5,0) {$F$};
% \node[data] (loss) at (6,0) {\textit{loss}};
% \draw[fleche] (X) -- (Hk);
% \draw[fleche] (Hk) -- (x);
% \draw[fleche] (x) -- (F);
% \draw[fleche] (F) -- (loss);
% \end{tikzpicture}
% \end{center}
The equivalence between \eqref{eq:constrained_sparse_problem} and \eqref{eq:sea_problem} is established in Appendix~\ref{app:problem_statement}.

Operator $H$ is piece-wise constant and finds the non-zero values in~$\sig$, the sparse support being induced by dense vector~$\explo$. 
Formulation \eqref{eq:sea_problem} has a similar structure as in the quantization case presented above. It uses the suitable loss $F$ and the sparsification operator $H$ in place of the quantification operator $H_q$.
Here, vector $\explo$ is dense and $\sig=H(\explo)$ is $k$-sparse.
% as
% \begin{align}
% H\left(\explo\right)  \in %\underset{\sig \in \sigspace, \, \SUPP{\sig} \subseteq \largest_\spars\left(\explo\right) }{\argmin}\normdd{\mat\sig - \obs}
% \underset{\substack{\sig \in \sigspace \\ \SUPP{\sig} \subseteq \largest_\spars\left(\explo\right) }}{\argmin} \frac{1}{2}\normdd{\mat\sig - \obs},
% \label{eq:sparsification_operator}
% \end{align}
% where $\SUPP{\sig}$ denotes the support of $\sig$. 
Applying the STE to the new formulation~\eqref{eq:sea_problem}, we obtain the update $\explo \leftarrow \explo - \eta \frac{\partial F}{\partial x}|_{H(\explo)}$, for a step-size $\eta>0$, where we have approximated  $\frac{\partial F\circ H}{\partial \explo}|_{\explo}=\frac{\partial F}{\partial x}|_{H(\explo)} \frac{\partial H}{\partial \explo}|_\explo \approx \frac{\partial F}{\partial \sig}|_{H(\explo)}$. This leads to an original algorithm for sparse recovery, based on the STE, which we analyze in this article.

\paragraph{Contributions.} The first contribution of the article is to adapt the STE to the sparse support recovery problem (as explained above). Doing so, we obtain a new sparsity-inducing algorithm that we call  {\it Support Exploration Algorithm (SEA)}. It uses the full gradient history over iterations as a heuristic in order to select the next support to optimize over.  SEA is supported by support recovery guarantees. In \cref{thm:recovery_RIP} and \cref{cor:recovery_SRIP}, the sufficient hypotheses guaranteeing the support recovery are on the Restricted Isometry Property (RIP) constants of $A$ and  $x^*$. These conditions are comparable to those in the state-of-the-art, albeit slightly more stringent.
% The theoretical guarantees do not explain the successes of SEA, observed in the experiments, on coherent problems that do not satisfy the RIP hypothesis. Their interest mainly lies in the fact that they apply to an instance of STE, for which very few guarantees of convergence exist. 
Their interest mainly lies in the fact that they apply to an instance of STE, for which very few guarantees of convergence exist.
However, the successes of SEA observed in the experiments extend to coherent problems where the RIP hypothesis is no longer satisfied. 
Additional support recovery statements are in \cref{thm:recovery} and in \cref{cor:recovery_ortho} of \cref{app:proof_RIP}. The proofs are based on the interpretation of SEA as a noisy version of an \lq Oracle algorithm\rq which is analyzed in \cref{oracle-sec}. 

The performances of SEA are compared to those of state-of-the-art algorithms on: 1/ phase-transition synthetic experiments for Gaussian matrices; 2/ spike deconvolution problems; 3/ classification and regression problems for real datasets.
An important feature of  SEA is that it can be used as a post-processing to improve the results of existing algorithms, as shown in the experiments.
% The experiments show that, when used as a post-processing, SEA improves the results of state-of-the-art algorithms. 
Also, because SEA has the ability to explore more supports, it performs remarkably well when the matrix $A$ is coherent. The code is available in the git repository of the project.
\footnote{\url{https://gitlab.lis-lab.fr/valentin.emiya/sea-icml-2024}}

\paragraph{Organization of the article.}

Related works are detailed in \cref{SOTA-sec}. Then, SEA is described in \cref{method-sec}. The theoretical analysis is in \cref{theory-sec}. The experiments are in Section \ref{expe-sec}. Conclusions and perspectives are in Section \ref{discussion-sec}.

In Appendices, a thorough comparison between SEA and the most similar algorithms of the state-of-the-art is detailed in~\cref{app:algorithms}. \cref{app:algorithms} also details an efficient implementation of SEA.
The proofs of the theoretical statements as well as complementary support-recovery statements are in Appendix  \ref{app:proof_RIP}. Complementary experimental results are in Appendices \ref{app:dt}, \ref{app:deconv}, and \ref{app:MLexp}.

% \ve{J'ai enlevé la définition de $\suppsol = \SUPP{\sol}$, il faut la réintroduire quand on en a besoin ou dans la partie Notations.}

%% file: Related_works.tex
\section{Related Works}\label{SOTA-sec}

\paragraph{On the STE.}
Although STE achieves performances defining the state-of-the-art, it is poorly understood. In \cite{li2017training}, the authors show that STE behaves well on convex problems and that a stochastic variant of STE does not on non-convex ones. For a two-layer linear neural network with a quantized activation function,  a well-chosen STE converges to a critical point of the population risk \cite{Yin2019UnderstandingSE} or reproduces a teacher network \cite{long2021learning}.  These are the only known formal guarantees for STE. Other works study the large dimension geometry of binary weights \cite{g.2018the}. In \cite{helwegen2019latent} the authors interpret $w$ as an inertia variable and design a new (related) algorithm. In \cite{chengstraight}, the authors view the STE as a projected Wasserstein gradient flow.

\paragraph{On sparse prior and support recovery.}

Sparse representations and sparsity-inducing algorithms are widely used in statistics and machine learning~\cite{Hastie_T_2015_Statistical_SLwS}, as well as in signal processing~\cite{Elad2010}. For instance, in machine learning, sparse representations are used to select relevant variables. They are also sought to interpret trained models.
In signal processing, linear inverse problems have a wide array of applications. The sparsity assumption is ubiquitous since most real signals can be exactly or approximately represented as sparse signals in some domains, e.g., communication signals in Fourier space, natural images in wavelet space. While sparse models are appealing, they are hard to estimate due to the underlying combinatorial difficulty of identifying the correct sparse support.

%In this article,\todo{Fr : Maintenant que je vois la mise en page avec l'abstract, j'enlèverais bien ce paragraphe. Caroline et Valentin étaient pour je crois. Si c'est la cas, je vous laisse faire.} we introduce a new sparsity inducing algorithm. The algorithm adapts the straight-through-estimator~\cite{hinton2012nnml,bengio2013estimating} (STE) to the sparse representation problem. We provide theoretical guarantees and experimental evidence in the context of support recovery (model selection).

%Compress Sensing was initiated by ~\cite{candes2006robust} and ~\cite{donoho2006compressed}. 
%This theory is very useful when dealing with sparse signal and combines acquisition and compression. Instead of using Shanon-Nyquist theorem with the Nyquist rate to sample at a rate equal to twice the highest frequency. Compress Sensing relies on a sampling rate that depends of the signal sparsity.

% \paragraph{Support recovery.} 
Our algorithm has been designed in a support recovery context.
In the noisy case $e\neq 0$, support recovery is a stronger guarantee than the one in the most standard compressed sensing setting, initiated in \cite{candes2006robust,donoho2006compressed}, when the goal is to upper-bound $\normd{x-x^*}$, for a well-chosen $x$.
%\todo[noinline]{$x$ or $x^*$?}\todo[noinline]{M: J'aurai dit $x$, car c'est l'algo qui choisi bien son estimée} Fr: C'est bien l'estimé $x$.
The first particularity of support recovery is to assume $x^*$ is truly $k$-sparse -- not just compressible. Also, support recovery guarantees always involve a hypothesis on $\min_{i\in S^*} |x^*_i|$, in addition of the incoherence hypothesis on $A$ \cite{wainwright2009sharp,meinshausen2006high,zhao2006model,cai2011orthogonal,NIPS2016_e9b73bcc}. We cannot indeed expect to recover an element $i\in S^*$ if $|x^*_i|$ is negligible when compared to all the other quantities involved in the problem \cite{wainwright2009sharp}.

\paragraph{Support recovery models and algorithms.}
%A famous model for support recovery is given by eq.~\eqref{eq:constrained_sparse_problem}.
% \begin{equation}
% \underset{\sig \in \sigspace,\, \normz{\sig} \leq \spars}{\minimize} F\left(\sig\right) := \frac{1}{2}\normdd{\mat\sig - \obs}~.
% \label{eq:sparse_problem}
% \end{equation}
%However, the sparsity constraint $\normz{\sig} \leq \spars$ induces a combinatorial, non-differentiable, and non-convex aspect in the problem, which is NP-Hard . 
%To avoid going through the %${\sigsize \choose \spars}$ 
%$\binom{\sigsize}{\spars}$ possible supports, each leading to a differentiable and convex sub-problem

Beyond \eqref{eq:constrained_sparse_problem}, various algorithms were investigated.
There are three main families of algorithms: relaxation, combinatorial approaches, and greedy algorithms.

The most famous relaxed model uses the $\ell^1$ norm and is known as the LASSO \cite{tibshirani1996regression} or Basis Pursuit Algorithm \cite{chen2001atomic}. 
Combinatorial approaches like Branch and Bound algorithms~\cite{benmhenni2021global}%\todo{Voir biblio de \cite{benmhenni2021global} si besoin de plus de refs}
, find the global minimum of \eqref{eq:constrained_sparse_problem} but lack scalability. 
Greedy algorithms can be divided into two categories. Greedy Pursuits like
Matching Pursuit (MP)~\cite{mallat1993matching} and Orthogonal Matching Pursuit (OMP)~\cite{pati1993orthogonal} 
are algorithms that start from an empty support and build up an estimate of $\sol$ by iteratively adding components to the current support and optimizing the components. As for thresholding algorithms like Iterative Hard Thresholding (IHT)~\cite{Blumensath2009}, Normalized Iterative Hard Thresholding (NIHT)~\cite{blumensath2010normalized}, Hard Thresholding Pursuit (HTP)~\cite{foucart2011hard},  Compressive Sampling Matching Pursuit (CoSaMP)~\cite{needell2009cosamp}, OMP with Replacement (OMPR)~\cite{jain2011orthogonal}, Exhaustive Local Search (ELS)~\cite{pmlr-v119-axiotis20a} (a.k.a. Fully Corrective Forward Greedy Selection with Replacement~\cite{shalev2010trading}) and  Subspace Pursuit (SP)~\cite{dai2009subspace}, they start from any vector and add a replacement step in the iterative process. It allows them to explore various supports before stopping at a local optimum. %Using the Restricted Isometric Property (RIP), it can be shown that these algorithms find the global minimum of our problem under RIP assumptions.

\paragraph{Position of the article.}

In this work, we take a different approach and apply the STE to a well-understood problem to compare its behavior, empirical performances, and theoretical guarantees to those of the well-established state-of-the-art. 

Compared to other sparse support recovery algorithms, the algorithm introduced in this article may belong to the family of greedy algorithms. A clear difference is the introduction of a non-sparse vector $\explo^t\in\RR^n$, which evolves during the iterative process and indicates which support should be tested at iteration $t$. We call $\explo^t$ the {\it support exploration variable}. It is the analog of the full-precision weights -- used by  BinaryConnect that also instantiates the STE -- to optimize binary weights of neural networks~\cite{courbariaux2015binaryconnect,hubara2016binarized}. We exhibit that the adaptation of STE for sparsification enables a different exploration/exploitation trade-off compared to the state-of-the-art. It explores more. This permits to obtain better performances than the state-of-the-art on very difficult --coherent-- problems. We establish that it is possible, in the sparse support recovery context, to obtain theoretical guarantees for the STE.

%\todo[inline]{Lien avec STE: Binary connect, etc... Binary Connect~\cite{courbariaux2015binaryconnect} introduced in \cite{hubara2016binarized} extends Binary Connect and mentions the Straight-Through Estimator proposed in Hinton 2012's lecture on Deep Neural Networks for Machine Learning. STE is developed in~\cite{bengio2013estimating}, where there is also a mention to Hinton's lecture No 15b but Valentin did not found STE in the video \url{https://www.youtube.com/watch?v=_Ex1Ur85AVs&list=PLoRl3Ht4JOcdU872GhiYWf6jwrk_SNhz9&index=70}}

%% file: Method.tex
\section{Method}\label{method-sec}

% We define the notations in Section \ref{notation-sec} and SEA in Section \ref{algo-sec} We detail its link with STE in \cref{ste-sec}.

After clarifying the notations in \cref{notation-sec}, SEA is described in detail in \cref{algo-sec}  and its computational complexity is discussed in \cref{complexity-sec}.

\subsection{Notations}\label{notation-sec}

For any $a, b \in \RR$ ($a$ and $b$ can be real numbers), the set of integers between $a$ and $b$ is denoted by $\intint{a, b}$ and $\floor{a}$ denotes the floor of $a$.
For any set $\supp \subseteq \intintn $, we denote the cardinality of $S$ by $\abs{S}$. 
The complement of $\supp$ in $\intintn$ is denoted by $\overline{\supp}$. 

%The vectors $\zerosig$ and $\zeroobs$ are respectively the null vectors of $\sigspace$ and $\obsspace$. \todo[noinline]{Fr : J'enlèverais les notations $\zerosig$ et $\zeroobs$.}
%Vector $0_{\obsspace}$ (resp. $1_{\obsspace}$) is the all-zeros (resp. all-ones) vector of $\obsspace$.
Given $\sig \in \sigspace$ and $\ii \in \intintn$, the $\ii^{th}$ entry of $\sig$ is denoted by $\sig_\ii$. 
The support of $\sig$ is denoted by $\SUPP{\sig} = \{\ii : \sig_\ii \neq 0\}$.
The $\ell^0$ pseudo-norm of $\sig$ is defined by $\normz{\sig} = \abs{\SUPP{\sig}}$.
The set containing the indices of the $\spars$ largest absolute entries of $\sig$ is denoted by $\largest_\spars\left(\sig\right)$. 
When ties lead to multiple possible choices for $\largest_\spars\left(\sig\right)$, we select the solution with the highest indices. For instance, $\text{largest}_k(0)=\llbracket n-k+1, n \rrbracket$.
%we assume $\largest_\spars\left(\sig\right)$ arbitrarily chooses one of the possible solutions. When such a case occurs, 

For any $\supp \subseteq \intint{1, \sigsize}$, $\mat \in \matspace$, and $\sig \in \sigspace$, we define $\truncv{\sig}{\supp} \in \field^{\abs{\supp}}$, the restriction of the vector $\sig$ to the indices in $\supp$, and $\matsupp \in \field^{\obssize \times \abs{\supp}}$, the restriction of the matrix $\mat$ to the set $\supp$ as the matrix composed of the columns of $\mat$ whose indexes are in $\supp$. 
The transpose of $\mat$ is denoted by $\mata \in \matspaceinv$.
The pseudoinverse of $\mat$ is denoted by $\matdag \in \matspaceinv$ and the pseudoinverse of $\matsupp$ by $\matdagsupp = (\matsupp)^\dag \in \field^{\abs{\supp} \times \obssize}$.
For any $d \in \NN$, the identity matrix of size $d$ is denoted by $I_{d}$. The symbol $\odot$ denotes the Hadamard product.

\subsection{The Support Exploration Algorithm}\label{algo-sec}

% The Support Exploration Algorithm depicted from the STE point of view in Algorithm~\ref{alg:SEA_backprop} is explicitly given in Algorithm~\ref{alg:SEA}.

The proposed Support Exploration Algorithm (SEA) is given in Algorithm~\ref{alg:SEA}. 
In terms of pseudocode, SEA resembles many state-of-the-art algorithms and is close to HTP and IHT (see comparison between SEA, HTP, and IHT in Appendix~\ref{app:algorithms}).
However, it stands out from the others for its exploratory behavior, which stems from the STE principle behind it.

% Utiliser l'option pour ajuster la hauteur souhaitée (sans unité = nombre de lignes)
%\begin{wrapfigure}[18]{r}{0.5\textwidth}
%\begin{minipage}{0.5\textwidth}
%\vspace{-7mm}
    \begin{algorithm}[tbh]
       \caption{Support Exploration Algorithm \label{alg:SEA}}
    \begin{algorithmic}[1] % Noise of the gradient
       \STATE {\bfseries Input:} noisy observation $\obs$, sampling matrix $\mat$, sparsity $\spars$, step size $\gradstep$
       \STATE {\bfseries Output:} sparse vector $\sigs$
       \STATE Initialize $\explo^0$ 
       \STATE $\iter \leftarrow 0$
       \REPEAT
       \STATE $\supp^\iter \leftarrow \largest_\spars\left(\explo^\iter\right)$ 
       \alglinelabel{line:SEA:update_S}
       % \STATE $\sigs^\iter \leftarrow 0_{\sigspace}$\alglinelabel{line:SEA:update_x0}
       % \STATE $\sigs^\iter_i \leftarrow 0$ for $i \in \overline{\supp^\iter}$
       % \alglinelabel{line:SEA:update_x0}
       % \STATE $\sigs^\iter_{\supp^\iter} \leftarrow \matdagsuppt \obs$ \alglinelabel{line:SEA:update_x1}
       \STATE $\begin{cases}
           \sigs^\iter_i & \leftarrow 0 \text{ for } i \in \overline{\supp^\iter}\\
           \sigs^\iter_{\supp^\iter} & \leftarrow \matdagsuppt \obs
       \end{cases}$
       \alglinelabel{line:SEA:update_x}
%       \STATE $loss^{\iter} \leftarrow \frac{1}{2}\normdd{\mat\sig^{\iter} - \obs}$\alglinelabel{line:SEA:loss}
       % \STATE $\explo^\iterp \leftarrow \explo^\iter - \eta\mat^T\left(\mat^T_{\supp^\iter}\sigs^\iter-\obs\right)$
       % \STATE $\sigs^\iter \leftarrow \underset{\substack{\sig \in \sigspace \\ \SUPP{\sig} \subseteq \supp^\iter }}{\argmin} \normdd{\mat\sig - \obs}$\alglinelabel{line:SEA:update_x}
       % \STATE $\sigs^\iterp = \underset{\sig \in \sigspace, \, \SUPP{\sig} \subseteq \supp^\iterp }{\argmin} \normdd{\mat\sig - \obs}$ \label{line:SEA:update_x}
       \STATE $\explo^\iterp \leftarrow \explo^\iter - \grad$ \alglinelabel{line:SEA:update_explo}
       \STATE $\iter \leftarrow \iterp$
       \UNTIL{halting criterion is $true$}
       % \STATE \alglinelabel{line:SEA:best} $\iterb \leftarrow \underset{\itero \in \intint{0, \iter}}{\argmin} \normdd{\mat\sigs^\itero - \obs}$
       \STATE \alglinelabel{line:SEA:best} $\iterb \leftarrow \underset{\itero \in \intint{0, \iter}}{\argmin} \normdd{\mat\sigs^\itero - \obs}$
       \STATE {\bfseries return} $\sigs^\iterb$
    \end{algorithmic}
    \end{algorithm}
    % \centering
    % \includegraphics[width=0.9\linewidth]{images/SEA_supp_3.drawio.png}
    % \caption{Visual representation of the main sets of indices encountered in the article.}
    % \label{fig:sea-supports}
    %\vskip -0.2in
%\end{minipage}
%\end{wrapfigure}

Each iteration begins by the foward pass given by \eqref{eq:sparsification_operator}, in which the current sparse solution $\sig^\iter$ is computed by applying the sparsification operator $H$ to a dense vector $\explo^\iter$:
after selecting the support $S^t$ from~$\explo^\iter$ at line~\ref{line:SEA:update_S},
the sparse solution $\sig^\iter= H(\explo^t)$ is computed at line~\ref{line:SEA:update_x}.
% This step includes the main computational difficulty, that consists in solving the (unconstrained) linear system $A_{S^t}^TA_{S^t}x_{S^t}=A_{S^t}^Ty$ where matrix $A$ is restricted to the current support $S^t$.
% The forward pass ends with the computation of $Ax^t$ at line~\ref{line:SEA:update_explo}.
The backward pass uses the STE principle 
\[\frac{\partial F\circ H}{\partial \explo}|_{\explo^\iter}\approx\frac{\partial F}{\partial x}|_{\sig^\iter}=\mata(\mat\sigst - \obs),
\] 
at line~\ref{line:SEA:update_explo}.
To illustrate with the analogy with BinaryConnect \cite{courbariaux2015binaryconnect}, the non-sparse vector $\explo$ is the analog of the full-precision  weights and $H(\explo)$ is the analog of the quantized weights.
To the best of our knowledge, this is the first use of the STE to solve a sparse linear inverse problem.

The key idea is that support $\supp^\iter$ is designated at line~\ref{line:SEA:update_S} by a non-sparse variable $\explo^\iter$ called the \textit{support exploration variable}.
It offers an original mechanism to explore supports in a more diverse way than existing algorithms.
Variable $\mathcal{X}^{t+1}=\mathcal{X}^0-\eta\sum_{t'=0}^{t}A^T(Ax^{t'}-y)$ is actually an accumulation of gradients taken in the sparse iterates and is used to designate the support of the next sparse iterate $x^{t+1}$. Consequently, unlike other descent-based algorithms, $\explo^\iter$ is not confined to the neighborhood of $k$-sparse vectors.
Its evolution is not intended to make the objective function decrease at each iteration.
In this regard, since the algorithm explores supports in a way that allows the functional to sometimes increase, the retained solution is the best one encountered along the iterations (line~\ref{line:SEA:best}).
Illustrations of this phenomenon are given in \cref{app:deconv-precise} where one can see that the behavior of the loss along the iterations shows important variations when a new support is explored. This is an important difference with the aforementioned state-of-the-art algorithms, resulting in increased exploration.

An important feature of SEA is that it can be used as a post-processing of the solution $\hat x$ of another algorithm. This is simply done by initializing $\explo^0 = \hat x$. In this case $S^0 = \SUPP{\hat x}$ (line \ref{line:SEA:update_S}) and $x^0$  improves or is equal to $\hat x$ (line \ref{line:SEA:update_x}). %\Mim[Since SEA returns the result obtained for the best time-step $\iterb$ (line  {\ref{line:SEA:best}}), it can only improve $\hat x$]{Redite avec le paragraphe suivant}. 
In the experiments, we have investigated the initialization with the result of OMP \cite{pati1993orthogonal}, ELS \cite{pmlr-v119-axiotis20a,shalev2010trading} and the initialization $\explo^0 = 0$. We observe that the initialization with ELS is generally preferable except for  difficult problems, when columns of $A$ are very coherent  (see \cref{deconv-sec}).

%Eventually, the solution returned by SEA is selected at line~\ref{line:SEA:best} as the best iterate encountered along the iterations (see line \ref{line:SEA:best}). %, since the exploration mechanism may result in a non-decreasing behavior of the objective function values. \Fr{J'ai ocmmenté car c'est déjà dit plus haut.}

Finally, as often, there are many possible strategies to design the halting criterion of the 'repeat' loop of Algorithm \ref{alg:SEA}. It is clear that a more permissive criterion allows for more exploration and better results, at the expense of computation time. We have not investigated this aspect in the experiments and leave this study for the future. We preferred to focus our experiments on the illustration of the potential benefits of SEA and, as a consequence, we always used a large  fixed number of passes in the 'repeat' loop of Algorithm \ref{alg:SEA}. %It is however an important question that we leave for future research. 

Similarly, since we have $\explo^\iter=\explo^0-\eta\sum_{t'=0}^{t-1}\mat^T(\mat x^{t'}-\obs)$ for all $t\geq1$, $\eta$ has no impact on $S^t$ and $x^t$ when $\mathcal{X}^0=0$ and therefore on the output $x^{t_{BEST}}$ of Algorithm~\ref{alg:SEA} as $\largest_\spars(\mathcal{X}^{t})$ does not depend on $\eta$.
In this case, indeed, the whole trajectory $(\explo^\iter)_{t\in\NN}$ is dilated by $\eta>0$ and the dilation has no effect on the selected supports $S^t$. When $\explo^0 \neq 0$, the initial support exploration variable is forgotten as the iterations progress. It is forgotten more rapidly when $\eta$ increases. We have not studied the tuning of this parameter in depth, leaving it for future research.

\subsection{Computational Complexity}\label{complexity-sec}
An efficient implementation of SEA is described in Appendix~\ref{app:algorithms:sea_efficient}.
The analysis of the computational complexity of SEA is based on two facts.
First, if the support $\supp^\iter$ obtained at line~\ref{line:SEA:update_S} has already been explored, then the sparse vector $\sigs^\iter$ and the gradient $\grad$ have already been computed. So, if these quantities have been memorized (as in \cref{alg:SEA_efficient}, Appendix~\ref{app:algorithms:sea_efficient}), the cost of the iteration is negligible. The overall cost thus depends on the number of explored supports rather than on the number of iterations.
Second, each time a new support is extracted, the cost of the iteration is dominated by solving the %subproblem $\argmin_{\SUPP{\sig} \subseteq \supp}F(\sig)$, with $\abs{\supp}=\spars$ 
(unconstrained) linear system $A_{S^t}^TA_{S^t}x_{S^t}=A_{S^t}^Ty$. While the pseudo-inverse is a convenient notation at line~\ref{line:SEA:update_x}, the solution may be obtained more efficiently, e.g. in $\mathcal{O}(\spars^2 n)$ to compute $A_{S^t}^TA_{S^t}$ and $A_{S^t}^Ty$, and apply the conjugate gradient algorithm.
The overall complexity is thus in $\mathcal{O}(n_{\text{supp}}\spars^2\sigsize)$ where $n_{\text{supp}}$ is the number of supports actually explored. 

The complexity of HTP, OMP, OMPR and ELS is also dominated by the number of times $A_{S}^TA_{S}x_{S}=A_{S}^Ty$ is solved for $S$ such that $\abs{\supp}=\spars$.
 As for SEA, efficient implementations of HTP, OMPR and ELS can save computations by storing all the explored support and related iterates. The HTP, OMP and OMPR then depend on the number of explored supports in a similar way as SEA. The OMP solves $\spars$ instances of them which results in less exploration and less computational cost. The ELS is much more demanding since it explores $(\sigsize-\spars)$ supports at each iteration, many of which are irrelevant. 
IHT has a lower complexity than SEA since it never inverses the system $A_{S}^TA_{S}x_{S}=A_{S}^Ty$.

As we will see in the deconvolution experiments in \cref{deconv-sec}, SEA outperforms ELS (see \cref{fig:dcv_mass}) while exploring two times less supports (see \cref{app:deconv-precise,app:deconv-n_support}).
The possibility of performing a random search with the same computational cost as SEA has been studied in \cref{app:deconv:random}.

%% file: Theoretical_analysis.tex
\section{Theoretical Analysis}\label{theory-sec}

In this section, we provide the theorem stating that SEA recovers the correct support for some\footnote{In particular it is necessary that $\min_{i\in S^*} |x_i^*|$ is sufficiently large.} $x^*$ when the matrix $A$ satisfies a RIP constraint. Then, we compare the conditions with existing support recovery conditions for state-of-the-art algorithms. In addition to the statements in this section, recovery statements are given in Appendices \ref{app:proof_RIP}. The interest of the theorems lies mainly in the fact that they apply to an instance of STE, for which guarantees are rare. From the practitioner's point of view, the theoretical analysis is not useful since SEA mostly shows promises in coherent scenarios in which the RIP hypothesis is not satisfied. 

In this section, we assume that columns of $\mat$ are normalized: for any $\ii \in \intintn$, $\normd{\mat_\ii} = 1$. As has been standard practice since Candès and Tao first proposed it in \cite{candes2005decoding}, we define for all $l\in \intintn$ the $l$th Restricted Isometry Constant of $\mat$ as the smallest non-negative number $\delta_l$ such that for any $\sig \in \sigspace$, such that $\norm{\sig}_0 \leq l$,
\begin{equation}
\label{eq:RIP}
    (1-\delta_l)\normdd{\sig} \leq \normdd{\mat\sig} \leq (1+\delta_l)\normdd{\sig}.
\end{equation}
If $\delta_l < 1$, $\mat$ is said to satisfy the Restricted Isometry Property of order $l$ or the $l$-RIP.

In this section, we assume that $\mat$ satisfies the $(2\spars + 1)$-RIP. In the scenarios of interest, $\RIPdkp$ is small. We define
\begin{equation}
\label{eq:alpha_gammaRIP}
\alphaRIP = \alphaRIPf \in \RR^*_+ 
\qquad \text{ and }\qquad
\gammaRIP = \gammaRIPf \in \RR^*_+.
\end{equation}

% \begin{equation}
% \label{eq:alphaRIP}
% \alphaRIP = \alphaRIPf \in \RR^*_+
% \end{equation}
% and
% \begin{equation}
% \label{eq:gammaRIP}
% \gammaRIP = \gammaRIPf \in \RR^*_+.
% \end{equation}

As soon $\delta_k$ is far from $1$ (for example $\delta_k\leq \frac{1}{2}$), %which is true in the scenarios of interest, 
$\alpha_k^{RIP}$ has the order of magnitude of $\delta_{2k+1}$ (in the example $\delta_{2k+1}\leq\alpha_k^{RIP}\leq 3\delta_{2k+1}$)  and $\gamma_k^{RIP}$ has the order of magnitude of $1+\delta_{2k+1}$ (in the example $(1+\delta_{2k+1})\leq\gamma_k^{RIP}\leq\sqrt{6} (1+\delta_{2k+1})$).

As is typical of support recovery statements, the next theorem includes a condition on $x^*$. 
We call this condition  {\it the Recovery Condition for the RIP case (\CondRIP)}. It is defined by
\begin{equation}\tag{\CondRIP}
\label{eq:HRIP}
\gammaRIP \norm{\noise}_2 < \frac{\min_{i\in S^*} |x_i^*|}{2 k} - \alphaRIP \norm{\sol}_2.
\end{equation}
%If (\ref{eq:HRIP}) holds, $\sol$ is said to satisfy the (\ref{eq:HRIP}) condition. 

%Before stating the theorem, we remind that if we replace, in Algorithm \ref{alg:SEA}, the STE update $\explo^\iterp \leftarrow \explo^\iter - \grad$ by the Oracle update $\explo^{\iter+1} \leftarrow\explo^\iter - \usolt $, where $\usol^\iter$ is defined in \eqref{eq:usol}, we obtain \cref{alg:oracle} that recovers $\suppsol$, see \cref{oracle-sec}.
% Since $\usolt$ cannot be computed, we update $\explo^\iter$ with the STE update, see line~\ref{line:SEA:update_explo} of Algorithm \ref{alg:SEA}.}

\begin{theorem}[Recovery - RIP case]
\label{thm:recovery_RIP}
Assume $\mat$ satisfies the $(2\spars+1)$-RIP and\footnote{The normalization aims at simplifying formulas by guaranteeing that $\delta_1=0$. It is done at no expense since, if $A$ is not normalized but satisfies \eqref{eq:RIP} for $l>1$, its normalization only has a small impact on $\delta_l$. Indeed, considering $\Delta\in\mathbb{R}^{n\times n}$ diagonal such that $\Delta_{i,i}=\|A_i\|_2$, $A\Delta^{-1}$ is normalized and for all $l$-sparse vector $x$ 
    $$ (1-\delta_l)\|\Delta^{-1}x\|_2^2\leq\|A\Delta^{-1}x\|_2^2\leq(1+\delta_l)\|\Delta^{-1}x\|_2^2.$$
    Using $1-\delta_1\leq\|A_i\|^2_2 \leq 1+\delta_1$, we can derive $l$-RIP constants for the normalized matrix $A \Delta^{-1}$.} for all $\ii \in \intintn$, $\normd{\mat_\ii} = 1$. 
%Then, for all $t\in\NN$, 
%\[\|\frac{u^t}{\eta} - A^T(Ax^t-y) \|_\infty\leq \alphaRIP \|x^*\|_2 + \gammaRIP \|e\|_2.
%\]
% \[\gradnoisemax \leq \eta (\alphaRIP \|x^*\|_2 + \gammaRIP \|e\|_2).
% \]
Assume moreover that $\sol$ satisfies \eqref{eq:HRIP}. 

Then for all initializations $\explo^0$ and all $\eta>0$, there exists $t_s\leq \itermaxRIP$ such that $\suppsol \subseteq S^{t_s}$, where
\begin{equation}
\label{eq:TRIP}
    \itermaxRIP =\frac{2 k \frac{\|\explo^0\|_{\infty}}{\eta} + (k+1) \min_{i\in S^*} |x_i^*| }{\min_{i\in S^*} |x_i^*| - 2k \left(\alphaRIP\normd{\sol} + \gammaRIP\normd{\noise}\right)}.
\end{equation}

If moreover, $x^*$ is such that 
\begin{equation}\label{min_hyp_RIP}
    \min_{i\in S^*} |x^*_i| > \frac{2}{\sqrt{1-\delta_{2k}}} \|e\|_2 
\end{equation}
and SEA performs more than $\itermaxRIP$ iterations, then $S^* \subseteq S^{\iterb}$ and $\|x^{\iterb} - x^*\|_2 \leq \frac{2}{\sqrt{1-\delta_{k}}} \|e\|_2$.
% \[S^* \subseteq S^{\iterb} ~~~~ \mbox{and}~~~~ \|x^{\iterb} - x^*\|_2 \leq \frac{2}{\sqrt{1-\delta_{k}}} \|e\|_2.
% \]
\end{theorem}

The proof is in \cref{app:proof_RIP}. To introduce the proof and provide the main intuition, we first detail in \cref{oracle-sec} a theorem and its proof stating that a variant of the support exploration algorithm, see \cref{alg:oracle}, using the oracle update rule defined by
$\explo^\iterp \leftarrow \explo^\iter - u^t$ where for all $i\in\intintn$
\[\usol^\iter_\ii =
    \begin{cases}
      -\gradstep \soli & \ii \in \suppx \\
      0 & \ii \in \suppxb,
    \end{cases}
\]
always recovers the true support $\suppsol$. This update rule depends on $x^*$ and has no practical application. Its interest lies in providing an \lq ideal\rq~update rule that allows for the fast recovery of the true support, as seen in \cref{thm:recovery:oracle}.

The intuition behind the success of the oracle update rule is that the non-zero entries of $\usol^\iter_\ii$ are for indices $\ii$ from the true support $\suppsol$ but for which $|\explo^\iter_\ii|$ is too small to be selected in $\supp^\iter$ at line~\ref{line:oracle:update_S}.
%that is missing in current support $\supp^\iter$.
Whatever the initial content of $\explo^0$, the oracle update rule always adds the same increment to $\explo^\iter_\ii$, for $\ii \in \suppx $, and those for $i\in\overline{\suppsol}$ never change. This guarantees that, at some subsequent iteration $\iter'\geq \iter$,  the true support $\suppsol$ is recovered among the $\spars$ largest absolute entries in $\explo^{\iter'}$, i.e., $\suppsol \subseteq \supp^{\iter'}$.

The proof of \cref{thm:recovery_RIP} relies on measuring the deviation of the trajectory $(\explo^\iter)_{\iter\in\NN}$ defined by \cref{alg:SEA} from the trajectory defined with the oracle update rule. More precisely, in \cref{thm:recovery} of \cref{subsec:genrecovery}, we provide a sufficient condition on the discrepancy between the oracle update and the STE-update guaranteeing that SEA visits the true support $\suppsol$. Then, in \cref{proof-thm-RIP-sec}, we establish that the hypotheses of \cref{thm:recovery_RIP} ensure that the discrepency is sufficiently small to satisfy the hypothesis of \cref{thm:recovery}. 

We also establish in \cref{cor:recovery_ortho} of \cref{annexe-ortho} that when the columns of $\mat$ are orthonormal and $\noise=0$, SEA recovers $\sol$ in less than $k+1$ iterations. Despite being a sanity check with no interest in applications, this result provides a meaningful case where the oracle update rule and the STE update rule coincide.

We emphasize that none of the proofs rely on the fact that a function decays. In particular, as will be illustrated in the experiments, $(F(H(\explo^\iter)))_{\iter \in\NN}$ generally exhibits erratic behavior. This is because, by construction in \cref{alg:SEA}, the next support, designated by $\explo^\iterp$, is not restricted to supports for which $F\circ H$ decays. SEA explores more supports than algorithms with this restriction and, for instance, does not get trapped in local minima. 

%The conclusion of \cref{thm:recovery} is that the iterative process of SEA recovers the correct support at some iteration $t$. We have in general no guarantee that this time-step $t$ is equal to $\iterb$. \Kro{We are however guaranteed that SEA returns a sparse solution such that $\normd{Ax^{\iterb} - y} \leq \normd{Ax^* - y}$. This estimation error (a.k.a. empirical risk) can be considered as a criterion of success since the estimation of SEA satisfies the sparsity constraint and achieves an equal or lower objective value than $x^*$.} We will see in \cref{cor:recovery_ortho}, \cref{thm:recovery_RIP} and \cref{cor:recovery_SRIP} that, when $A$ is sufficiently incoherent and $\|e\|_2$ is small enough, we actually have $S^* \subseteq \SUPP{x^\iterb}$.

Let us now discuss $\itermaxRIP$.

When (\ref{eq:HRIP}) holds, $\itermaxRIP$ increases as  $\min_{i\in S^*} |x_i^*| - 2k \left(\alphaRIP\normd{\sol} + \gammaRIP\normd{\noise}\right) $ decreases. In particular, the number of iterations required by the algorithm  to provide the correct solution increases when the information on some of the columns of $S^*$ diminishes, i.e. when $\min_{i\in S^*} |x_i^*|$ decreases.  Also, the initializations $\explo^0\neq 0$ have an apparent negative impact on the number of iterations required in the worst case. This is because in the worst-case $\explo^0$ is poorly chosen and SEA needs iterations to correct this poor choice.

%\Mim[Concerning $\gradstep$, notice that, since $u^t$ is proportional to $\gradstep>0$, $\gradnoisemax$ is proportional to $\gradstep>0$ and therefore ({\ref{eq:RIP}}) is independent of $\gradstep$.]{} 
When the conditions of Theorem~\ref{thm:recovery_RIP} are met, any $\explo^0$ and $\gradstep$ permit the recovery of $S^*$. $\explo^0$ and $\gradstep$ only influence $\itermaxRIP$. In this regard, since the larger $\gradstep$, the faster SEA overrides the initialization $\explo^0$, the choice of $\eta$ is very much related to the question of the quality of the initialization. The latter is often beneficial in practice.

To illustrate (\ref{eq:HRIP}), we provide below a simplified condition which is shown in \cref{cor:recovery_SRIP} to be stronger than (\ref{eq:HRIP}) in the noiseless scenario.
We say $\sol$ satisfies the Simplified Recovery Condition in the RIP case if there exists $\RIPthresh\in(0,1)$ such that
\begin{equation}\tag{\CondSRIP}
\label{eq:HSRIP}
    \RIPthreshf \leq \RIPthresh.
\end{equation}
%If (\ref{eq:HSRIP}) holds, $\sol$ is said to satisfy the (\ref{eq:HSRIP}) condition for $\mat$. 

\begin{corollary}[Noiseless recovery - simplified RIP case]
\label{cor:recovery_SRIP}
Assume $\normd{\noise} = 0$, $\mat$ satisfies the $(2\spars+1)$-RIP and for all $\ii \in \intintn$, $\normd{\mat_\ii} = 1$.

If moreover $\sol$ satisfies (\ref{eq:HSRIP}), then $\sol$ satisfies (\ref{eq:HRIP}).
As a consequence, for $\explo^0 = \zerosig$ and for all $\eta>0$, if SEA performs more than  
%\begin{equation}
%\label{eq:TSRIP}
    $\itermaxSRIP = \frac{\spars + 1}{1 - \RIPthresh}$
%\end{equation}
iterations, we have $S^* \subseteq S^{\iterb}$ and $x^{\iterb} = x^*.$
% \[S^* \subseteq S^{\iterb} \qquad \mbox{and}\qquad x^{\iterb} = x^*.
% \]
\end{corollary}

The proof is in \cref{app:proof_SRIP}.

Compared to the support recovery guarantees for the LASSO \cite{wainwright2009sharp,meinshausen2006high,zhao2006model}, the OMP \cite{cai2011orthogonal}, the HTP \cite{foucart2011hard,NIPS2016_e9b73bcc} and the ARHT \cite{pmlr-v119-axiotis20a} the recovery conditions provided in \cref{thm:recovery_RIP} and \cref{cor:recovery_SRIP} for SEA are stronger. All conditions involve a condition on the incoherence of $A$ and a condition similar to \eqref{min_hyp_RIP}. In the case of the LASSO algorithm, the latter is not very explicit. However, none of the support recovery conditions involve a condition like (\ref{eq:HRIP}) and (\ref{eq:HSRIP}). Let elaborate on these two conditions.

One notable limitation of (\ref{eq:HRIP}) and (\ref{eq:HSRIP}) is that if $\alpha_k^{RIP} \neq 0$, there is no guarantee of recovering the support of an $\sol$ such that $\max_{i\in \suppsol} |\soli| \gg \min_{i\in \suppsol} |\soli|$. In fact, by increasing $\sol_j$ for $j\in\suppsol$ such that $|\sol_j| \neq \min_{i\in\suppsol} |\soli|$, we can maintain $\min_{i\in\suppsol} |\soli|$ unchanged while sufficiently increasing $\|\sol\|_2$, causing (\ref{eq:HRIP}) and (\ref{eq:HSRIP}) to fail. We do not have this problem with conditions similar to \eqref{min_hyp_RIP}.
We experimented on the influence of $\frac{\normd{\sol}}{\min_{\ii\in\suppsol}\abs{\soli}}$ on the recovery performances in \cref{app:deconv:u_10}.

Another critical question arises: What (necessary and sufficient) condition on $\alpha_k^{RIP}$, $\gamma_k^{RIP}$ and $\|e\|_2$ ensures the existence of $\sol$ satisfying (\ref{eq:HRIP}) or (\ref{eq:HSRIP})?

To answer this question for the condition (\ref{eq:HRIP}), we first remark that there exists $\sol$ satisfying (\ref{eq:HRIP}) if and only if there exists some constant $c\in\RR$ and $\sol$, such that for all $i\in S^*$, $x^*_i = c$, satisfying (\ref{eq:HRIP}). Indeed, if $\sol$ satisfies (\ref{eq:HRIP}), we take $c=\min_{i\in\suppsol} |\soli|$ and the vector $\tilde x^*$ defined by $\tilde  x^*_i= \min_{i\in\suppsol} |\soli|$ for all $i\in\suppsol$, and $\tilde x^*_i=0$ for all $i\not\in\suppsol$ also satisfies (\ref{eq:HRIP}), since we have $ \min_{i\in\suppsol} |\tilde x^*_i| =  \min_{i\in\suppsol} |\soli|$ and $\|\tilde x^*\|_2\leq \|\sol\|_2$.

We can rewrite  (\ref{eq:HRIP}) when for all $i\in S^*$, $x^*_i = c$ for some constant $c\in\RR$ and obtain
\[ \gammaRIP \|e\|_2 < |c| \left(  \frac{1 - 2 k \alphaRIP  |S^*|^{\frac{1}{2}}}{ 2 k }  \right). 
\]
The existence of $c\in\RR$ such that this condition holds only depends on the sign of $1 - 2 k \alphaRIP |S^*|^{\frac{1}{2}}$. If $1 - 2 k\alphaRIP |S^*|^{\frac{1}{2}}\leq 0$, there does not exist any $c$ satisfying the condition; if $1 - 2 k\alphaRIP |S^*|^{\frac{1}{2}} > 0$, any $c\in\RR$ satisfying
\[|c| \geq \left( \frac{2 k \gammaRIP}{1 - 2k\alphaRIP |S^*|^{\frac{1}{2}}}\right)  \|e\|_2
\]
leads to an $x^*$ that satisfies \eqref{eq:HRIP}. Therefore, when $|\suppsol|=k$, the condition $\alphaRIP < \frac{1}{2} k^{-\frac{3}{2}}$ is necessary and sufficient to guarantee the existence of an $x^*$ satisfying  \eqref{eq:HRIP}. Similar developments concerning (\ref{eq:HSRIP}) are provided in \cref{comments-RIP-sec}.

We remind that $\alphaRIP$ has the order of magnitude of $\delta_{2k+1}$. Therefore, the condition is $\alphaRIP < \frac{1}{2} k^{-\frac{3}{2}}$ is more stringent than the equivalent conditions for other state-of-the-art methods. For instance, the conditions described in \cite{pmlr-v119-axiotis20a} are of the form $\delta_{2k} < C$, for a universal constant $C<1$.

Initializing SEA with the solution of an algorithm enjoying better conditions of recovery is a simple and effective way for SEA to inherit its support recovery guarantee as soon as \eqref{min_hyp_RIP} holds. This can formally be proved using the same proof as in \cref{subsec:recovery_RIP}.

%\fm{Ils demandent de déplacer \cref{comments-RIP-sec} qui fait une 1/2 page ici. J'attends de voir si on a la place.} 

As will be seen later in the experiments of \cref{expe-sec}, SEA performs well even when $A$ is coherent. This is not explained by \cref{thm:recovery_RIP} and \cref{cor:recovery_SRIP} which use the RIP assumption. The main interest of the above theoretical results lies in the fact that they apply to an instance of the STE. Another theory needs to be developed to explain the good behavior of SEA for sparse support recovery when $A$ is coherent.

%To conclude, SEA performs well even when $A$ is coherent, see Section \ref{deconv-sec}. This is not explained by \cref{thm:recovery_RIP} and \cref{cor:recovery_SRIP} which use the RIP assumption. The main interest of the above theoretical results lies in the fact that they apply to an instance of the STE. Another theory needs to be developed to explain the good behavior of SEA for sparse support recovery when $A$ is coherent.

%Improving the theoretical analysis in these directions is left for the future. The current statements permit to see that SEA is a sound algorithm. To the best of our knowledge, this is the first time such guarantees are given for an algorithm based on the STE.

%% file: Experimental_analysis.tex
\section{Experimental Analysis}\label{expe-sec}

We compare SEA to state-of-the-art algorithms on two tasks in the noisy setting:  phase transition diagrams (\cref{dt-sec} and \cref{app:dt}) and spike deconvolution problems for signal processing (\cref{deconv-sec} and \cref{app:deconv}). 
For completeness, additional comparisons between SEA and state-of-the-art algorithms for linear and logistic regression tasks in supervised learning settings are provided in \cref{app:MLexp}.

The tested algorithms are Iterative Hard Thresholding (IHT)~\cite{Blumensath2009}, Hard Thresholding Pursuit (HTP)~\cite{foucart2011hard}, Orthogonal Matching Pursuit~\cite{mallat1993matching,pati1993orthogonal}, OMP with Replacement (OMPR)~\cite{jain2011orthogonal} and Exhaustive Local Search (ELS)~\cite{pmlr-v119-axiotis20a}. 
%(a.k.a. Fully Corrective Forward Greedy Selection with Replacement~\cite{shalev2010trading}).
OMPR and ELS are initialized with the solution of OMP.
Three versions of SEA are studied: the cold-start version \SEAZ, where SEA is initialized with the null vector, and the warm-start versions \SEAELS and \SEAOMP, where SEA is initialized with the solutions of ELS and OMP, respectively. We have also studied HTP and IHT initialized with OMP and ELS. They are called \HTPOMP, \HTPELS, \IHTOMP and \IHTELS.

For all algorithms, each least-square projection for a fixed support, as in Line~\ref{line:SEA:update_x} of \cref{alg:SEA}, is solved using the conjugate gradient descent of SciPy~\cite{2020SciPy-NMeth}.
For all algorithms, $256\spars$ iterations are performed.
The results of HTP and to a lesser extent IHT and SEA depend on the choice of the step size.
For the sake of fairness of the comparison with OMP, OMPR and ELS, we did not optimize the choice of the step size.
The step size of SEA, HTP, and IHT is arbitrarily\footnote{We do not report further experiments for $\eta\in\{\frac{2^l}{L}~|~ \mbox{for }l=\intint{-3,+3}\}$ that do not significantly alter the results in terms of running time, stability, performance, and do not impact our conclusions on phase transition diagrams. Variation of the step size for IHT and HTP in the deconvolution experiment is reported in \cref{app:deconv:step_size} and does not impact our conclusions.} fixed to $\eta=\frac{1.8}{L}$, where $L$ is the spectral radius of $\mat$.
The columns of $A$ are normalized before solving the problem.
%We do not report further experiments for $\eta\in\{\frac{2^l}{L}|$ for $l=-3,\cdots,+3\}$ which do not change significantly the results in terms of running time, stability, performance, and do not affect our conclusions. 
%The columns of $\mat$ are normalized before solving the problem.
The sparse vector $\sol \in \sigspace$ is random.
Indexes of the support are randomly picked, uniformly without replacement.
The non-zero entries of $\sol$ are drawn uniformly in $[-2, -1]\cup[1, 2]$ as in~\cite{Elad2010}.
The noise $\noise$ is drawn uniformly using the same method as described in \cite{blanchard2015performance}. Their detailed descriptions are in the next two sections.
For each experiment, the metrics used for performance evaluation are defined in the corresponding subsection.
The code is implemented in Python 3 and is available in the git repository of the project~\footnote{\url{https://gitlab.lis-lab.fr/valentin.emiya/sea-icml-2024}}.
%This will be replaced by the repository link in the final version of the paper.}.
As explained in~\cref{complexity-sec} and in~\cref{app:algorithms:sea_efficient}, the computational cost of SEA mainly depends on the number of explored supports.
The illustration related to the number of explored supports for a fixed number of iterations and the efficiency of the exploration can be found in \cref{app:deconv-n_support}.
%In the proposed experiments, the runtime of SEA is typically on the order of magnitude of ELS\@.
%\hspace{-0.01cm}

\subsection{Phase Transition Diagram Experiment}\label{dt-sec}

% Utiliser l'option pour ajuster la hauteur souhaitée (sans unité = nombre de lignes)
%\begin{wrapfigure}[18]{r}{0.6\textwidth}
%\vspace{-6mm}
\begin{figure}
    \centering
    \includegraphics[width=0.7\linewidth]{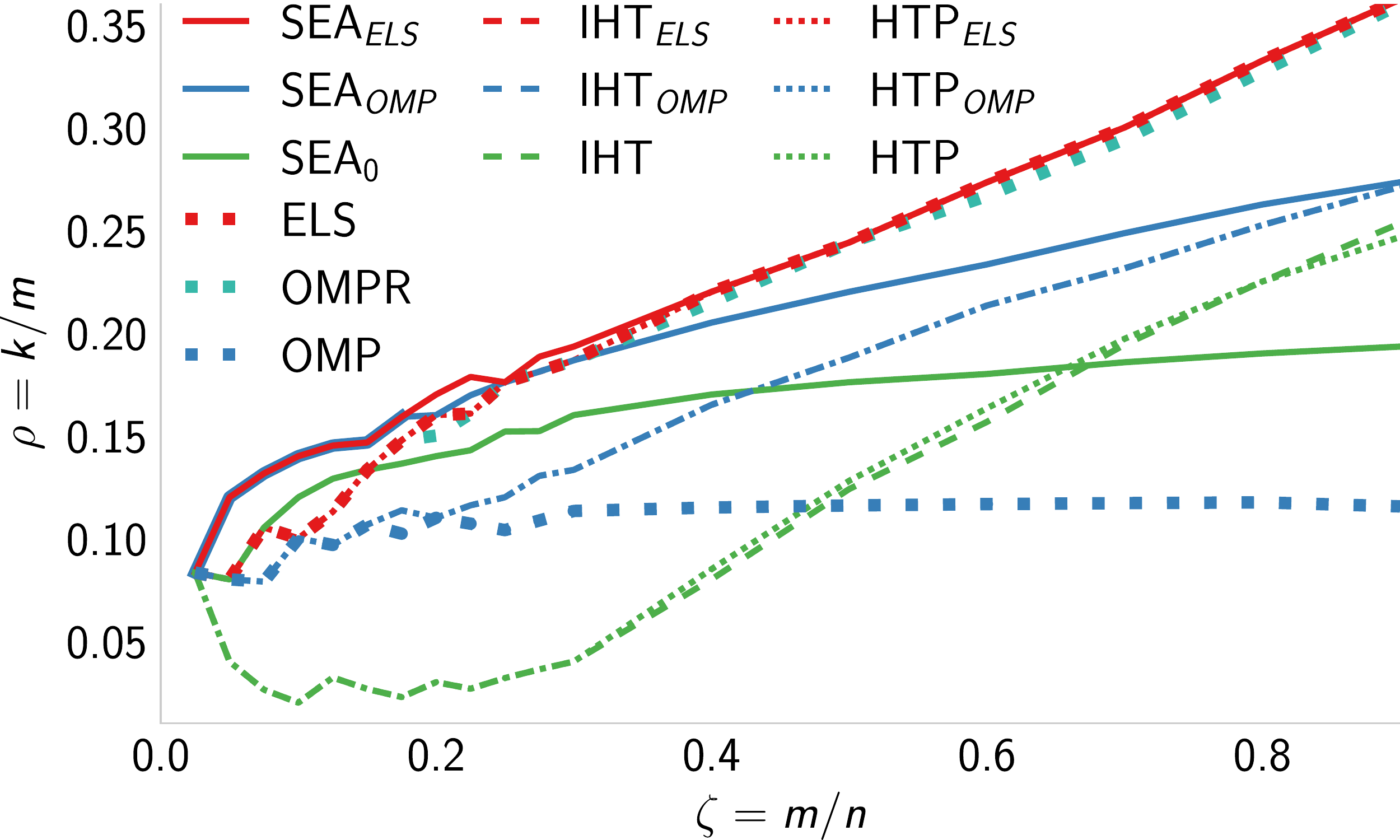} % V2.06
    \caption{Phase transition diagram: each curve is the threshold below which the related algorithm recovers at least $95\%$ of the supports. $\zeta$ denotes the ratio between the number of rows and the number of columns in $A$ while $\rho$ denotes the ratio between the sparsity and the number of rows in $A$. Matrix $A$ have i.i.d. standard Gaussian entries and non-zero entries in $\sol$ are drawn uniformly in $[-2, -1]\cup[1, 2]$. $n=500$ is fixed and results are obtained from $1000$ runs.}
    \label{fig:hm}
\end{figure}
%\end{wrapfigure}

Introduced by Donoho and Tanner~\cite{donoho2009observed} and used in compressed sensing \cite{donoho2009message, foucart2013invitation}, phase transition diagrams show the recovery limits of an algorithm depending on the undersampling/indeterminacy $\unsampf = \frac{\obssize}{\sigsize}$ of $\mat$, and the sparsity/density $\densef = \frac{\spars}{\obssize}$ of $\sol$.

We fix $\sigsize = 500$, $m$ takes $18$ values in $\intint{1, n}$ and $k$ all values in $\intint{1,0.5m}$.
For each triplet $(\obssize,\sigsize,\spars)$ and each algorithm, we run $\nruns = 1000$ experiments (described below) to assess the success rate $\frac{\success}{\nruns}$ of the algorithm, where $\success$ is the number of problems successfully solved. 
A problem is considered successfully solved if the support of the output of the algorithm is equal to $\suppsol$.
For each run, the entries of $\mat \in \matspace$ are drawn independently from the standard normal distribution.
The restricted isometry constants are poor when $\unsampf = \frac{\obssize}{\sigsize}$ is small and improve when $\obssize$ grows~\cite{bah2014bounds}.
The noise $\noise$ is drawn uniformly from the sphere of radius $0.01\normd{\mat\sol}$ in $\RR^m$.

\cref{fig:hm} shows results from this experiment.
Each curve indicates the threshold below which the algorithm has a success rate larger than $95\%$. The higher the curve, the better.
We see that OMP, HTP and IHT achieve poor recovery successes.
The smooth, decreasing part of the HTP and IHT curves on the left is an artifact due to the discrete values of $(m,n,k)$ and actually corresponds to a phase transition located at $k=1$.
%, which are only located at small values of sparsity $\spars$.
%\SEAZ is on par with OMP\@.
\SEAZ outperforms OMP, HTP and IHT when $\frac{m}{n} < 0.6$.
All the OMP-initialized algorithms (in blue) improve OMP performance except in the most coherent cases ($\frac{m}{n} < 0.2$) where \HTPOMP and \IHTOMP fail while \SEAOMP exhibits the best improvement.
Contrary to \HTPELS and \IHTELS, \SEAELS (in red) improves further ELS performances and outperforms the other algorithms for all $\frac{\obssize}{\sigsize}$.
The main improvements are when $\frac{\obssize}{\sigsize}$ is small ($\frac{m}{n} < 0.4$), i.e., for the most coherent matrices $\mat$.
Thus, SEA refines a good support candidate into a better one by exploring new supports and achieves recovery for higher values of sparsity $\spars$ than competitors.
The actual superiority of \SEAELS and \SEAOMP for coherent matrices ($\frac{m}{n} < 0.3$) is a major conclusion from this experiment and illustrates its ability to successfully explore supports in difficult problems where competitors fail.
We study the noiseless setup (i.e., $\noise = \zeroobs$) in \cref{app:dt}.
% \cref{fig:hm} shows results from this experiment.
% Each curve indicates the threshold below which the algorithm has a success rate larger than $95\%$.
% The higher the curve, the better.
% We see that HTP and IHT achieve poor recovery successes.
% The smooth, decreasing part of the related curves on the left is an artifact due to the discrete values of $(m,n,k)$ and actually corresponds to a phase transition located at $k=1$.
% %, which are only located at small values of sparsity $\spars$.
% \SEAZ is on par with OMP\@.
% \SEAOMP, OMPR and ELS ---and, to a lesser extent, HTP and IHT--- improve OMP performances, in particular, when $\frac{\obssize}{\sigsize}\geq 0.5$, i.e.\ when matrices $\mat$ are less coherent. Contrary to HTP and IHT, \SEAELS improves further ELS performances and outperforms the other algorithms for all $\frac{\obssize}{\sigsize}$.
% The main improvements are when $\frac{\obssize}{\sigsize}$ is small ($\frac{m}{n} < 0.6$), i.e., for the most coherent matrices $\mat$.
% Thus, SEA refines a good support candidate into a better one by exploring new supports and achieves recovery for higher values of sparsity $\spars$ than competitors.
% The actual superiority of \SEAELS and \SEAOMP for coherent matrices ($\frac{m}{n} < 0.6$) is a major conclusion from this experiment and illustrates its ability to successfully explore supports in difficult problems where competitors fail.
% We study the noisy setup (i.e., $\noise \neq \zeroobs$) in \cref{app:dt}.

\subsection{Deconvolution Experiment}\label{deconv-sec}

Deconvolution purposes arise in many signal processing areas such as microscopy or remote sensing.
Of particular interest is the deconvolution of sparse signals, also known as point source deconvolution~\cite{Bernstein2019} or spike deconvolution \cite{Duval2015,duval2017sparse}, assuming the linear operator is known (contrary to blind approaches \cite{pmlr-v97-kuo19a}). The objective is to recover spike positions and amplitudes.

We set $\sigsize = 500$, a convolution matrix $\mat$ corresponding to a Gaussian filter with a standard deviation equal to $3$.
The coherence of matrix $\mat$ is $\max_{i\neq j} |\mat_i^T \mat_j | = 0.97$, resulting in very difficult problems for which the support recovery theorems do not apply.
For each sparsity level $\spars\in\intint{1,50}$, every algorithm is tested on $\nruns = 200$ distinct problems corresponding to different $\spars$-sparse vectors $\sol$.
The maximal number of iterations is $1000$, for all algorithms.
The noise $\noise$ is drawn uniformly from the sphere of radius $0.1\normd{\mat\sol}$ of $\RR^m$, aiming for a signal-to-noise ratio of $20$dB.
Cases where sparsity is wrongly estimated and noise is stronger or applied differently are studied in \cref{app:deconv:wrong_k,app:deconv:noise_robustness,app:deconv:noise_on_x}.

\begin{figure}
    \centering
    \includegraphics[width=0.7\linewidth]{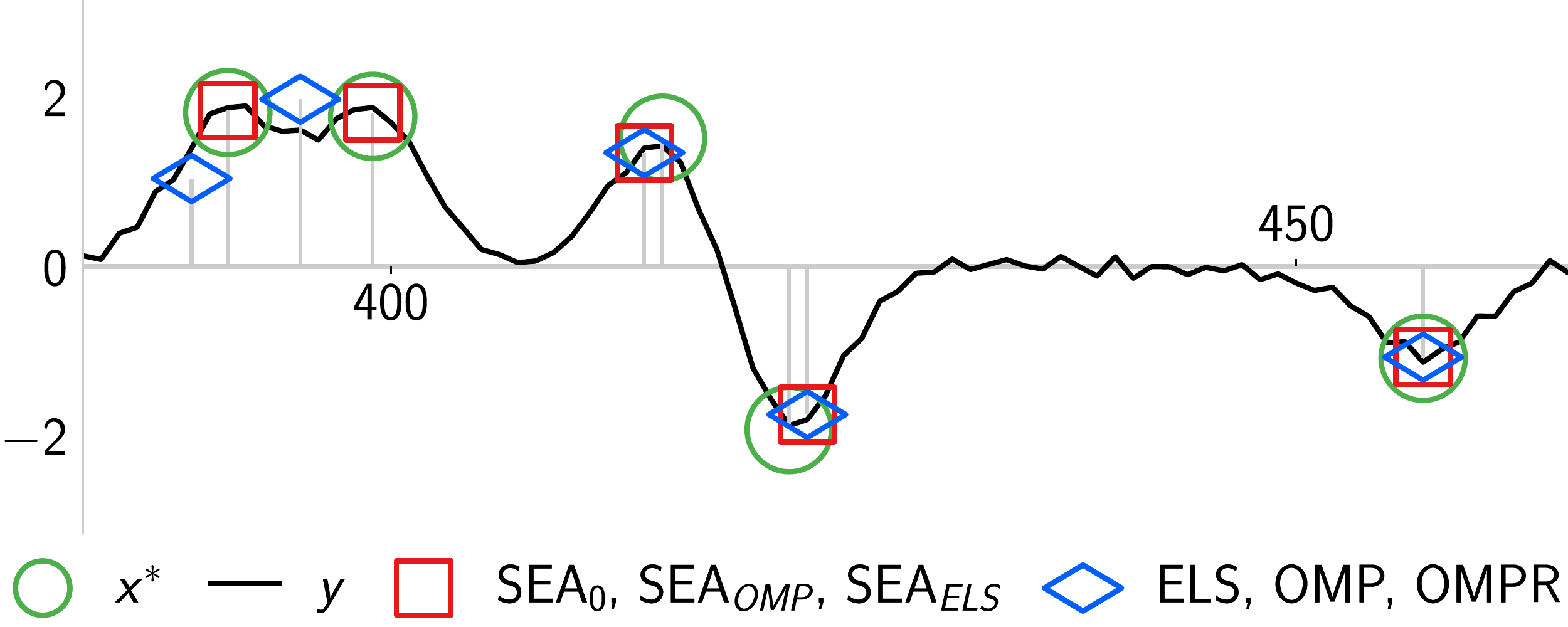}
    \caption{Spike deconvolution: representation of an instance of $\sol$ and $\obs$ with the solutions provided by the algorithms when $\spars = 20$. This is a cropped version of a crowded area (spikes are close).
    }%
    \label{fig:dcv_prec}%
\end{figure}

\cref{fig:dcv_prec} illustrates the results for a $20$-sparse vector $\sol$ restricted to a crowded area of the full signal (the later being depicted in \cref{app:dcv_prec}).
Generally speaking, isolated spikes are recovered by almost all algorithms. 
However, algorithms often fail to accurately identify spikes when they are close to each other.
For instance, ELS, OMP and OMPR falsely detect entries in the highest energetic part of the signal (around position 400) and are trapped in a local minimum. \SEAZ, \SEAOMP, and \SEAELS recover the original signal with a better precision than its competitors. 
It is worth mentioning that only SEA recovers perfectly this signal in the noiseless settings (see \cref{app:deconv:loss}).
To illustrate the exploratory behavior of SEA, we show in \cref{app:deconv-precise}, the evolution of $\normd{\mat \sigst-\obs}$ when $\iter$ and the number of explored supports varies, for the experiment of \cref{fig:dcv_prec}. 

\begin{figure}
    \centering
    \includegraphics[width=0.7\linewidth]{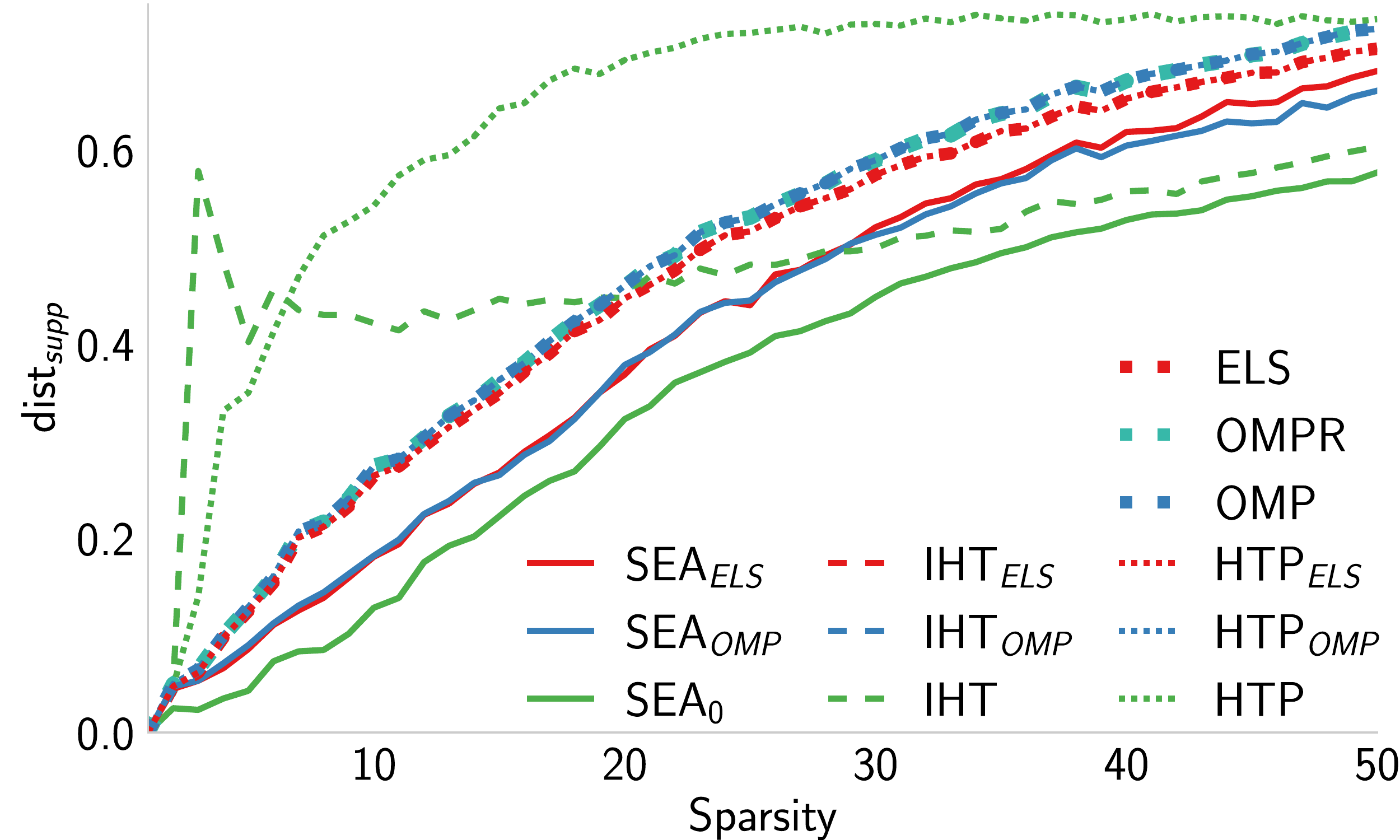}
    \caption{Spike deconvolution: average support distance between $\suppsol$ and the support of the solutions provided by several algorithms as a function of the sparsity level $\spars$.}
    \label{fig:dcv_mass}
\end{figure}

On \cref{fig:dcv_mass}, for each algorithm and for all $\spars\in\intint{1,50}$, we display the support distance metric \cite{Elad2010} averaged over $\nruns=200$ runs and defined by
\begin{equation}
    \text{dist}_{\text{supp}}(\sig) = \frac{\spars - \abs{\suppsol \cap \ \SUPP{\sig}}}{\spars}
    \label{eq:supp_dist}
\end{equation}
(the lower the distance, the better).
% \begin{equation}
%     \text{dist}_{\text{supp}}(\sig) = \frac{\spars - \abs{\suppsol \cap \ \SUPP{\sig}}}{\spars}.
%     \label{eq:supp_dist}
% \end{equation} 
For all considered sparsity values, \SEAZ, \SEAOMP, and \SEAELS outperform the other algorithms.
SEA improves OMP and ELS results while they are never enhanced by HTP nor IHT (curves are superimposed).
Note that for small $k$, IHT shows poor performance because it assigns several neighboring elements of the support to the largest peak of $y$ and fails to correct this error afterward.
As $\spars$ increases, due to the increasing difficulty of the problem, the algorithms are gradually becoming unable to recover $\suppsol$.
%\mm{\st{By exploring various supports, SEA finds better supports than its competitors} On average, SEA finds a support at a closer distance from the true support than its competitors by exploring a variety of supports.}
Using a cold-start strategy, \SEAZ is here the best performing algorithm. The analyses conducted in \cref{app:deconv-n_support} indicate that the exploration carried out by SEA can be more efficient than the support element swaps performed by ELS. These experiments also suggest that a warm-start strategy, such as \SEAELS or \SEAOMP, may lead the algorithm to get trapped in a local minimum. The choice of the best strategy appears to depend on the quality of the initialization. We recommend selecting it based on empirical performance. 
%However, we are currently unable to anticipate which strategy should be preferred since we find opposite conjectures in \cref{dt-sec}.
% We provide an analysis for additional metrics in \cref{app:deconv-global} and for the noiseless case in \cref{app-deconv-noiseless}. Both lead to the same conclusions.
The same conclusions are drawn when using additional metrics (\cref{app:deconv-global}) and in the noiseless case (\cref{app-deconv-noiseless}).

%% file: Discussion.tex
\section{Conclusions and Perspectives}
\label{discussion-sec}

In this article, we proposed SEA: a new principled algorithm for sparse support recovery, based on STE. Experiments show that SEA supplements state-of-the-art algorithms and outperforms them in particular when $A$ is coherent, thanks to its better exploration ability. Indeed, SEA initialized with the output of ELS is generally a good strategy to try to improve recovery results. Nonetheless, the cold-start strategy where SEA is initialized at $0$ may also be profitable: it is the best setting in problems with very coherent matrices like in the deconvolution experiment. Understanding which strategy should be preferred remains an open question.

We established guarantees when the matrix $A$ satisfies the RIP, which we hope gives new insight on the STE.
The theoretical guarantees involve conditions on $\sol$ that are not present for similar statements for other algorithms and that might restrict their applicability.
Improving the theoretical analysis in the following directions are promising perspective. 
The algorithm perform well when $A$ is coherent: this is not explained by the current theoretical analysis which only applies to matrices satisfying the RIP. Checking explicitly RIP conditions being NP-hard \cite{Tillmann2014}, we will investigate theoretical guarantees based on mutual incoherence~\cite{Donoho2001}.
Also it would be interesting to adapt the strategy developed for obtaining the theoretical guarantees to other contexts, such as the optimization of quantized neural networks. 

Finally, this paper opens up broader perspectives. The proposed STE is a deterministic approach for support exploration and may also be compared to or extended by the use of stochastic heuristics. Also, it would be of interest to study, either theoretically or numerically, the behavior of SEA in the compressed sensing setting. There are many perspectives of SEA and STE applications to sparse inverse problems such as sparse matrix factorization, tensor problems, as well as real-world applications such as in biology and astronomy.

%% file: Acknowledgement.tex
\section*{Acknowledgement}
% J'ai enlevé ANITI
F. Malgouyres gratefully acknowledges the support of IRT Saint Exupéry and the DEEL project\footnote{\url{https://www.deel.ai/}} and thanks Franck Mamalet for all the discussions on the STE.

M. Mohamed was supported by a PhD grant from ``Emploi Jeunes Doctorants (EJD)'' plan which is funded by the French institution ``Conseil régional Provence-Alpes-Côte d'Azur'' and Euranova France. M. Mohamed gratefully acknowledges their financial support.
% no acknowledgement pour Caro

%% file: Problem_statement.tex
\section{Problem Statement}
\label{app:problem_statement}
The equivalence between problem~\eqref{eq:constrained_sparse_problem} and problem~\eqref{eq:sea_problem} is established by the following proposition. Before stating the proposition, let us remind
\begin{equation}\label{eq:least_square_appendix}
F\left(\sig\right) = \frac{1}{2}\normdd{\mat\sig - \obs}\qquad, \forall \sig \in \sigspace
\end{equation}
and
\begin{equation}\label{eq:sparsification_operator:appendix}
H\left(\explo\right) \in
\underset{\substack{\sig \in \sigspace \\ \SUPP{\sig} \subseteq \largest_\spars\left(\explo\right) }}{\argmin} \frac{1}{2}\normdd{\mat\sig - \obs}\qquad, \forall \explo \in \sigspace.
\end{equation}
Let us also recall the optimization problem~\eqref{eq:constrained_sparse_problem}
\begin{equation}
\underset{\sig \in \sigspace}{\minimize} \, F\left(\sig\right) \text{ s.t. } \normz{\sig} \leq \spars
\label{eq:constrained_sparse_problem:appendix}
\end{equation}
and the optimization problem~\eqref{eq:sea_problem}
\begin{equation}
\underset{\explo \in \sigspace}{\minimize} \, F\left(H\left(\explo\right)\right).
\label{eq:sea_problem:appendix}
\end{equation}
\begin{proposition}[Optimization problem equivalence]
\label{thm:problem_statement}
For all $m,n,\spars \in \NN$, $\mat \in \matspace $ and $\obs\in\obsspace$. Problem \eqref{eq:constrained_sparse_problem:appendix} is equivalent to problem \eqref{eq:sea_problem:appendix}, in the sense that 
\begin{enumerate}
\item for any solution $\explo^* \in \argmin_{\explo \in \RR^n} F(H(\explo))$ of \eqref{eq:sea_problem:appendix},  $H(\explo^*)$ is solution of \eqref{eq:constrained_sparse_problem:appendix}. 
\item  for any minimizer $x'$ of \eqref{eq:constrained_sparse_problem:appendix}, we have $x' \in \argmin_{\explo \in \RR^n} F(H(\explo))$, i.e.,  $x'$ is solution of \eqref{eq:sea_problem:appendix}. 
\end{enumerate}

\end{proposition}

\begin{proof}
To establish the first item, we consider a solution $\explo^* \in \argmin_{\explo \in \RR^n} F(H(\explo))$  of \eqref{eq:sea_problem:appendix}. By definition of $H$, in \eqref{eq:sparsification_operator:appendix}, $H(\explo^*)$ is $k$-sparse. To prove that it minimizes \eqref{eq:constrained_sparse_problem:appendix}, consider $x\in\RR^n$ such that $\|x\|_0 \leq k$, we have
\begin{equation}\label{eqrongoun}
F(H(\explo^*)) \leq F(H(x)) \leq F(x),
\end{equation}
where the first inequality is due to the hypothesis on $\explo^*$, and the last inequality to the definition of $H$. Finally, since $H(\explo^*)$ is $k$-sparse and \eqref{eqrongoun} holds for all $k$-sparse vector $x$, we conclude that $H(\explo^*)$ is solution of \eqref{eq:constrained_sparse_problem:appendix}. 

To prove the second item, consider a minimizer $x'$ of \eqref{eq:constrained_sparse_problem:appendix} and $\explo\in\RR^n$. By definition of $H$, $H(\explo)$ is $k$-sparse. Using that $x'$ is solution of \eqref{eq:constrained_sparse_problem:appendix}, we therefore have
\begin{equation}\label{zrgonetg}
F(x') \leq F(H(\explo)).
\end{equation}
Moreover, since $x'$ is $k$-sparse, we have $\SUPP{\sig'} \subseteq \largest_\spars\left(x'\right)$, and by the definition of $H$,
\[F(H(x')) \leq F(x').
\]
Combining with \eqref{zrgonetg}, we obtain $F(H(x')) \leq F(H(\explo))$, for all $\explo\in\RR^n$,  and conclude that $x'$ is solution of \eqref{eq:sea_problem:appendix}.

\end{proof}

%% file: Algorithms.tex
\section{Additional Algorithms}
\label{app:algorithms}
In this appendix, more details are given about SEA pseudo-code: the main differences with state-of-the-art algorithms HTP and IHT are discussed in Section~\ref{app:algorithms:sota} and tricks for an efficient implementation of SEA are given in Section~\ref{app:algorithms:sea_efficient}.

\subsection{State-of-the-Art Algorithms}
\label{app:algorithms:sota}

In terms of pseudo-code, SEA looks similar to Hard Thresholding Pursuit (Algorithm~\ref{alg:HTP}, \citep{foucart2011hard}) and to a less extent to Iterative Hard Thresholding (Algorithm~\ref{alg:IHT}, \citep{Blumensath2009}).
In this section, we highlight the differences between these algorithms. In particular, in \cref{alg:SEA:bis}, \cref{alg:HTP} and \cref{alg:IHT} distinctions are pointed out in red.

Both HTP and IHT are projected descent algorithms that alternate a gradient step at a sparse estimate $\sigst$ and a projection of the resulting variable $\explo^\iter$ onto the set of sparse vectors.
The whole difference with SEA lies in the introduction of the support exploration variable $\explo^\iter$ and its interaction with the sparse vector $\sigst$.
HTP and IHT perform a regular gradient step $\explo^\iterp \leftarrow \sigs^\iter - \grad$ (where $\explo$ denotes an intermediate variable here, not a support exploration variable) while SEA uses an STE update $\explo^\iterp \leftarrow \explo^\iter - \grad$ of the support exploration variable itself ($\explo^\iter$) with a gradient computed at $\sigs^\iter$.
As a consequence, the vector $\explo^{t}$ in HTP or IHT is always one gradient step away from sparse vector $\sigs^{t}$. They do not explore much. This is not the case with SEA. The support exploration variable $\mathcal{X}^{t}$ is not expected to minimize the objective: it rather accumulates all the gradient iterates, where the gradient is computed at $x^t$. This is the whole point of the STE. In particular, $(\mathcal{X}^{t})_{t \in \mathbb{N}}$ is not restricted to a small portion of $\RR^n$ in the vicinity of sparse vectors. It can explore much more than in HTP and IHT. This explains why SEA has a different exploration/exploitation trade-off. It explores more. As can also be seen from the experiments in \cref{app:deconv:loss}, the loss oscillates a lot during SEA's iterative process, but SEA retains the best solution $x^{\iterb}$ encountered during the exploration. SEA is not based on a descent principle as IHT, HTP and such.
%Consequently, SEA results in a trajectory $(\mathcal{X}^{t})_{t\in\mathbb{N}}$ that passes in the region of the space such that $S^ * \subseteq largest_k(\mathcal{X}^{t})$. 

Finally, one may also notice that HTP stops as soon as the gradient is small enough such that the support does not change during two successive iterations. On the contrary, SEA keeps accumulating gradients so that the support may remain unchanged for many iterations before a new support is explored. This is clearly visible in the illustrations of \cref{app:deconv:loss}.
%This is more visible in the efficient implementation of SEA proposed in Appendix~\ref{app:algorithms:sea_efficient}, where the contrast between the pseudocodes of SEA and of HTP is more apparent.

%\todo{Garde-t-on IHT?}
\begin{multicols}{3}
\begin{algorithm}[H]
   \caption{SEA\\ (copy of Algorithm~\ref{alg:SEA})}
   \label{alg:SEA:bis}
\begin{algorithmic}[1] % Noise of the gradient
   % \STATE {\bfseries Input:} noisy observation $\obs$, sampling matrix $\mat$, sparsity $\spars$, step size $\gradstep$
    \STATE {\bfseries Inputs:}\\
    noisy observation $\obs$,\\
    sampling matrix $\mat$,\\
    sparsity $\spars$,\\
    step size $\gradstep$
   \STATE {\bfseries Output:} sparse vector $\sigs$
   \STATE Initialize $\explo^0$ 
   \STATE $\iter \leftarrow 0$
   \REPEAT
   \STATE $\supp^\iter \leftarrow \largest_\spars\left(\explo^\iter\right)$ 
   \alglinelabel{line:SEA:bis:update_S}
   % \STATE $\sigs^\iter_{\overline{\supp^\iter}} \leftarrow 0$
   % \STATE $\sigs^\iter_i \leftarrow 0$ for $i \in \overline{\supp^\iter}$
   % \STATE $\sigs^\iter_{\supp^\iter} \leftarrow \matdagsuppt \obs$ \alglinelabel{line:SEA:bis:update_x}
   \STATE $\begin{cases}
       \sigs^\iter_i & \leftarrow 0 \text{ for } i \in \overline{\supp^\iter}\\
       \sigs^\iter_{\supp^\iter} & \leftarrow \matdagsuppt \obs\\
   \end{cases}$ \alglinelabel{line:SEA:bis:update_x}
   % \STATE $\sigs^\iter \leftarrow \underset{\substack{\sig \in \sigspace \\ \SUPP{\sig} \subseteq \supp^\iter }}{\argmin} \normdd{\mat\sig - \obs}$ \alglinelabel{line:SEA:bis:update_x}
   \STATE $\explo^\iterp \leftarrow \textcolor{red}{\explo^\iter} - \grad$ \alglinelabel{line:SEA:bis:update_explo}
   \STATE $\iter \leftarrow \iterp$
   \UNTIL{halting criterion is $true$}
   \STATE \alglinelabel{line:SEA:bis:best} \textcolor{red}{$\iterb \leftarrow \underset{\itero \in \intint{0, \iter}}{\argmin} \normdd{\mat\sigs^\itero - \obs}$}
   \STATE {\bfseries return} $\textcolor{red}{\sigs^\iterb}$
\end{algorithmic}
\end{algorithm}

\begin{algorithm}[H]
    \caption{HTP\\ \citep{foucart2011hard}}
    \label{alg:HTP}
\begin{algorithmic}[1] % Noise of the gradient
    \STATE {\bfseries Inputs:}\\
    noisy observation $\obs$,\\
    sampling matrix $\mat$,\\
    sparsity $\spars$,\\
    step size $\gradstep$
    \STATE {\bfseries Output:} sparse vector $\sigs$
    \STATE Initialize $\explo^0$ 
    \STATE $\iter \leftarrow 0$
    \REPEAT
        \STATE $\supp^\iter \leftarrow \largest_\spars\left(\explo^\iter\right)$ \alglinelabel{line:HTP:update_S}
        % \STATE $\sigs^\iter_{\overline{\supp^\iter}} \leftarrow 0$
        % \STATE $\sigs^\iter_i \leftarrow 0$ for $i \in \overline{\supp^\iter}$
        % \STATE $\sigs^\iter_{\supp^\iter} \leftarrow \matdagsuppt \obs$ \alglinelabel{line:HTP:update_x}
        \STATE $\begin{cases}
            \sigs^\iter_i & \leftarrow 0 \text{ for } i \in \overline{\supp^\iter}\\
            \sigs^\iter_{\supp^\iter} & \leftarrow \matdagsuppt \obs\\
        \end{cases}$
        \alglinelabel{line:HTP:update_x}
        % \STATE $\sigs^\iter \leftarrow \underset{\substack{\sig \in \sigspace \\ \SUPP{\sig} \subseteq \supp^\iter }}{\argmin} \normdd{\mat\sig - \obs}$ \alglinelabel{line:HTP:update_x}
        \STATE $\explo^\iterp \leftarrow \textcolor{red}{\sigs^\iter} - \grad$ \alglinelabel{line:HTP:step_x}
        \STATE $\iter \leftarrow \iterp$
    \UNTIL{halting criterion is $true$}\\
    $\phantom{\iterb \leftarrow \underset{\itero \in \intint{0, \iter}}{\argmin} \normdd{\mat\sigs^\itero - \obs}}$
    \STATE {\bfseries return} $\textcolor{red}{\sigs^\iter}$
\end{algorithmic}
\end{algorithm}

\begin{algorithm}[H]
    \caption{IHT\\ \citep{Blumensath2009}}
    \label{alg:IHT}
\begin{algorithmic}[1] % Noise of the gradient
    \STATE {\bfseries Inputs:}\\
    noisy observation $\obs$,\\
    sampling matrix $\mat$,\\
    sparsity $\spars$,\\
    step size $\gradstep$
    \STATE {\bfseries Output:} sparse vector $\sigs$
    \STATE Initialize $\explo^0$ 
    \STATE $\iter \leftarrow 0$
    \REPEAT
        \STATE $\supp^\iter \leftarrow \largest_\spars\left(\explo^\iter\right)$ \alglinelabel{line:IHT:update_S}
        % \STATE $\sigs^\iter_{\overline{\supp^\iter}} \leftarrow 0$\alglinelabel{line:IHT:update_x}
        % \STATE $\sigs^\iter_i \leftarrow 0$ for $i \in \overline{\supp^\iter}$ %\alglinelabel{line:IHT:update_x}
        % \STATE \textcolor{red}{$\sigs^\iter_{\supp^\iter} \leftarrow \explo^\iter_{\supp^\iter}$ \alglinelabel{line:IHT:update_x}}
        \STATE $\begin{cases}
            \sigs^\iter_i & \leftarrow 0 \text{ for } i \in \overline{\supp^\iter}\\
            \textcolor{red}{\sigs^\iter_{\supp^\iter}} & \textcolor{red}{\leftarrow \explo^\iter_{\supp^\iter}}\\
        \end{cases}$
        \alglinelabel{line:IHT:update_x}
        % \STATE \textcolor{red}{$\sigs^\iter_i \leftarrow 0$} for $i \in \overline{\supp^\iter}$ \alglinelabel{line:IHT:update_x}
        \STATE $\explo^\iterp \leftarrow \textcolor{red}{\sigs^\iter} - \grad$ \alglinelabel{line:IHT:step_x}
        \STATE $\iter \leftarrow \iterp$
    \UNTIL{halting criterion is $true$}
    $\phantom{\iterb \leftarrow \underset{\itero \in \intint{0, \iter}}{\argmin} \normdd{\mat\sigs^\itero - \obs}}$
    \STATE {\bfseries return} $\textcolor{red}{\sigs^\iter}$
\end{algorithmic}
\end{algorithm}
\end{multicols}

\subsection{Efficient Implementation of SEA}
\label{app:algorithms:sea_efficient}

Algorithm~\ref{alg:SEA} is presented in a way that favors clarity and simplifies the theoretical analysis.
In practice, one can notice that if the support $\supp^\iter$ does not change (line~\ref{line:SEA:update_S}), then the sparse vector $\sigs^t$ and the gradient $\grad$ do not change either.
Algorithm~\ref{alg:SEA_efficient} is an equivalent pseudo-code for a computationally-efficient implementation.
The most expensive computations ---the sparse projection at line~\ref{line:SEA_efficient:update_x} and the gradient at line~\ref{line:SEA_efficient:grad}--- are only required when the support has never been explored before.
Also, the best sparse vector can be memorized on the fly (line~\ref{line:SEA_efficient:xbest}).
Hence, the remaining operations, that are performed at each iteration, have a low computational cost: support extraction (line~\ref{line:SEA_efficient:update_S}), search for a previous, identical support (line~\ref{line:SEA_efficient:if_unseen_support}) and STE update (line~\ref{line:SEA_efficient:update_explo}).
This computationally-efficient version of SEA has a larger spatial complexity due to the memorization of all the supports and gradients seen along the iterations. However, this overhead is limited since 1/ for each explored support, only two vectors are memorized, one of them being sparse; and 2/ the number of explored supports is generally much lower than the number of iterations. For instance, in the deconvolution experiment, on average, less than $1000$ vectors of size $500$ (including $500$ $k$-sparse vectors) are stored during the running time of SEA.

In addition, solving the (unconstrained) linear system $A_{S^t}^TA_{S^t}x_{S^t}=A_{S^t}^Ty$ can also be performed efficiently. While the pseudo-inverse is a convenient notation at line~\ref{line:SEA_efficient:update_x}, the solution may be obtained more efficiently, e.g. in $\mathcal{O}(\spars^2 n)$ to compute $A_{S^t}^TA_{S^t}$ and $A_{S^t}^Ty$, and apply the conjugate gradient algorithm.
This complexity is a worst-case scenario so in practice, the solution is generally obtained more quickly.

\def\allsupports{\boldsymbol{\mathcal{S}}}
\def\allg{\boldsymbol{{g}}}
\begin{algorithm}[tbh]
    \caption{Support Exploration Algorithm: efficient implementation \label{alg:SEA_efficient}
    }
\begin{algorithmic}[1] % Noise of the gradient
    \STATE {\bfseries Input:} noisy observation $\obs$, sampling matrix $\mat$, sparsity $\spars$, step size $\gradstep$
    \STATE {\bfseries Output:} sparse vector $\sigs$
    \STATE Initialize $\explo^0$ 
    \STATE $F^{BEST} \leftarrow +\infty$ 
    % \STATE $\supp^0 = \largest_\spars\left(\explo^0\right)$ 
    \STATE $\iter \leftarrow 0$
    \STATE $\allsupports \leftarrow \left\lbrace\right\rbrace, \allg \leftarrow \left\lbrace\right\rbrace$
    \REPEAT
        \STATE $\supp \leftarrow \largest_\spars\left(\explo^\iter\right)$  \alglinelabel{line:SEA_efficient:update_S}\\
        ~\\
        \COMMENT{Compute sparse vector and gradient only for unseen supports}
        % \IF{there exists $\tau<\iter$ such that $\supp^\tau=\supp^\iter$}
        \IF{$\supp \notin \allsupports$}\alglinelabel{line:SEA_efficient:if_unseen_support}
        % \STATE $\sigs^S_i \leftarrow 0$ for $i \in \overline{\supp}$
        % \alglinelabel{line:SEA_efficient:update_x0}
        % \STATE $\sigs^S_{\supp} \leftarrow \matdagsupp\obs$ \alglinelabel{line:SEA_efficient:update_x1}
        \STATE $\begin{cases}
            \sigs^S_i & \leftarrow 0 \text{ for } i \in \overline{\supp}\\
            \sigs^S_{\supp} & \leftarrow \matdagsupp\obs\\
        \end{cases}$
        \alglinelabel{line:SEA_efficient:update_x}
 %           \STATE $\sigs^\supp \leftarrow \underset{\substack{\sig \in \sigspace \\ \SUPP{\sig} \subseteq \supp}}{\argmin} \normdd{\mat\sigs - \obs}$ \alglinelabel{line:SEA_efficient:update_x}
            \STATE $loss^\supp \leftarrow \frac{1}{2}\normdd{\mat\sig^\supp - \obs}$\\~
            \STATE $g^\supp \leftarrow \gradstep \mata\left(\mat\sigs^\supp - \obs\right) $\alglinelabel{line:SEA_efficient:grad}\\
            ~\\
            \COMMENT{Memorize support and gradient}
            \STATE $\allsupports \leftarrow \allsupports \cup \left\lbrace \supp \right\rbrace$,  $\allg \leftarrow \allg \cup \left\lbrace g^S \right\rbrace$\\
            ~\\
            \COMMENT{Memorize best iterate}
            \IF{$loss^\supp < F^{BEST}$}
               \STATE $F^{BEST} \leftarrow loss^\supp$
               \STATE $\sigs^{BEST} \leftarrow \sigs^\supp$ \alglinelabel{line:SEA_efficient:xbest}
            \ENDIF
            % \STATE $\sigs^\iter \leftarrow \sigs^{\supp^\iter}$
            % \STATE $g^\iter \leftarrow g^\tau$
        % \ELSE
            % \STATE $\sigs^\iter \leftarrow \underset{\substack{\sig \in \sigspace \\ \SUPP{\sig} \subseteq \supp^\iter }}{\argmin} \normdd{\mat\sig - \obs}$ \alglinelabel{line:SEA_efficient:update_x}
            % \STATE $g^\iter \leftarrow \grad$\alglinelabel{line:SEA_efficient:grad}\\
            % ~\\
            % \COMMENT{Memorize best iterate}
            % \IF{$\normdd{\mat\sigs^\iter - \obs} < F_{BEST}$}
            %    \STATE $F_{BEST} \leftarrow \normdd{\mat\sigs^\iter - \obs}$
            %    \STATE $\iterb \leftarrow \iter$ \alglinelabel{line:SEA_efficient:tbest}
            % \ENDIF
        \ELSE
            \STATE Retrieve $g^\supp$ in $\allg$
        \ENDIF\\
        ~\\
        % \COMMENT{Update support exploration variable while support keeps unchanged}
        \COMMENT{Update support exploration variable}
        \STATE $\explo^\iterp \leftarrow \explo^\iter - g^\supp$
        \alglinelabel{line:SEA_efficient:update_explo}
        % \WHILE{$\supp^\iter = \largest_\spars\left(\explo^\iter\right)$}
        % \alglinelabel{line:SEA_efficient:while_S_unchanged}
        %     \STATE $\explo^\iterp \leftarrow \explo^\iterp - g^\iter$ \alglinelabel{line:SEA_efficient:update_explo}
        % \ENDWHILE
    \STATE $\iter \leftarrow \iterp$
    \UNTIL{halting criterion is $true$}
    % \STATE \alglinelabel{line:SEA_efficient:best} $\iterb \leftarrow \underset{\itero \in \intint{0, \iter}}{\argmin} \normdd{\mat\sigs^\itero - \obs}$
    \STATE {\bfseries return} $\sigs^{BEST}$
\end{algorithmic}
\end{algorithm}

%% file: RIP_Proof.tex
\section{Proofs and Complements of the Theoretical Analysis}
\label{app:proof_RIP}

The proof of \cref{thm:recovery_RIP} relies on the fact that when the hypotheses of the theorem are satisfied, the trajectory $(\explo^t)_{t\in\NN}$ is close to the trajectory of an algorithm that has access to an oracle update. The appendix first contains a description and an analysis of this algorithm in \cref{oracle-sec}. Then, in \cref{subsec:genrecovery}, we analyze how much the STE-update can deviate from the Oracle Update Rule so that the true support $\suppsol$ is still recovered. Finally, we prove \cref{thm:recovery_RIP} in \cref{proof-thm-RIP-sec}. We prove \cref{cor:recovery_SRIP} in \cref{app:proof_SRIP} and conclude with comments on the conditions \eqref{eq:HRIP} and \eqref{eq:HSRIP} in \cref{comments-RIP-sec}.

\input{oracle}

\input{General_Proof}

\input{Orthogonal_Proof}

%%%%%%%%%%%%%%%
% Debut de la preuve dans le cas RIP
\subsection{Proof of \texorpdfstring{\cref{thm:recovery_RIP}}{Theorem \ref{thm:recovery_RIP}}}\label{proof-thm-RIP-sec}
We remind in \cref{subsec:RIC} known properties of RIP matrices. We bound in \cref{subsec:lemmas} the error made when approximating $\sol$ on a specific support $\supp$. This permits us to bound $\gradnoise^\iter$ and apply \cref{thm:recovery} to prove \cref{thm:recovery_RIP} in \cref{subsec:recovery_RIP}.  We finally apply \cref{thm:recovery_RIP} in \cref{app:proof_SRIP} to prove \cref{cor:recovery_SRIP}. Before that and to illustrate and quantify that the STE-update  $\explo^\iterp \leftarrow \explo^\iter - \grad$ is a noisy version of the Oracle update $\explo^{\iter+1} \leftarrow\explo^\iter - \usolt $, we provide in the following theorem an upper bound on the discrepancy between the two updates. This bound is pivotal in the proof of \cref{thm:recovery_RIP}. The statement of \cref{thm:recovery_RIP} is, up to the additional upper-bound \eqref{eruneqruvyn}, the same as the statement of the following theorem, which we prove in this section.

\begin{theorem}[Recovery - RIP case]
\label{thm:recovery_RIP:app}
Assume $\mat$ satisfies the $(2\spars+1)$-RIP and for all $\ii \in \intintn$, $\normd{\mat_\ii} = 1$. Then, for all $t\in\NN$, 
\begin{equation}\label{eruneqruvyn}
\|\frac{u^t}{\eta} - A^T(Ax^t-y) \|_\infty\leq \alphaRIP \|x^*\|_2 + \gammaRIP \|e\|_2.
\end{equation}

% \[\gradnoisemax \leq \eta (\alphaRIP \|x^*\|_2 + \gammaRIP \|e\|_2).
% \]
If moreover $\sol$ satisfies (\ref{eq:HRIP}), then for all initializations $\explo^0$ and all $\eta>0$, there exists $t_s\leq \itermaxRIP$ such that $\suppsol \subseteq S^{t_s}$, where
\begin{equation}
\label{eq:TRIP:app}
    \itermaxRIP =\frac{2 k \frac{\|\explo^0\|_{\infty}}{\eta} + (k+1) \min_{i\in S^*} |x_i^*| }{\min_{i\in S^*} |x_i^*| - 2k \left(\alphaRIP\normd{\sol} + \gammaRIP\normd{\noise}\right)}.
\end{equation}

If moreover, $x^*$ is such that 
\begin{equation}\label{min_hyp_RIP:app}
    \min_{i\in S^*} |x^*_i| > \frac{2}{\sqrt{1-\delta_{2k}}} \|e\|_2 
\end{equation}
and SEA performs more than $\itermaxRIP$ iterations, then $S^* \subseteq S^{\iterb}$ and $\|x^{\iterb} - x^*\|_2 \leq \frac{2}{\sqrt{1-\delta_{k}}} \|e\|_2$.
% \[S^* \subseteq S^{\iterb} ~~~~ \mbox{and}~~~~ \|x^{\iterb} - x^*\|_2 \leq \frac{2}{\sqrt{1-\delta_{k}}} \|e\|_2.
% \]
\end{theorem}

\subsubsection{Reminders on Properties of RIP Matrices}
\label{subsec:RIC}

We first remind the definition of Restricted Isometry Constant in (\ref{eq:RIP}) and a few properties of RIP matrices.

\begin{description}
    \item[Fact 1:]
    For any $\spars, \sparssec \in \intintn$, such that $\spars \leq \sparssec$, we have
    \begin{equation}
        \label{eq:RIPkk}
        \RIPk \leq \RIPksec.
    \end{equation}
    
    \item[Fact 2:]
    For any $\suppsec, \supp \subseteq \intintn$, such that $\suppsec \cap \supp = \text{\O}$ and $A$ satisfies the $(\abs{R}+\abs{S})$-RIP. We remind Lemma 1 of \cite{dai2009subspace} (see also \cite{candes2005decoding}): For any $\sig \in \sparsspace$
    \begin{equation}
    \label{eq:RIPsetset}
        \normd{\mata_\suppsec \matsupp \sig} \leq \RIP_{\abs{\suppsec} + \abs{\supp}}~ \normd{\sig} .
    \end{equation}
    For completeness, we prove this inequality below. Let $A$, $R$, $S$ and $x$ be as above, we have,
    \[  \|A_R^T A_S \frac{x}{\|x\|_2}\|_2 = \max_{x' : \|x'\|_2=1} \langle x' ,A_R^T A_S \frac{x}{\|x\|_2} \rangle.
    \]
    Using $\langle x' ,A_R^T A_S \frac{x}{\|x\|_2} \rangle =\langle A_Rx' , A_S \frac{x}{\|x\|_2} \rangle \leq \frac{1}{2} \|A_Rx' + A_S \frac{x}{\|x\|_2}\|_2^2 $, the fact that $\suppsec \cap \supp = \text{\O}$, and that $A$ satisfies the $(\abs{R}+\abs{S})$-RIP defined in \eqref{eq:RIP}, we obtain for all $x'\in\RR^{\abs{R}}$ such that $\|x'\|_2=1$
    \[\langle x' ,A_R^T A_S \frac{x}{\|x\|_2} \rangle\leq \frac{1}{2} (1+\delta_{\abs{R}+\abs{S}}) \left(\|x'\|_2^2 + \|\frac{x}{\|x\|_2} \|_2^2\right) = (1+\delta_{\abs{R}+\abs{S}}) . \] 
    This concludes the proof of \eqref{eq:RIPsetset}.
    \item[Fact 3:]  Let us assume that $A$ satisfies the $\abs{\supp}$-RIP. Taking inspiration of Proposition 3.1 of \cite{needell2009cosamp}, for any singular value $\vp \in \field$ of $\matsupp$, and the corresponding right singular vector $\Vp \in \sparsspace$, we have $\normd{\matsupp\Vp} = \vp$. 
    Using (\ref{eq:RIP}), $1 - \RIPsupp \leq \vp^2 \leq 1 + \RIPsupp$. 
    All singular values of $\matsupp$ and $\matasupp$ lie between $\sqrt{1-\RIPsupp}$ and $\sqrt{1+\RIPsupp}$.
  
    As a consequence, for any $\siga \in \obsspace$, we have
    \begin{equation}
        \label{eq:RIPtranspose}
        \normd{\matasupp\siga} \leq \sqrt{1+\RIPsupp} ~ \normd{\siga}.
    \end{equation}
    
    \item[Fact 4:]
    Let us assume that $A$ satisfies the $\abs{\supp}$-RIP. 
    Using the same reasoning, we find that the eigenvalues of $\matasupp\matsupp$ lie between $1-\RIPsupp$ and $1+\RIPsupp$.
    This implies that $\matasupp\matsupp$ is non-singular and that the eigenvalues of $(\matasupp\matsupp)^{-1}$ lie between $\frac{1}{1 + \RIPsupp}$ and $\frac{1}{1 - \RIPsupp}$. 
    Then $\matsupp$ is full column rank and for any $\sig \in \sparsspace$
    \begin{equation}
        \label{eq:RIPinverse}
        \normd{(\matasupp\matsupp)^{-1}\sig} \leq \frac{1}{1-\RIPsupp}\normd{\sig}.
    \end{equation}
    
    \item[Fact 5:]
    Let us assume that $A$ satisfies the $\abs{\supp}$-RIP.
    By using one last time the same reasoning, we find that the eigenvalues of $\matasupp\matsupp - \id_{\abs{\supp}}$ lie between $-\RIPsupp$ and $\RIPsupp$. Finally, for any $\sig \in \sparsspace$,
    \begin{equation}
        \label{eq:RIPid}
        \normd{(\matasupp\matsupp - \id_{\abs{\supp}})\sig} \leq \RIPsupp\normd{\sig}.
    \end{equation}
\end{description}

\subsubsection{Preliminaries}
\label{subsec:lemmas}
In this section, the facts from \cref{subsec:RIC} are used to bound from above the error $\normd{\truncv{(\sigst - \sol)}{\suppt}}$, where $\truncv{(.)}{\suppt}$ is the restriction of the vector to the support $\suppt$ and $\suppt$ is defined in Algorithm \ref{alg:SEA}, line \ref{line:SEA:update_S}. 
This bound will lead to an upper bound on $\normd{\gradnoise^\iter}$.
Throughout the section, we assume $\mat$ satisfies the $(2\spars+1)$-RIP.
Figure \ref{fig:sea-supports} might help visualize the different sets of indices considered in the proof.
\begin{lemma}
\label{lem:approximation_error_RIP}
    If $\mat$ satisfies the $(2\spars+1)$-RIP, for any $\iter \in \NN$, %$\supp \subseteq \intintn$ such that $\abs{\supp} \leq \spars$,
    \begin{equation*}
        \normd{\truncv{(\sigst - \sol)}{\suppt}}
        \leq
        \frac{\RIPdk}{1-\RIPk}\frac{\normd{\usolt}}{\gradstep} + \frac{\sqrt{1+\RIPk}}{1-\RIPk}\normd{\noise}.
    \end{equation*}
\end{lemma}

\begin{proof}
    For any $\iter \in \NN$% and $\supp \subseteq \intintn$
    , using the definition of $\sigst$ in \cref{alg:SEA} and (\ref{eq:usol}), we find 
    \begin{align}
        \truncv{\sigst}{\suppt} &= \matdagsuppt\obs \nonumber\\
        &= \matdagsuppt(\rest{\sol}{\suppsol} + \noise) \nonumber\\
        &= \matdagsuppt \rest{\sol}{\suppcap} - \frac{1}{\gradstep} \matdagsuppt \rest{\usolt}{\suppx} + \matdagsuppt \noise \label{eq:sigst_before_RIP}.
    \end{align}
    
    We also have
    \begin{align}
        \matdagsuppt \rest{\sol}{\suppcap} 
        &= \matdagsuppt 
        \begin{bmatrix}
            \mat_{\suppcap} & \mat_{\suppt \setminus \suppsol}
        \end{bmatrix}
        \begin{bmatrix}
            \truncv{\sol}{\suppcap}\\
            0
        \end{bmatrix}\nonumber\\
        &= \matdagsuppt \rest{\sol}{\suppt}. \label{eq:a_dag_min_error_support}
    \end{align}
    
    Since $\RIPdkp < 1$, (\ref{eq:RIPkk}) implies that $\RIPk \leq \RIPdkp < 1$ and the smallest singular value of $\matsuppt$ is larger than $\sqrt{1 - \RIPk} \geq \sqrt{1 - \RIPdkp} >0$. Therefore $\matsuppt$ is full column rank and
    \begin{equation}
    \label{eq:matadag}
        \matdagsuppt = (\matasuppt\matsuppt)^{-1}\matasuppt.
    \end{equation}
    
    Combining (\ref{eq:sigst_before_RIP}), (\ref{eq:a_dag_min_error_support}) and (\ref{eq:matadag}), we obtain
    \begin{equation*}
        \truncv{\sigst}{\suppt} = \truncv{\sol}{\suppt} - \frac{1}{\gradstep} \matdagsuppt \rest{\usolt}{\suppx} + \matdagsuppt\noise.
    \end{equation*}
    
    Using (\ref{eq:matadag}), we find
    \begin{align*}
        \normd{\truncv{(\sigst - \sol)}{\suppt}} 
        &= \normd{\frac{1}{\gradstep} \matdagsuppt \rest{\usolt}{\suppx} - \matdagsuppt\noise}\\
        &\leq \frac{1}{\gradstep} \normd{(\matasuppt\matsuppt)^{-1}\matasuppt\rest{\usolt}{\suppx}} + \normd{(\matasuppt\matsuppt)^{-1}\matasuppt\noise}.
    \end{align*}
    
    Finally, using (\ref{eq:RIPinverse}), then (\ref{eq:RIPsetset}), (\ref{eq:RIPkk}), (\ref{eq:RIPtranspose}) and (\ref{eq:usol}), we finish the proof
    \begin{align*}
        \normd{\truncv{(\sigst - \sol)}{\suppt}}
        &\leq \frac{1}{1 - \RIPk}\left(\frac{1}{\gradstep} \normd{\matasuppt\rest{\usolt}{\suppx}} + \normd{\matasuppt\noise}\right) \\
        & \leq \frac{1}{1 - \RIPk}\left(\frac{\RIPdk}{\gradstep} \normd{\truncv{\usolt}{\suppx}} + \sqrt{1+\RIPk}\normd{\noise}\right)  \\
        &= \frac{\RIPdk}{1-\RIPk}\frac{\normd{\usolt}}{\gradstep} + \frac{\sqrt{1+\RIPk}}{1-\RIPk}\normd{\noise}. 
    \end{align*}
\end{proof}

We have the following upper bound on $\normd{\gradnoise^\iter}$. This bound is given in \cref{thm:recovery_RIP:app}.

\begin{lemma}[Bound of $\gradnoise^\iter$ - RIP case]
    \label{lem:bound_gradnoise_RIP}
    If $\mat$ satisfies the $(2\spars+1)$-RIP, for any $\iter \in \NN$,
    \begin{equation*}
        \|\frac{u^t}{\eta} - A^T(Ax^t-y) \|_\infty
        =
        \frac{1}{\gradstep} \norminf{\gradnoise^\iter}
        \leq
        \alphaRIP\normd{\sol} + \gammaRIP\normd{\noise},
    \end{equation*}
    
    where $\alphaRIP$ and $\gammaRIP$ are defined in~\eqref{eq:alpha_gammaRIP}.
\end{lemma}

\begin{proof}
Let $\iter \in \NN$ and $\ii \in \intintn$, reminding the definition of $\gradnoise^\iter$ in (\ref{eq:gradnoise}), we have
\begin{align}
\label{eq:gradnoise_eq}
    \abs{\gradnoise_\ii^\iter} &= \abs{\usolt_\ii - \gradstep\matia(\mat\sigst - \obs)} \nonumber\\
    &= \abs{\usolt_\ii - \gradstep\matia\mat(\sigst - \sol) + \gradstep\matia\noise}\\
    &= \abs{\usolt_\ii - \gradstep\matia\rest{(\sigst - \sol)}{\suppcup} + \gradstep\matia\noise}.\nonumber
\end{align}

We distinguish three cases: $\ii  \in \suppt$, $\ii \in \suppz$ and $\ii \in \suppx$. 
We prove that in the three cases 
\begin{equation}
\label{eq:gradnoise_eq_int}
    \abs{\gradnoise_\ii^\iter}
    \leq
    \gradstep \left(\RIPdkp\normd{\sigst - \sol} + \normd{\noise}\right).
\end{equation}

\underline{$1^{st} \, case$}: If $\ii \in \suppt$, using \eqref{romnrevjnj}, $\gradnoise_\ii^\iter = 0$ and \eqref{eq:gradnoise_eq_int} holds.
%If $\ii \in \suppt$, because of the definitions of $\usolt$ and $\sigst$, $\gradnoise_\ii^\iter = 0$ and (\ref{eq:gradnoise_eq_int}) holds.

\underline{$2^{nd} \, case$}: If $\ii \in \suppz$, using the definition of $\usolt$ in (\ref{eq:usol}), (\ref{eq:RIPsetset}), (\ref{eq:RIPkk}) and the fact that $\normd{\mat_\ii} = 1$ we obtain
\begin{align*}
    \abs{\gradnoise_\ii^\iter}
    &= \abs{-\gradstep\matia\rest{(\sigst - \sol)}{\suppcup} + \gradstep\matia\noise}  \\
    &\leq \gradstep\left(\normd{\matia\rest{(\sigst - \sol)}{\suppcup}} + \normd{\matia\noise}\right)  \\
    &\leq \gradstep\left(\RIPdkp\normd{\truncv{(\sigst - \sol)}{\suppcup}} + \normd{\noise}\right)\\
    &=
    \gradstep \left(\RIPdkp\normd{\sigst - \sol} + \normd{\noise}\right)
\end{align*}

\underline{$3^{rd} \, case$}: If $\ii \in \suppx$, reminding that $\iib$ is the complement of $\{\ii\}\subseteq\intintn$, and since $\matia \mat_{\{\ii\}}= \|\mat_\ii\|_2 = 1$ and $x_i^t=0$, \eqref{eq:gradnoise_eq} becomes
%If $\ii \in \suppx$, reminding that $\iib$ is the complement of $\{i\}\subseteq\intintn $, (\ref{eq:gradnoise_eq}) becomes
\begin{align*}
    \abs{\gradnoise_\ii^\iter}
    &= \abs{-\gradstep\soli - \gradstep\matia\mat(\sigst - \sol) + \gradstep\matia\noise}\\
    &= \gradstep\abs{-\soli- \matia \rest{(\sigst - \sol)}{\{i\}} -   \matia\rest{(\sigst - \sol)}{\iib} + \matia\noise} \\
    &= \gradstep\abs{-\soli- (\sigst_i - \soli)-   \matia\rest{(\sigst - \sol)}{\iib} + \matia\noise} \\
    &= \gradstep\abs{-\matia\rest{(\sigst - \sol)}{\iib} + \matia\noise} \\
    &\leq \gradstep \left(\abs{\matia\rest{(\sigst - \sol)}{\iib}} + \abs{\matia\noise}\right).
\end{align*}
Using (\ref{eq:RIPsetset}), (\ref{eq:RIPkk}) and $\normd{\mat_\ii} = 1$, we obtain
\begin{align*}
    \abs{\gradnoise_\ii^\iter}
    &\leq \gradstep \left(\RIPdk\normd{\sigst - \sol} + \normd{\noise}\right)\\
    &\leq \gradstep \left(\RIPdkp\normd{\sigst - \sol} + \normd{\noise}\right).
\end{align*}
Regrouping the three cases, we conclude that for all $\ii \in \intintn$, (\ref{eq:gradnoise_eq_int}) holds.
We now finish the proof.

Using (\ref{eq:usol}) followed by \cref{lem:approximation_error_RIP}, we find
\begin{align*}
    \abs{\gradnoise_\ii^\iter}
    &\leq \gradstep\left(\RIPdkp\left(\normd{\truncv{(\sigst - \sol)}{\suppt}} + \frac{\normd{\usolt}}{\gradstep}\right) + \normd{\noise}\right)  \\
    &\leq \gradstep \left( \alphaRIPf\frac{\normd{\usolt}}{\gradstep} + \left(\gammaRIPf\right)\normd{\noise} \right)\\
    &\leq \gradstep\left(\alphaRIP\normd{\sol} + \gammaRIP\normd{\noise}\right), 
\end{align*}
where the last inequality holds because $\frac{\normd{\usolt}}{\gradstep} \leq \normd{\sol}$.
\end{proof}

\subsubsection{End of the Proof of \texorpdfstring{\cref{thm:recovery_RIP:app}}{Theorem~\ref{thm:recovery_RIP:app}}}
\label{subsec:recovery_RIP}
We now resume to the proof of \cref{thm:recovery_RIP:app} and assume $\mat$ satisfies the $(2\spars+1)$-RIP and $\sol$ satisfies (\ref{eq:HRIP}). We remind the definitions of $\itermax$ in (\ref{eq:Tmax}) and $\itermaxRIP$ in (\ref{eq:TRIP}).

Using (\ref{eq:gradnoisemax}) and \cref{lem:bound_gradnoise_RIP}, we have 
\begin{equation}
\label{eq:gradnoisemax_RCRIP}
    \gradnoisemax = \gradnoisemaxf \leq \gradstep\left(\alphaRIP\normd{\sol} + \gammaRIP\normd{\noise}\right).
\end{equation}

Combined with (\ref{eq:HRIP}), that is $\gammaRIP \norm{\noise}_2 < \frac{\min_{i\in S^*} |x_i^*|}{2 k} - \alphaRIP \norm{\sol}_2$, this implies that  
\begin{equation*}
    \gradnoisemax 
    < \frac{\gradstep}{2k} \min_{i\in \suppsol} |x_i^*|.
\end{equation*}

Therefore (\ref{eq:Hr}) holds and \cref{thm:recovery} implies that there exists $t_s\leq \itermax$ such that $S^* \subseteq S^{t_s}$, with
\[ \itermax = \frac{2 k \|\explo^0\|_{\infty} + (k+1)\eta \min_{i\in S^*} |x_i^*| }{\eta \min_{i\in S^*} |x_i^*| - 2 k \gradnoisemax }.
\]

Using \eqref{eq:gradnoisemax_RCRIP}, we obtain
\[\itermax \leq \frac{2 k \|\explo^0\|_{\infty} + (k+1)\eta \min_{i\in S^*} |x_i^*| }{\eta \min_{i\in S^*} |x_i^*| - 2k \gradstep\left(\alphaRIP\normd{\sol} + \gammaRIP\normd{\noise}\right)} = T_{RIP}.
\]

%For $a= 2 \sumHr >0$ and $b=\sumi \deltafrac + \spars + 1 >0$, the function $u \mapsto f_{a,b}(u) \triangleq \frac{b}{1 - a u}$ is non-decreasing on $[0, \frac{1}{a})$. Moreover, (\ref{eq:gradnoisemax_RCRIP}) and (\ref{eq:HRIP}) imply that
%\[0\leq \gradnoisemax \leq \gradstep\left(\alphaRIP\normd{\sol} + \gammaRIP\normd{\noise}\right) < \gradstep \frac{1}{a}.
%\]
%and therefore $\itermax = f_{a,b}(\frac{\gradnoisemax}{\gradstep})\leq f_{a,b}(\alphaRIP\normd{\sol} + \gammaRIP\normd{\noise}) = \itermaxRIP$. Therefore, since $t_s\leq T_{max}$, $t_s \leq \itermaxRIP$. As a conclusion, there exists $t_s \leq \itermaxRIP$ such that $S^* \subseteq S^{t_s}$.

We still need to prove that, when $ \min_{i\in S^*} |x^*_i| > \frac{2}{\sqrt{1-\delta_{2k}}} \|e\|_2 $, $\iterb$ satisfies $S^*\subseteq S^\iterb$, as well as the last upper-bound of \cref{thm:recovery_RIP:app} .

Assume by contradiction that 
\begin{equation}\label{ronetbuqerv}
    \min_{i\in S^*} |x^*_i| > \frac{2}{\sqrt{1-\delta_{2k}}} \|e\|_2
\end{equation} 
holds but $S^*\not\subset S^\iterb$. The construction of $\iterb$, in line \ref{line:SEA:best} of \cref{alg:SEA}, and the existence $t_s$ such that $S^*\subseteq S^{\iter_s}$ guarantee that
\[\|A x^\iterb - y \| \leq \|A x^{\iter_s} - y \| \leq \|A x^* - y \| = \|e\|_2.
\]
Therefore, using the left inequality in \eqref{eq:RIP}, we obtain
\begin{eqnarray*}
\sqrt{1-\delta_{2k}}\|x^\iterb - x^*\|_2 &\leq & \|A(x^\iterb - x^*)\|_2 \\
& \leq & \|A x^\iterb - y \|_2+\|A x^* - y \|_2\\
& \leq & 2 \|e\|_2.
\end{eqnarray*}
On the other hand, since we assumed $S^*\not\subset S^\iterb$ we have
\[\|x^\iterb - x^*\|_2 \geq \min_{i\in S^*} |x^*_i|.
\]
We conclude that $\min_{i\in S^*} |x^*_i| \leq \frac{2}{\sqrt{1-\delta_{2k}}} \|e\|_2$ which contradicts \eqref{ronetbuqerv}.

As a conclusion, when $ \min_{i\in S^*} |x^*_i| > \frac{2}{\sqrt{1-\delta_{2k}}} \|e\|_2 $, we have $S^*\subseteq S^\iterb$.

In this case, since the support of $x^\iterb - x^*$ is of size smaller than $k$, we can redo the above calculation and obtain
\[\sqrt{1-\delta_{k}}\|x^\iterb - x^*\|_2 \leq \|A(x^\iterb - x^*)\|_2 \leq 2 \|e\|_2.
\]
This leads to the last inequality of \cref{thm:recovery_RIP:app} and concludes the proof.

\subsection{Proof of \texorpdfstring{\cref{cor:recovery_SRIP}}{Corollary~\ref{cor:recovery_SRIP}}}
\label{app:proof_SRIP}

We assume that $\sol$ satisfies (\ref{eq:HSRIP}) and that $\normd{\noise} = 0$. Let us first prove that $\sol$ satisfies (\ref{eq:HRIP}). Using (\ref{eq:HSRIP}) we have
\begin{align*}
    0 < 1 - \RIPthresh 
    &\leq 1 - \RIPthreshf \nonumber\\
    &= \frac{2\spars}{\minsol} \left( \frac{\minsol}{2\spars} - \alphaRIP\normd{\sol}  \right).
\end{align*}

As a consequence, since $2\spars > 0$ and $\minsol > 0$,
\begin{equation}\label{ervniuon}
    0 < \frac{\minsol}{2\spars} - \alphaRIP\normd{\sol}.
\end{equation}

%Using $|S^*|\leq k$, we obtain $$\minsol~ \left(\sumHrwgs \right) \leq \spars,$$
%and deduce from \eqref{ervniuon}
%\begin{equation*}
%    \gammaRIP\normd{\noise} = 0 < \frac{1}{2\sumHrwgs} - \alphaRIP\normd{\sol}.
%\end{equation*}
We conclude that $\sol$ satisfies the (\ref{eq:HRIP}) for $\mat$.

Applying \cref{thm:recovery_RIP} and since $\normd{e} = 0$ and $\explo^0 = \zerosig$, we know that there exists $t\leq \itermaxRIP =\frac{ (k+1) }{ 1 - 2k \alphaRIP\frac{\normd{\sol}}{\min_{i\in S^*} |x_i^*|}}$ such that $S^* \subseteq S^t$. It is not difficult to check that the function $f:\RR\longrightarrow \RR$ defined for all $u\in\RR$ by $f(u) = \frac{\spars+1}{1 - u}$ is increasing on $[0, 1)$. By applying $f$ to
\begin{equation*}
    0 \leq  \RIPthreshf \leq \RIPthresh < 1,
\end{equation*}
 we obtain 
 \begin{equation*}
    \itermaxRIP = f\left(\RIPthreshf\right) \leq f(\RIPthresh) =  \itermaxSRIP,
 \end{equation*}
where $\itermaxSRIP$ is defined in \cref{cor:recovery_SRIP}.
 
Therefore $t\leq \itermaxSRIP$ and we conclude that there exists $t\leq \itermaxSRIP$ such that $S^* \subseteq S^t$.
 
 The last statement of \cref{cor:recovery_SRIP} is a direct consequence of \cref{thm:recovery_RIP} and $x^*$ satisfies \eqref{min_hyp_RIP} with $\|e\|=0$.
 
 This concludes the proof of \cref{cor:recovery_SRIP}.

 \subsection{Comments on \texorpdfstring{\eqref{eq:HRIP} and \eqref{eq:HSRIP}}{recovery conditions}}\label{comments-RIP-sec}

The hypotheses of \cref{thm:recovery_RIP} are on the RIP of $A$ and there are two hypotheses on $x^*$: (\ref{eq:HRIP}) and \eqref{min_hyp_RIP}. The condition (\ref{eq:HRIP}) is described below \cref{cor:recovery_SRIP}. In \cref{cor:recovery_SRIP}, the condition becomes \eqref{eq:HSRIP}. We adapt in the section the analysis of the (\ref{eq:HRIP}) to condition \eqref{eq:HSRIP} and show that it is not vacuous under a similar constraint on $\alphaRIP$. 

If $\alphaRIP$ is too large, there does not exist any $\sol$ satisfying (\ref{eq:HSRIP}). It is for instance the case if $\alphaRIP\geq 0.5$. On the contrary, a sufficient condition of existence of vectors $\sol$ satisfying (\ref{eq:HSRIP}) is that the constant $\alphaRIP$ satisfies $2\spars^{\frac{3}{2}}\alphaRIP \leq \RIPthresh < 1$. In this case, when all non-zero entries of $\sol$ are equal, we have $\|\sol\|_2 = \sqrt{|S^*|} \minsol $ and 
$\RIPthreshf =  2k\alphaRIP \sqrt{|S^*|}  \leq  2\spars^{\frac{3}{2}}\alphaRIP 
 \leq \RIPthresh < 1$.
Summarizing, when $2\spars^{\frac{3}{2}}\alphaRIP  < 1$, there exist vectors $x^*$ satisfying \eqref{eq:HSRIP} and the condition of \cref{cor:recovery_SRIP} is not vacuous.

As a conclusion, when $\alphaRIP < \frac{1}{2} k^{-\frac{3}{2}}$, both \cref{thm:recovery_RIP} and \cref{cor:recovery_SRIP} can be applied for a non-empty set of vectors $x^*$. Moreover, we can prove that the interior of the sets of $\sol$ satisfying respectively \eqref{eq:HRIP} and \eqref{eq:HSRIP} are not empty. The two sets grow as $\alphaRIP$ decreases. When $\|e\|_2=0$, the two sets are conical.

%% file: oracle.tex
\subsection{Support Exploration Algorithm Using the Oracle Update Rule}\label{oracle-sec}

The theoretical analysis of SEA and the understanding of the underlying behavior of the algorithm rely on the introduction of an oracle case where the true solution $\sol$% to problem~\eqref{eq:sparse_problem} 
and its support $\suppsol$ are known by the algorithm.
In that case, at iteration $t$, we can use the oracle update rule $\explo^\iterp \leftarrow \explo^\iter-\usol^\iter$, using the direction $\usol^\iter$ defined for any index $\ii\in\intintn$ by

\begin{equation}
    \label{eq:usol}
    \usol^\iter_\ii =
    \begin{cases}
      -\gradstep \soli & \ii \in \suppx \\
      0 & \ii \in \suppxb,
    \end{cases}
\end{equation}

where $\supp^\iter=\largest_\spars\left(\explo^\iter\right)$ contains the indices of the $\spars$ largest absolute entries in $\explo^\iter$ and $\gradstep>0$ is an arbitrary step size.
The resulting pseudo-code is given by Algorithm~\ref{alg:oracle} and we show the important supports in \cref{fig:sea-supports}.

\begin{algorithm}[tbh]
   \caption{Support Exploration Algorithm using the Oracle Update Rule}
   \label{alg:oracle}
\begin{algorithmic}[1] % Noise of the gradient
   \STATE {\bfseries Input:} true solution $x^*$, true support $\suppsol$, sparsity $\spars$, step size $\gradstep$, noisy observation $\obs$, sampling matrix $\mat$
   \STATE {\bfseries Output: }sparse vector $\sigs$
   \STATE Initialize $\explo^0$ 
   \STATE $\iter \leftarrow 0$
   \REPEAT
   \STATE $\supp^\iter \leftarrow \largest_\spars\left(\explo^\iter\right)$ 
   \alglinelabel{line:oracle:update_S}
   %\STATE $\sigs^\iter = \underset{\substack{\sig \in \sigspace \\ \SUPP{\sig} \subseteq \supp^\iter }}{\argmin} \normdd{\mat\sig - \obs}$ \alglinelabel{line:oracle:update_x}
   % \STATE $\sigs^\iterp = \underset{\sig \in \sigspace, \, \SUPP{\sig} \subseteq \supp^\iterp }{\argmin} \normdd{\mat\sig - \obs}$ \label{line:SEA:update_x}
  \STATE $\explo^\iterp \leftarrow \explo^\iter - u^t$ \alglinelabel{line:oracle:update_explo}
   \STATE $\iter \leftarrow \iterp$
   \UNTIL{$u^{t-1} = 0$}
   % \STATE $\sigs \leftarrow \underset{\substack{\sig \in \sigspace \\ \SUPP{\sig} \subseteq \supp^{t-1}}}{\argmin} \normdd{\mat\sigs - \obs}$
  \STATE $\begin{cases}
       \sigs_i & \leftarrow 0 \text{ for } i \in \overline{\supp^{\iter-1}}\\
       \sigs_{\supp^{\iter-1}} & \leftarrow \matdag_{\supp^{\iter-1}} \obs\\
   \end{cases}$ 
   \STATE {\bfseries return} $\sigs$
\end{algorithmic}
\end{algorithm}

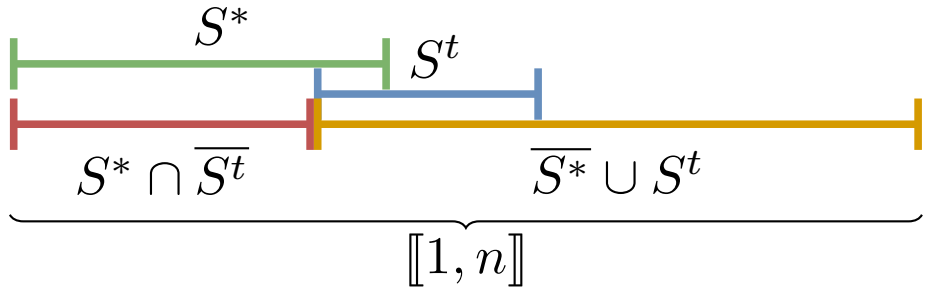
\begin{figure}[htb]
    \centering
    \scalebox{2}{
    \begin{tikzpicture}
    \node[] at (3.4,0.65) {$\suppsol$};
    \draw [|-|,line width=0.5mm,drawGreen] (2,0.4) -- (4.5,0.4);
    \node[] at (4.8,0.45) {$\supp^\iter$};
    \draw [|-|,line width=0.5mm,drawBlue] (4,0.2) -- (5.5,0.2);
    \draw [|-|,line width=0.5mm,drawRed] (2,0) -- (4,0);
    \node[] at (3,-0.3) {$\suppx$};
    \draw [|-|,line width=0.5mm,drawOrange] (4,0) -- (8,0);
    \node[] at (6,-0.3) {$\suppxb$};
    \path[draw,decorate,decoration={brace,mirror}] (2,-0.6) -- (8,-0.6)
    node[midway,below]{$\intintn$};
    \end{tikzpicture}}
    \caption{Visual representation of the main sets of indices encountered in the article.}
    \label{fig:sea-supports}
\end{figure}

Notice $\usol^\iter_\ii$ is non-zero for indices $\ii$ from the true support $\suppsol$ but for which $|\explo^\iter_\ii|$ is too small to be selected in $\supp^\iter$ at line~\ref{line:oracle:update_S}.
%that is missing in current support $\supp^\iter$.
Whatever the initial content of $\explo^0$, the oracle update rule always adds the same increment to $\explo^\iter_\ii$, for $\ii \in \suppx $. This guarantees that, at some subsequent iteration $\iter'\geq \iter$,  the true support $\suppsol$ is recovered among the $\spars$ largest absolute entries in $\explo^{\iter'}$, i.e., $\suppsol \subseteq \supp^{\iter'}$, the intersection is empty, $u^{t'}=0$ and \cref{alg:oracle} stops.

In the following theorem, we formalize this statement and give an upper bound on the number of iterations required by the support exploration algorithm using the oracle update rule.

\begin{theorem}[Recovery - Oracle Update Rule]
\label{thm:recovery:oracle}
For all matrices $A$, error vectors $e$, initializations $\explo^0$ and for all $\eta>0$, there exists 
\[t_s\leq \itermax^{{\scriptscriptstyle oracle}} = k \left( 1 + \frac{2  \|\explo^0\|_\infty}{\eta \min_{i\in\suppsol}|\soli|} \right)
\]
such that $\suppsol \subseteq S^{t_s}$, where $S^t$ is defined in \cref{alg:oracle} line \ref{line:oracle:update_S}. 

Moreover, $u^{t_s} = 0$ and \cref{alg:oracle} returns $\underset{\substack{\sig \in \sigspace \\ \SUPP{\sig} \subseteq \supp^{t_s}}}{\argmin} \normdd{\mat\sigs - \obs}$.

%\Mim{On peut également prouver que $\suppsol \subseteq \supp^\iterb$ car une fois que $\suppsol \subseteq \suppt$, $\usolt = 0$ et $\suppt$ reste constant, pour tout $\iter \geq \iters$} \Fr{Oui, tu as raison. Encore plus, avec le critère d'arrêt que tu as proposé, l'algo s'rrête dés que $S^* \subseteq S$. Je te laisse le rédiger.}
\end{theorem}
\subsubsection{Proof of \texorpdfstring{\cref{thm:recovery:oracle}}{Theorem~\ref{thm:recovery:oracle}}}
\label{app:proof_oracle}

We denote, for all $t\in\NN^*$ and all $\ii \in \intintn$, and $S^t$ defined in \cref{alg:oracle}, line \ref{line:oracle:update_S}
\begin{equation}\label{eqrutnbt}
     \countingit = \abs{\{ \itero \in \intintt: \ii \in \suppx[\itero] \}}.
\end{equation}   
We extend the definition to $t=0$ and set, for all $\ii \in \intintn$, $c_i^0=0$.

We can prove by induction on $t$ that, given the definition of $\explo^t$ in \cref{alg:oracle} and $u^t$ in \eqref{eq:usol}, for all $t\in\NN$,
\begin{equation}\label{explot_oracle}
    \explo^t =  \explo^0 + \eta ~c^t \odot x^*,
\end{equation}
where $\odot$ denotes the Hadamard product.

The following lemma states that if $c_i^t$ is large then $i$ is always selected by \cref{alg:oracle}.
\begin{lemma}\label{tmp1}
For all $i \in \suppsol$ and all $t\in \NN^*$
    \begin{equation*} 
    \mbox{ if }\quad c_i^t > \frac{2\|\explo^0\|_\infty}{\eta |\soli|}\quad\mbox{ then }\quad\forall t'\geq t, \quad i\in\supp^{t'}.
\end{equation*}
\end{lemma}
\begin{proof}
    Let $i \in \suppsol$ and $t\in \NN^*$ be such that $c_i^t > \frac{2\|\explo^0\|_\infty}{\eta |\soli|}$. Consider $t'\geq t$.

Since $t \mapsto c_i^t$ is non-decreasing, we have
\[c_i^{t'} \geq c_i^{t} > \frac{2\|\explo^0\|_\infty}{\eta |\soli|}.
\]
Therefore, for all $j\in \overline{\suppsol}$,
\[ |\explo_i^{t'}| = |\explo_i^0 + \eta c_i^{t'} x_i^*| \geq |\eta c_i^{t'} x^*_i| - |\explo^0_i| > 2\|\explo^0\|_\infty - |\explo^0_i| \geq \|\explo^0\|_\infty \geq |\explo^0_j| = |\explo^{t'}_j|.
\]
Therefore $|\explo_i^{t'}|$ is larger than at least $n-k$ elements of $\{|\explo_j^{t'}| : j\in\intintn\}$. Said differently, $i\in \largest_\spars(\explo^{\iter'}) = S^{t'}$. 

This concludes the proof of \cref{tmp1}.
\end{proof}

This leads to the following upper bound.
\begin{lemma}\label{tmp2}
For all $i \in \suppsol$ and all $t\in \NN^*$
    \begin{equation*} 
     c_i^t \leq \frac{2\|\explo^0\|_\infty}{\eta |\soli|} + 1.
\end{equation*}
\end{lemma}
\begin{proof}
If \cref{tmp2} is false, there exists $i\in\suppsol$ and $t\in\NN^*$ such that $c_i^{t} > \frac{2\|\explo^0\|_\infty}{\eta |\soli|} + 1$.

We denote 
\[t' = \min\{t\in\NN^* : c_i^t > \frac{2\|\explo^0\|_\infty}{\eta |\soli|} + 1\}.
\]
We have $c_i^{t'} >\frac{2\|\explo^0\|_\infty}{\eta |\soli|} + 1 \geq 1$ and therefore $t'>1$. As a consequence, $t'-1\in\NN^*$ and  $c_i^{t'-1}$ is defined by \eqref{eqrutnbt}. Because of the definitions of $t'$ and $c_i^{t'}$, we must have $c_i^{t'} = c_i^{t'-1}+1$. Therefore, $c_i^{t'-1} > \frac{2\|\explo^0\|_\infty}{\eta |\soli|}$. Since $i\in\suppsol$ and $t'-1\in\NN^*$, using \cref{tmp1}, we conclude that $i\in\supp^{t'-1}$ and, using the definition of $c_i^{t'}$ in \eqref{eqrutnbt}, that $c_i^{t'} = c_i^{t'-1}$. 

This is impossible and we conclude that \cref{tmp2} holds.
\end{proof}

{\bf Proof of \cref{thm:recovery:oracle}:}
We denote
\begin{equation}\label{def_ts}
    t_s = \min\{t\in\NN : \suppsol \subseteq \suppt\}.
\end{equation}
By convention, if for all $t\in\NN$, $\suppsol \not\subseteq \suppt$, we set $t_s = +\infty$. The first statement of \cref{thm:recovery:oracle} is obvious if $t_s=0$. We assume below that $t_s\geq 1$.

Consider $t\in \intint{0, t_s-1}$, using of the definition of $t_s$ in \eqref{def_ts}, there exists $i\in\suppsol \cap \overline{\suppt}$. Using the definition of $c_i^{t+1}$ in \eqref{eqrutnbt}, we obtain $c_ i^{t+1}=c_ i^{t}+1$. Since for all $j\in\intintn$, $t \mapsto c_j^t$ is non-decreasing, we conclude that
\[\mbox{for all } t\in \intint{0, t_s-1},\qquad \sum_{i\in\suppsol} c_i^{t+1} \geq \sum_{i\in\suppsol} c_i^{t} +1.
\]
We therefore obtain
\begin{eqnarray*}
\sum_{i\in\suppsol} c_i^{t_s} & = &  \sum_{i\in\suppsol} \left(\sum_{i=0}^{t_s-1} (c_i^{t+1} - c_i^t) \right) \\
& = & \sum_{i=0}^{t_s-1} \left((\sum_{i\in\suppsol}c_i^{t+1}) -(\sum_{i\in\suppsol}c_i^{t}) \right) \\
& \geq &\sum_{i=0}^{t_s-1} 1 = t_s.
\end{eqnarray*}
Using \cref{tmp2}, we obtain
\[t_s \leq \sum_{i\in\suppsol} c_i^{t_s} \leq \sum_{i\in\suppsol} \left(\frac{2\|\explo^0\|_\infty}{\eta |\soli|} + 1\right) \leq k \left( 1 + \frac{2  \|\explo^0\|_\infty}{\eta \min_{i\in\suppsol}|\soli|} \right).
\]
To conclude the proof, we simply remark that, since $S^*\subseteq S^{t_s}$, by definition of $u^{t_s}$ in \eqref{eq:usol}, $$u^{t_s} = 0.$$
\hfill{$\Box$}

Since $\sol$ and $\suppsol$ are not available in practice, we replace in \cref{alg:oracle} the oracle update $\usol^\iter$ by the surrogate $\grad$ (line~\ref{line:SEA:update_explo}).
The choice of this surrogate is an application of STE and is natural. For instance, one can show that $\usol^\iter=\grad$ in the simple case where $\mat$ has orthonormal columns and the observation is noiseless (see Corollary~\ref{cor:recovery_ortho} and \cref{lem:ortho} in Appendix \ref{app:proof_ortho}). In the general setting, SEA can be interpreted as a noisy version of the support exploration algorithm using the oracle update. Theorem \ref{thm:recovery_RIP} and its proof in Appendix \ref{app:proof_RIP} are based on the fact that $\usol^\iter - \grad$ is small, under suitable hypotheses on $x^*$ and the RIP constants of $A$.

%% file: General_Proof.tex
\subsection{If the STE-Update is Sufficiently Close to the Oracle Update, SEA Visits \texorpdfstring{$S^*$}{the True Support}}
\label{subsec:genrecovery}

To prove \cref{thm:recovery_RIP}, we first provide a general recovery theorem here. %in \cref{subsec:genrecovery}. 
The theorem states that if the discrepancy between the Oracle update and the STE update is sufficiently small, SEA visits $S^*$.

To do so, we define, for all $\iter \in \NN$,  the gradient noise: $\gradnoise^\iter \in \sigspace$ as
\begin{equation}
\label{eq:gradnoise}
    \gradnoise^\iter = \usolt - \grad.
\end{equation}
We define the maximal gradient noise norm
\begin{equation}
\label{eq:gradnoisemax}
    \gradnoisemax = \gradnoisemaxf \in \field.
\end{equation}

We define the Recovery Condition (\Condr) as 
\begin{equation}\tag{\Condr}
\label{eq:Hr}
\gradnoisemax < \frac{\eta}{2k} \min_{i\in S^*} |x_i^*|.
\end{equation}

\begin{theorem}[Recovery - General case]
\label{thm:recovery}
If (\ref{eq:Hr}) holds, then for all initializations $\explo^0$ and all $\eta>0$, there exists $t_s\leq T_{max}$ such that $\suppsol \subseteq S^{t_s}$, where $S^t$ is defined in \cref{alg:SEA} line \ref{line:SEA:update_S} and
\begin{equation}
\label{eq:Tmax}
    \itermax = \frac{2 k \|\explo^0\|_{\infty} + (k+1)\eta \min_{i\in S^*} |x_i^*| }{\eta \min_{i\in S^*} |x_i^*| - 2 k \gradnoisemax }.
\end{equation}

\end{theorem}

The proof is in \cref{app:proof_general}, right after the comments below.

The main interest of \cref{thm:recovery} is to formalize quantitatively that, when $u^t - \grad$ is sufficiently small, SEA visits the correct support.
However, the condition (\ref{eq:Hr}) is difficult to use and interpret since it involves both $A$, $x^*$, and all the sparse iterates $x^t$. This is why we provide in \cref{cor:recovery_ortho}, \cref{thm:recovery_RIP} and \cref{cor:recovery_SRIP} sufficient conditions on $A$, $e$ and $\sol$ guaranteeing that (\ref{eq:Hr}) holds.

The conclusion of \cref{thm:recovery} is that the iterative process of SEA visits the correct support at some iteration $t$. We have in general no guarantee that this time-step $t$ is equal to $\iterb$. We are however guaranteed that SEA returns a sparse solution such that $\normd{Ax^{\iterb} - y} \leq \normd{Ax^{t_s} - y} \leq \normd{Ax^* - y}$. This does not give a guarantee on the support recovery but on the reconstruction error. In machine learning, this upper bound can be used to derive an upper bound of the risk. %We will see in \cref{cor:recovery_ortho}, \cref{thm:recovery_RIP} and \cref{cor:recovery_SRIP} that, when $A$ is sufficiently incoherent and $\|e\|_2$ is small enough, we actually have $S^* \subseteq \SUPP{x^\iterb}$.

Concerning the value of $\itermax$, a quick analysis shows that $\itermax$ increases with  $\gradnoisemax$, when (\ref{eq:Hr}) holds. In other words, the number of iterations required by the algorithm to provide the correct solution increases with the discrepancy between $u^t$ and $\grad$. %This confirms the intuition behind the construction of SEA. 

The initializations $\explo^0\neq 0$ have an apparent negative impact on the number of iterations required in the worst case. This is because in the worst-case $\explo^0$ would be poorly chosen and SEA needs iterations to correct this poor choice. 

Concerning $\gradstep$, notice that, since $u^t$ is proportional to $\gradstep>0$, $\gradnoisemax$ is proportional to $\gradstep>0$ and therefore (\ref{eq:Hr}) is independent of $\gradstep$. When possible, any $\gradstep$ permits the recovery of $S^*$. The only influence of $\gradstep$ is on $\itermax$. In this regard, since $\gradnoisemax$ is proportional to $\gradstep>0$, the denominator of \eqref{eq:Tmax} is proportional to  $\gradstep$. % It only influences the numerator of \eqref{eq:Tmax}. 
In the numerator, we see that the larger $\gradstep$ is, the faster SEA will override the initialization $\explo^0$. The choice of $\eta$ is very much related to the question of the quality of the initialization discussed above.

\subsubsection{Proof of \texorpdfstring{\cref{thm:recovery}}{Theorem~\ref{thm:recovery}}}
\label{app:proof_general}

To prove \cref{thm:recovery}, we need to adapt a closed formula for the exploratory variable $\explo^\iter$ already encountered in the proof of \cref{thm:recovery:oracle}. Then, we will study the properties of this closed formula through the counting vector $\countingt$ in \cref{sec:genP1}. to find a sufficient condition of support recovery. Then we prove \cref{thm:recovery} in \cref{thm_recovery_proof_subsec}.

\subsubsection{Preliminaries}  \label{sec:genP1}

%Before the proof, we remind Figure~\ref{fig:sea-supports} in Figure \ref{fig:sea-supports_annexe}.

%\begin{figure}[tbh]
%\begin{center}
%\includegraphics[width=0.4\linewidth]{images/SEA_supp_3.drawio.png}
%\caption{Visual representation of the main sets of indices encountered in the article.}
%\label{fig:sea-supports_annexe}
%\end{center}
%\vskip -0.2in
%\end{figure}

We remind Figure \ref{fig:sea-supports} on which the mains supports are represented and we remind, for each iteration $\iter \in \NN$ and $\ii \in \intintn$,  the Oracle Update already defined in \eqref{eq:usol}
\begin{equation*}
    \usolt_\ii =
    \begin{cases}
      -\gradstep \soli & \ii \in \suppx \\
      0 & \ii \in \suppxb .
    \end{cases}       
\end{equation*}
We also remind the gradient noise, already defined in \eqref{eq:gradnoise}, $\gradnoise^\iter = \usolt - \grad$. 

We first remark that, for any $\ii \in \suppt$,  
\begin{equation}\label{romnrevjnj}
    \gradnoise^\iter_\ii = 0.
\end{equation} 
% To prove this equality, we remark that, for all $\ii \in \suppt$, $\usol^\iter_\ii = 0$ and prove that $(\grad[])_\ii = 0$. Indeed, the latter  holds because $i\in S^t$ and $x^t = \argmin_{\SUPP{\sig} \subseteq S^t} \|Ax-y\|_2^2$ (see \cref{alg:SEA}, line \ref{line:SEA:update_x}). If we denote $F(x) =  \|Ax-y\|_2^2$, the Karush-Kuhn-Tucker equations of the latter constrained optimization problem require $\nabla F(x^t)$ to be in the orthogonal complement of the vector space $\{x\in\mathbb{R}^n : \SUPP{\sig} \subset S^t\}$. That is $(\nabla F(x^t))_i$ = $ 2 (A^T(A x^t - y ) )_i = 0$, for all $i\in S^t$. This concludes the proof of \eqref{romnrevjnj}.
%To prove this equality, we remark that, for all $\ii \in \suppt$, $\usol^\iter_\ii = 0$ and prove that $(\grad[])_\ii = 0$. Indeed, the latter  holds because $i\in S^t$ and $x^t = H(\explo^t) = \argmin_{\SUPP{\sig} \subseteq S^t} \frac{1}{2}\|Ax-y\|_2^2$ (see \eqref{eq:sparsification_operator} and \cref{alg:SEA}, line \ref{line:SEA:update_x}). If we denote $F(x) =  \frac{1}{2}\|Ax-y\|_2^2$, the Karush-Kuhn-Tucker equations of the latter constrained optimization problem require $\nabla F(x^t)$ to be in the orthogonal complement of the vector space $\{x\in\mathbb{R}^n : \SUPP{\sig} \subset S^t\}$. That is $(\nabla F(x^t))_i$ = $(A^T(A x^t - y ) )_i = 0$, for all $i\in S^t$. This concludes the proof of \eqref{romnrevjnj}.
To prove this equality, we remark that, for all $\ii \in \suppt$, $u_i^t = 0$ and prove that $(\mata(\mat x^t-\obs))_i = 0$. Indeed, the latter holds because $\ii \in \suppt$ and $x_{\suppt}^t = \matdagsuppt y$ (see \cref{alg:SEA}, line \ref{line:SEA:update_x}). As is well-known for the Moore-Penrose inverse, $\matsuppt\matdagsuppt$ is the orthogonal projector onto $\text{colspan}(\matsuppt)$. Therefore, $\matsuppt\matdagsuppt y-y$ is orthogonal to $\text{colspan}(\matsuppt)$ and for all $x'\in{\mathbb R}^k$, $0= \langle \matsuppt \matdagsuppt y-y, \matsuppt x' \rangle = \langle \matsuppt^T (\matsuppt \matdagsuppt y-y),  x' \rangle$. Therefore, $\matsuppt^T (\matsuppt \matdagsuppt y-y)=0$. Since, $x_i^t=0$ for all $i\in \overline{S^t}$, we also have $\matsuppt \matdagsuppt y=\matsuppt x_{S^t}^t=\mat x^t $ and we deduce that for all $\ii\in \suppt$, $(\mat^T(\mat x^t - y ) )_i =(\matsuppt^T (\matsuppt \matdagsuppt y-y) )_j = 0$, where the line $j\in\llbracket 1,k\rrbracket$ of $\matsuppt^T$ corresponds to the line $\ii\in \suppt$ of $\mat^T$. This concludes the proof of \eqref{romnrevjnj}.

As a consequence of the definition of $\gradnoise^\iter$ and SEA, line \ref{line:SEA:update_explo}, for any $\iter \in \NN$, 
\begin{equation}
\label{eq:iterations}
    \explo^\iterp
    =
    \explo^\iter + \gradnoise^\iter - \usolt. 
\end{equation}

The gradient noise $\gradnoise^\iter$ is the error preventing the gradient from being in the direction of the oracle update $\usol^\iter$. 
At each iteration, this error is accumulating in $\explo^\iter$. 
With $\gradnoiseacc^0 = 0$, %_{\sigspace}$, 
for any $\iter \in \NN^*$, we define this accumulated error by
\begin{equation}
\label{eq:accumulation}
    \gradnoiseacc^\iter = \sum_{\itero = 0}^{\iter - 1} \gradnoise^\itero \in \sigspace.
\end{equation}
As already done in the proof of \cref{thm:recovery:oracle} for the support sequence defined in \cref{alg:oracle}, we define counting vectors. However, this time they are defined for the sequence defined in \cref{alg:SEA}. We keep the same notations for simplicity. We set $\counting^0 = 0$, %$\counting^0 = 0_{\sigspace}$, 
for any $\iter \in \NN^*$ and $\ii \in \intintn$, we also define the counting vector by
\begin{equation}
\label{eq:counting_def}
    \countingit = \abs{\{ \itero \in \intintt: \ii \in \suppx[\itero] \}}.
\end{equation}

We will use the recursive formula for $\counting^\iter$: For any $\iter \in \NN$, $\ii \in \intintn$
\begin{equation}
\label{eq:counting}
    \countingitp =
    \begin{cases}
      \countingit + 1 & \mbox{if }\ii \in \suppx\\
      \countingit & \mbox{if }\ii \in \suppxb.
    \end{cases}       
\end{equation}

For any $\ii \in \intintn$, the sequence $(\countingit)_{\iter \in \NN}$ is non-decreasing.

Using (\ref{eq:iterations}), (\ref{eq:accumulation}) and (\ref{eq:counting_def}), we can establish by induction on $t$ that for any $\iter \in \NN$, 
\begin{equation}\label{prop:counting}
\explo^\iter = \explo^0 + \gradstep \countingt \odot \sol + \gradnoiseacc^\iter,
\end{equation}
where $\odot$ denotes the Hadamard product. This generalizes \eqref{explot_oracle} to the noisy setting.

%\subsubsection{Preliminary 2: Counting vector behavior}
%\label{sec:genP2}

As can be seen from \eqref{prop:counting}, the error accumulation $\gradnoiseacc^\iter$ is responsible for the exploration in the wrong directions. 
While $\countingt \odot \sol$ encourages exploration in the direction of the missed components of $\sol$. 
Below, we provide important properties of $(\countingt)_{t\in\NN}$.

 At each iteration of SEA, using (\ref{eq:counting}) when $ \suppsol \nsubseteq \suppt$, there exists at least one $i\in\suppsol$ such that $c_i^{t+1} =c_i^{t}+1 $. Using also that, for all $\ii \in \suppsol$, $(\countingit)_{\iter \in \NN}$ is non-decreasing we obtain 
 \begin{equation}\label{prop:increase}
 \mbox{for all }\iter \in \NN\mbox{ such that }\suppsol \nsubseteq \suppt, \qquad \sumi \countingitp \geq \bigl(\sumi \countingit \bigr) + 1
\end{equation}

%Before going further, we need to restrain our analysis to iterations before the first support recovery.

We define the first recovery iterate\footnote{Again, a similar quantity is defined in the proof of \cref{thm:recovery:oracle} for the supports $S^t$ defined in \cref{alg:oracle}. We use the same notation although this time the quantity is defined for the sets $S^t$ defined in \cref{alg:SEA}. It should not be ambiguous since the notations are used in different proofs and sections. } $t_s$ as the smallest iteration $t$ such that $S^* \subseteq S^t$. More precisely, 
\begin{equation}
\label{eq:iters}
    \iters = \min{\{\iter, \, \suppsol \subseteq \suppt\} \in \NN}. 
\end{equation}
By convention, if $\suppsol$ is never recovered, $\iters = +\infty$.
By induction on $\iter$, using \eqref{prop:increase}, we obtain a lower bound on $\sumi \countingit$:
\begin{equation}\label{cor:lower_bound}
     \mbox{For all }\iter \leq \iters, \qquad \sumi \countingit \geq \iter.
\end{equation}

Let us now upper bound $\sumi \countingit$. 
We first remind the definition of $\gradnoisemax$ in (\ref{eq:gradnoisemax}). We define the sharp Recovery Condition
\begin{equation}\tag{\Condr'}
\label{eq:Hr'}
\gradnoisemax < \frac{1}{2 \sumHr}
\end{equation}
 and
 \begin{equation}
\label{eq:Tmax'}
    \itermax' = \frac{\sumi \deltafrac + \spars + 1}{1 - 2 \gradnoisemax \sumHr}.
\end{equation}
 
If (\ref{eq:Hr'}) holds, we define for any $\ii \in \suppsol$, the $\ii^{th}$ counting threshold by
\begin{equation}
\label{eq:countingthreshi}
\countingthreshi = \deltafracall.
\end{equation}

\begin{proposition}[Upper bound]
\label{prop:upper_bound}
If (\ref{eq:Hr'}) holds, for any $\ii \in \suppsol$ and any $\iter \leq \itermax' $, we have $\countingit \leq \countingthreshi + 1$.
\end{proposition}

\begin{proof}
Assume (\ref{eq:Hr'}) holds. We have $\itermax' > 0$. Let $\ii \in \suppsol$, we distinguish two cases:

\underline{$1^{st} \, case$}: If for all $\iter \leq \itermax'$, $\countingit \leq \countingthreshi$: Then, obviously, for any $\iter \leq \itermax'$, $\countingit \leq \countingthreshi + 1$.

\underline{$2^{nd} \, case$}: If there exists $\iter \leq \itermax'$, such that $\countingit > \countingthreshi$: 

We define $\iteri = \min{\{ \iter \in \NN: \, \countingit > \countingthreshi \}}$. We have $\iteri \leq \itermax'$. The proof follows two steps:
\vspace{-0.1cm}
\begin{flalign}
\label{eq:UpperBoundp1}
\text{\hspace{0.17cm} 1. We will prove that for all } \iter \in \intint{\iteri, \itermax' }, \, \countingit = \countingi^\iteri. && \\
\text{\hspace{0.17cm} 2. We will prove that for all } \iter \leq \itermax' , \, \countingit \leq \countingthreshi + 1. \label{eq:UpperBoundp2}
\end{flalign}

\begin{enumerate}
    \item Let $\iter \in \intint{\iteri, \itermax'}$, we have, using \eqref{prop:counting}, the triangle inequality and the fact that $\countingit \geq \countingi^\iteri > \countingthreshi$
    \begin{align*} 
        \abs{\explo_\ii^\iter} & = \abs{\explo_\ii^0 + \gradstep \countingit \soli + \gradnoiseacc_\ii^\iter}\\
        & \geq \gradstep \countingit \abs{\soli} - \abs{\explo_\ii^0} - \abs{\gradnoiseacc_\ii^\iter}\\
        & > \gradstep \countingthreshi \abs{\soli} - \abs{\explo_\ii^0} - \abs{\gradnoiseacc_\ii^\iter}.
    \end{align*}
    Using the definition of $\countingthreshi$, in (\ref{eq:countingthreshi}), we obtain
    \begin{align*}
        \abs{\explo_\ii^\iter} & > \gradstep \deltafracall \abs{\soli} - \abs{\explo_\ii^0} - \abs{\gradnoiseacc_\ii^\iter} \\
        & = \max_{\jj \notin \suppsol}\abs{\explo_\jj^0} + 2\itermax'\gradnoisemax - \abs{\gradnoiseacc_\ii^\iter}.
    \end{align*}
    Since for any $\jj \in \intintn$, $\abs{\gradnoiseacc_\jj^\iter} \leq \sum_{\itero = 0}^{\iter - 1} \abs{\gradnoise_\jj^\itero} \leq \iter\gradnoisemax \leq \itermax'\gradnoisemax$, we have
    \begin{align}
        \abs{\explo_\ii^\iter} & > \max_{\jj \notin \suppsol}\abs{\explo_\jj^0} +\max_{\jj \notin \suppsol} \abs{\gradnoiseacc_\jj^\iter} + \abs{\gradnoiseacc_\ii^\iter} - \abs{\gradnoiseacc_\ii^\iter}\nonumber\\
        & \geq \max_{\jj \notin \suppsol}\abs{\explo_\jj^0 + \gradnoiseacc_\jj^\iter} \nonumber\\
        & = \max_{\jj \notin \suppsol}\abs{\explo_\jj^\iter},\label{eq:explo_maj}
    \end{align}
    where the last equality holds because of \eqref{prop:counting} and for all $\jj \notin \suppsol$, all $\iter \in \NN$, $\countingt_\jj = 0$.

    Equation \eqref{eq:explo_maj} implies 
    that   $|\mathcal{X}_ i^t|$ is larger than $|\{j\not\in S^* \}|$ elements of $\{|\mathcal{X}_ j^t|\ |\ j\in\intint{1, n}\}$ and, since  $\abs{\suppsol} \leq \spars$, we have 
    $|\{j\not\in S^* \}|= n -\abs{S^*} \geq n-k$. Finally, $|\mathcal{X}_ i^t|$ is larger than $n-k$ elements of $\{|\mathcal{X}_ j^t|\ |\ j\in\intint{1, n}\}$ and $i\in\text{largest}_ k(\mathcal{X}^t)=S^t$. 
    
    As a conclusion, for all $\iter \in \intint{\iteri, \itermax'}$, $\ii \in \suppt$. Using (\ref{eq:counting})
    , this leads to $\countingitp = \countingit$. Therefore, for all $\iter \in \intint{\iteri, \itermax' + 1}$, $\countingit = \countingi^\iteri$. This concludes the proof of the first step.
    
    \item Since $\iteri = \min{\{ \iter \in \NN: \countingit > \countingthreshi \}}$ and since $\countingi^0 = 0$, $\iteri \geq 1$. 
    Since by definition of $\iteri$, $\countingi^{\iteri - 1} \leq \countingthreshi$ and $\countingi^\iteri \neq \countingi^{\iteri - 1}$; we find that $\countingi^\iteri = \countingi^{\iteri - 1} + 1 \leq \countingthreshi + 1$.
    
    Using (\ref{eq:UpperBoundp1}) , for all $\iter \in \intint{\iteri, \itermax' }, \, \countingit = \countingi^\iteri \leq \countingthreshi + 1$. Finally, since $(\countingit)_{\iter \in \NN^*}$ is non-decreasing, it follows that for any $\iter \leq \iteri - 1, \, \countingit \leq \countingi^{\iteri - 1} \leq \countingthreshi$. This concludes the proof of (\ref{eq:UpperBoundp2}). 
\end{enumerate}
\end{proof}

\subsubsection{Proof of \texorpdfstring{\cref{thm:recovery}}{Theorem~\ref{thm:recovery}}}\label{thm_recovery_proof_subsec}

To prove \cref{thm:recovery}, we first prove a sharper, but difficult-to-interpret theorem.
\begin{theorem}[Recovery - General case]
\label{thm:recovery'}
If (\ref{eq:Hr'}) holds, then for all initializations $\explo^0$ and all $\eta>0$, there exists $t_s\leq T_{max}'$ such that $\suppsol \subseteq S^{t_s}$, where $S^t$ is defined in \cref{alg:SEA} line \ref{line:SEA:update_S} and
\[
    \itermax' = \frac{\sumi \deltafrac + \spars + 1}{1 - 2 \gradnoisemax \sumHr}.
\]
\end{theorem}

\begin{proof}

We assume (\ref{eq:Hr'}) holds and prove \cref{thm:recovery'} using the results of \cref{sec:genP1}.

In order to do this, we first show that $\itermax' = \sumi \countingthreshi + \spars + 1$, then we demonstrate that $\iters \leq \itermax'$.

Since (\ref{eq:Hr'}) holds, using the definition of $\itermax'$, we calculate
\begin{align*}
    \itermax' &= \frac{1}{1 - 2 \gradnoisemax \sumHr} \left(\sumi \deltafrac + \spars + 1 \right)\\
    \left(1 - 2 \gradnoisemax \sumHr \right) \itermax' &= \sumi \deltafrac + \spars + 1 \\
    \itermax' &= \sumi \deltafrac + \spars + 1 + 2\itermax'\gradnoisemax  \sumHr \\
    \itermax' &= \sumi \deltafracall + \spars + 1.\\
\end{align*}
Using (\ref{eq:countingthreshi}), we obtain $\itermax' = \sumi \countingthreshi + \spars + 1$.

We finally prove \cref{thm:recovery} by contradiction. Assume by contradiction that $\iters > \itermax'$, where $\iters$ is defined in (\ref{eq:iters}). Using \eqref{cor:lower_bound} with $\iter = \floor{\itermax'} < \iters$, we have
\begin{equation}
\label{eq:thm_recov_contradiction}
    \sumi \countingi^{\floor{\itermax'}} 
    \geq \floor{\itermax' }
    = \floor{\sumi \countingthreshi + \spars + 1 }
    > \sumi \countingthreshi + \spars.
\end{equation}

However, using $\abs{\suppsol} \leq \spars$ and \cref{prop:upper_bound} for $\iter = \floor{\itermax'}$ , we find 
\[
\sumi \countingthreshi + \spars 
\geq \sumi (\countingthreshi + 1)
\geq \sumi \countingi^{\floor{\itermax'}} 
%> \sumi \countingthreshi + \spars.
\]
This contradicts (\ref{eq:thm_recov_contradiction}) and we can conclude that $\iters \leq \itermax'$. This proves \cref{thm:recovery'}.
\end{proof}

{\bf Proof of \cref{thm:recovery}:}

If (\ref{eq:Hr}) holds, that is $\gradnoisemax < \frac{\eta}{2k} \min_{i\in S^*} |x_i^*|$, since $\sum_{i\in S^*} \frac{1}{|x_i^*|} \leq \frac{k}{\min_{i\in S^*} |x_i^*|}$, we have
\[ \gradnoisemax < \frac{1}{2 \sumHr}.
\]
Therefore, (\ref{eq:Hr'}) holds, and we can apply \cref{thm:recovery'}. It ensures that for all initializations $\explo^0$ and all $\eta>0$, there exists $t_s\leq T_{max}'$ such that $\suppsol \subseteq S^{t_s}$.

To prove \cref{thm:recovery}, it suffices to prove that $\itermax'\leq \itermax$. Using  $\sum_{i\in S^*} \frac{1}{|x_i^*|} \leq \frac{k}{\min_{i\in S^*} |x_i^*|}$, we obtain
\[1-2\gradnoisemax \sum_{i\in S^*} \frac{1}{\eta |x_i^*|} \geq 1-2\gradnoisemax \frac{k}{\eta\min_{i\in S^*} |x_i^*|}
\]
and using $\min_{j\not\in S^*} |\explo^0_j| + |\explo^0_i|\leq 2\|\explo^0\|_{\infty}$ we have
\[\itermax' =\frac{\sumi \deltafrac + \spars + 1}{1 - 2 \gradnoisemax \sumHr}  \leq  \frac{\frac{2 k \|\explo^0\|_{\infty}}{\eta \min_{i\in S^*} |x_i^*|} + k + 1}{1 - 2\gradnoisemax \frac{k}{\eta\min_{i\in S^*} |x_i^*|}} = \itermax.
\]

%% file: Orthogonal_Proof.tex
\subsection{Warm-Up: SEA Recovers the Correct Support when the Columns of \texorpdfstring{$A$}{A} are Orthonormal} \label{annexe-ortho}

The following corollary particularizes \cref{thm:recovery} to the noiseless and orthogonal case. In practice, a complicated algorithm like SEA is of course useless in such a case, and the state-of-the-art algorithms mentioned in the introduction have similar recovery properties. We give this corollary mostly to illustrate the diversity of links between the properties of the triplet $(A,x^*,e)$ and $\gradnoisemax$ and the behavior of SEA, where we remind the definitions of $b^t$ and $\gradnoisemax$ in \eqref{eq:gradnoise} and \eqref{eq:gradnoisemax}. The following \cref{cor:recovery_ortho} is not only a sanity check for the convergence of SEA under simplistic assumptions, but it also provides a helpful case to understand the proof of \cref{thm:recovery_RIP}. Indeed, it gives a case where the oracle updates introduced in the proof coincides the surrogate, STE update, as mentioned in the last paragraph of \cref{oracle-sec}.

\begin{corollary}[Recovery - Orthogonal case]
\label{cor:recovery_ortho}
If $\mat$ is a tall (or square) matrix with orthonormal columns ($\mata A = I_\sigsize$) and $\normd{\noise} = 0$, then $\gradnoisemax = 0$.
As a consequence, for all $\sol$, for initialization $\explo^0 = \zerosig$ and all $\eta>0$, if SEA performs more than $\spars + 1$ iterations, we have 
$S^* \subseteq S^{\iterb}$ and $x^{\iterb} = x^*$.
% \[S^* \subseteq S^{\iterb} \qquad \mbox{and}\qquad x^{\iterb} = x^*.
% \]
\end{corollary}

%The proof is in \cref{app:proof_ortho}.

\label{app:proof_ortho}
To prove \cref{cor:recovery_ortho}, we first show in \cref{lem:ortho} that the gradient noise $b^t$ is null for all $t\in\NN$. Then, we apply \cref{thm:recovery} and prove that $S^* \subseteq S^{\iterb}$ and $\sigs^\iterb = \sol$.

\begin{lemma}
\label{lem:ortho}
    If $\mat$ is a tall (or square) matrix with orthonormal columns ($\mata A = I_\sigsize$) and $\|e\|_2=0$, then for any $\iter \in \NN$ and any $\gradstep > 0$,
    \begin{equation*}
        \eta \mata\left(\mat\sigst - \obs \right) = \usolt,
    \end{equation*}
    i.e. $b^t = 0$.
\end{lemma}

\begin{proof}
    Let $\iter \in \NN$. Notice first that since $\|e\|_2=0$ and $\mat$ is  a tall (or square) matrix with orthonormal columns
    \begin{equation}\label{eq:grad_ortho}
        \mata\left(\mat\sigst - \obs \right) =  \mata\mat \left( \sigst- \sol \right) = \sigst- \sol.
    \end{equation}
    
%    \begin{align}
%        \left( \mata\left(\mat\sigst - \obs \right) \right)_\ii 
 %       &= \matia (\mat\sigst - \obs) \nonumber\\
%        &= \matia \mat (\sigst - \sol) \nonumber\\
%        &= \sumj \matia \mat_\jj (\sigst - \sol)_\jj \nonumber\\
%        &= (\sigst - \sol)_\ii. \label{eq:grad_ortho}
%    \end{align}
    
    To prove the Lemma, we distinguish three cases: $\ii \in \suppt$, $\ii \in \suppz$ and $\ii \in \suppx$.
    
    \underline{$1^{st} \, case$}: If $\ii \in \suppt$, $\eta\left( \mata\left(\mat\sigst - \obs \right) \right)_\ii = 0 = u^t_i$. The first equality is a consequence of the definition of $\sigs^\iter$ in \cref{alg:SEA}, line \ref{line:SEA:update_x}. The second is due to the definition of $u^t$, in \eqref{eq:usol}.
    
    \underline{$2^{nd} \, case$}: If $\ii \in \suppz$, taking the $i$th entry of (\ref{eq:grad_ortho}) and using the support constraints of $x^t$ and $\sol$, we find
    \[\eta \left( \mata\left(\mat\sigst - \obs \right) \right)_\ii  = 0 = u^t_i,
    \]
    where the second equality is due to the definition of $u^t$, in \eqref{eq:usol}.

    \underline{$3^{rd} \, case$}: If $\ii \in \suppx$, the $i$th entry of (\ref{eq:grad_ortho}) becomes
    \[
        \eta \left( \mata\left(\mat\sigst - \obs \right) \right)_\ii
        = - \eta\soli = u^t_i,
    \]
    where again the second equality is due to the definition of $u^t$, in \eqref{eq:usol}.
    
\end{proof}

\begin{proof}
We now resume the proof of \cref{cor:recovery_ortho} and assume that  $\mat$ is a tall (or square) matrix with orthonormal columns ($\mata A = I_\sigsize$), $\normd{\noise} = 0$ and $\explo^0 = \zerosig$. 
We remind the definition of $\itermax$ in (\ref{eq:Tmax}).
%Let $\iter \in \NN$, we remind the definitions of $\gradnoisemax$ in (\ref{eq:gradnoisemax}), $\gradnoise^\iter$ in (\ref{eq:gradnoise}) and $\usolt$ in (\ref{eq:usol}).

Using \cref{lem:ortho}, (\ref{eq:gradnoisemax}) and (\ref{eq:gradnoise}), we find that $\gradnoisemax = 0$. 
Therefore (\ref{eq:Hr}) holds for all $x^*$ and \cref{thm:recovery} implies that there exists $t_s\leq T_{max}$ such that $\suppsol \subseteq S^{t_s}$.  Since $\explo^0 = \zerosig$ and $\gradnoisemax = 0$, we find $\itermax = \spars + 1$.

Since $\normd{\noise} = 0$, we know from \cref{thm:recovery} and the definitions of $\iterb$ and $x^t$ in \cref{alg:SEA} that \[\normd{\mat\sig^\iterb - \obs} \leq \|\mat\sig^{t_s} -\obs\|_2 \leq \normd{\mat\sol - \obs} =0.
\]
Using that $\mat$ is a tall (or square) matrix with orthonormal columns ($\mata A = I_\sigsize$), and $\normd{\noise} = 0$, this leads to
\begin{align*}
    0 = & \mat\sigs^\iterb - y \\
    = & \mata\mat(\sigs^\iterb - \sol) \\
    = & \sigs^\iterb - \sol.
\end{align*}
Therefore, $S^* = \SUPP{\sol} =\SUPP{x^\iterb} \subseteq S^\iterb$.

This concludes the proof of \cref{cor:recovery_ortho}.
\end{proof}

%% file: DT_experiments.tex
\section{Additional Results for Phase Transition Diagram Experiment}
\label{app:dt}

We consider the same experiment as in \cref{dt-sec} but in the noiseless setting.
The analog of the curves of \cref{fig:hm} are in \cref{fig:hm_001}. Again, an artifact stemming from the discrete values of $(m,n,k)$ is responsible for the smooth and decreasing part observed on the left side of the phase transition curves, in a region where $k=1$.
Without noise, all algorithms exhibit a similar phase transition curve and maintain the same ranking as in the noisy setting. The conclusions that are drawn in the noiseless setting from \cref{fig:hm_001} are analog to those in \cref{dt-sec} in the noisy setting.

\begin{figure}[!htb]
    \centering
    \includegraphics[width=0.6\linewidth]{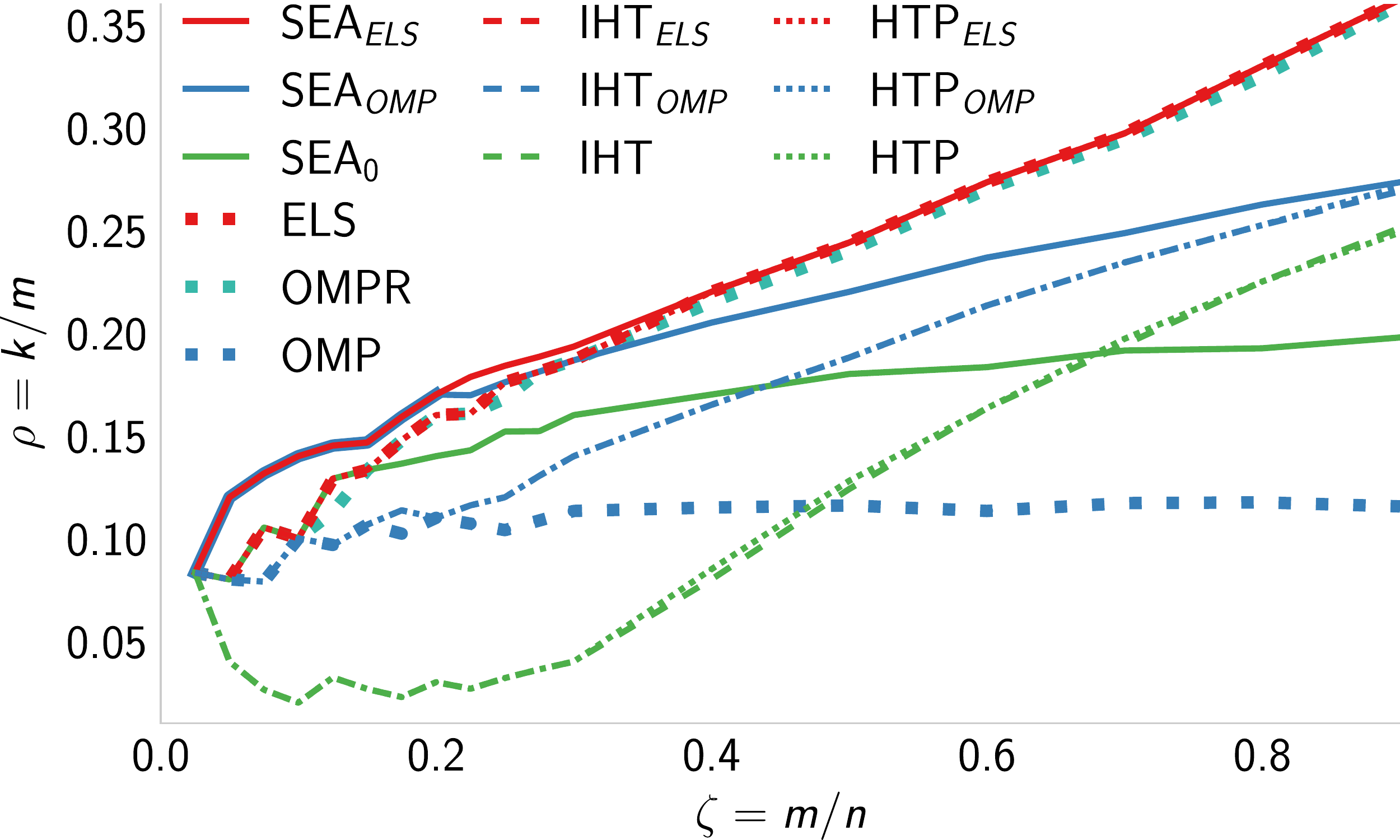}
    \caption{Phase transition diagram (noiseless setting). Problems below each curve are solved by the related algorithm with a success rate larger than $95\%$. $\zeta=m/n$ denotes the ratio between the number of rows and the number of columns in $A$ while $\rho=k/m$ denotes the ratio between the sparsity and the number of rows in $A$. Matrix $A$ have i.i.d. standard Gaussian entries and non-zero entries in $\sol$ are drawn uniformly in $[-2, -1]\cup[1, 2]$. $n=500$ is fixed and results are obtained from $1000$ runs.}
    \label{fig:hm_001}
\end{figure}

%% file: Deconv_experiment.tex
\section{Additional Results in Deconvolution}
\label{app:deconv}

To supplement \cref{deconv-sec}, we present additional results for the initial experimental setup. 
In \cref{app:dcv_prec}, we provide the results for the full signal of \cref{fig:dcv_prec}. 
In \cref{app:deconv-precise}, we display the loss along the iterative process for the experiment shown in \cref{fig:dcv_prec}. 
In \cref{app:deconv-n_support} (resp. \cref{app:deconv-global}), we show the average number of explored supports (resp. the average loss and Wasserstein distance) over the $r=200$ problems solved to construct \cref{fig:dcv_mass} as $k$ varies. 
We present results for the noiseless case when $\noise = \zeroobs$ in \cref{app-deconv-noiseless}.
Finally, we present the analog of \cref{fig:dcv_mass} for different configurations.
We depict the case where the sparsity provided to the algorithm is incorrect in \cref{app:deconv:wrong_k}.
The noise robustness is studied in \cref{app:deconv:noise_robustness} and the variant $\obs = \mat(\sol + \noise')$ of the initial problem $\obs = \mat\sol + \noise'$  is studied in \cref{app:deconv:noise_on_x}.
The impact of increasing the magnitude of $\frac{\normd{\sol}}{\min_{\ii\in\suppsol}\abs{\soli}}$, testing the importance of conditions \ref{eq:HRIP} and \ref{eq:HSRIP}, is shown in \cref{app:deconv:u_10}.
The impact of the variation of the step size for IHT and HTP is presented in \cref{app:deconv:step_size}.
Lastly, considerations on a random search are discussed in \cref{app:deconv:random}.

\subsection{Deconvolution: Examining the Specific Instance from \texorpdfstring{\cref{fig:dcv_prec}}{deconvolution experiment}}
\label{app:dcv_prec}

In \cref{fig:deconv:noisy_full}, we present the full signal corresponding to the cropped instance in \cref{fig:dcv_prec}. For clarity, \cref{fig:deconv:noisy_crop} displays the same crop as \cref{fig:dcv_prec}, presenting the results obtained with all studied algorithms.

In this representation, nearly all algorithms can identify isolated spikes. However, challenges arise when spikes are close, leading algorithms to struggle with precise localization. 
Notably, IHT and HTP exhibit false detections in the most energetic part of the signal (around positions 140, 180, 260, and 400), getting trapped in local minima. 
On this experiment (this is not the case in general), initializing SEA with ELS or OMP allows \SEAELS and \SEAOMP to find a better approximation of $\suppsol$ than \SEAZ. 
These two versions of SEA successfully recover the original signal, except for two spikes between positions 410 and 425.
In contrast, other algorithms fail due to the coherence of $A$ and the presence of additive noise.

\begin{figure}[!htb]
    \centering
    \includegraphics[width=\linewidth]{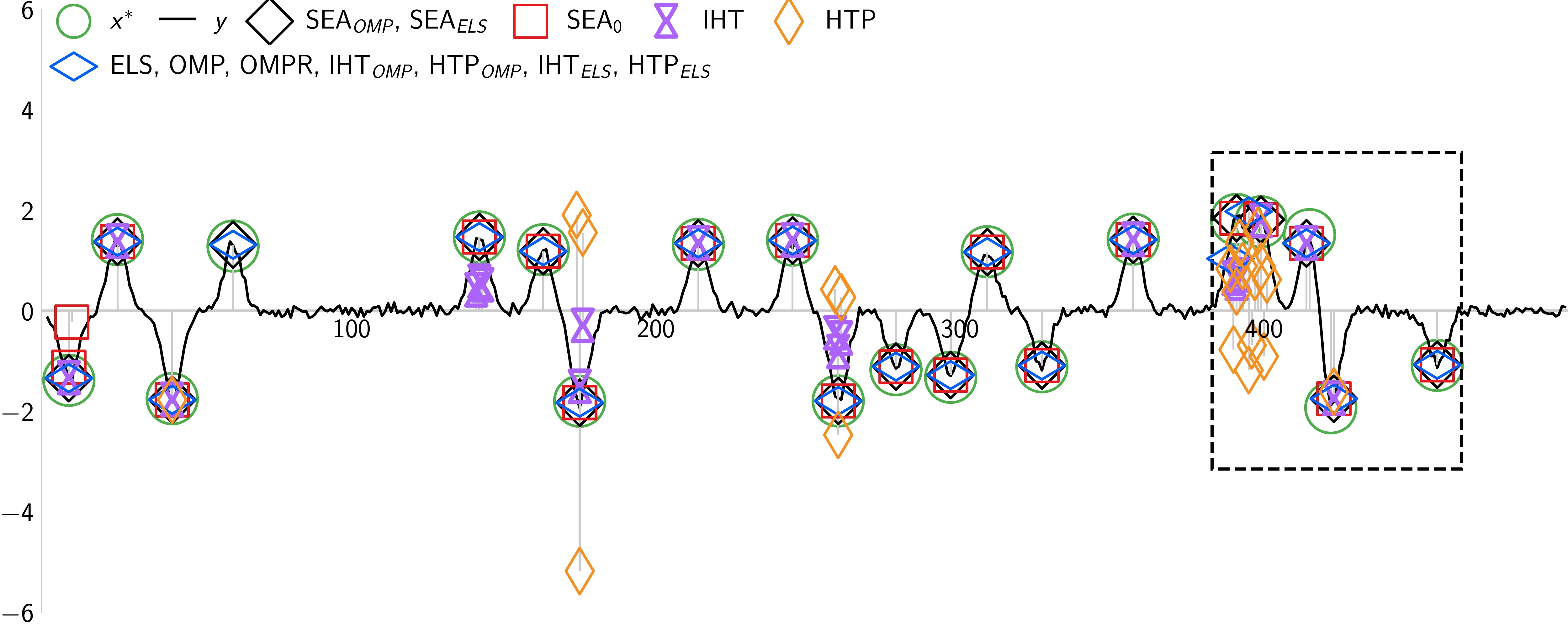}
    \caption{Full version of \cref{fig:dcv_prec}. Representation of an instance of $\sol$, $\obs$ and the solutions provided by the algorithms ($\spars = 20$, $n=500$).}\label{fig:deconv:noisy_full}
\end{figure}

\begin{figure}[!htb]
    \centering
    \includegraphics[width=\linewidth]{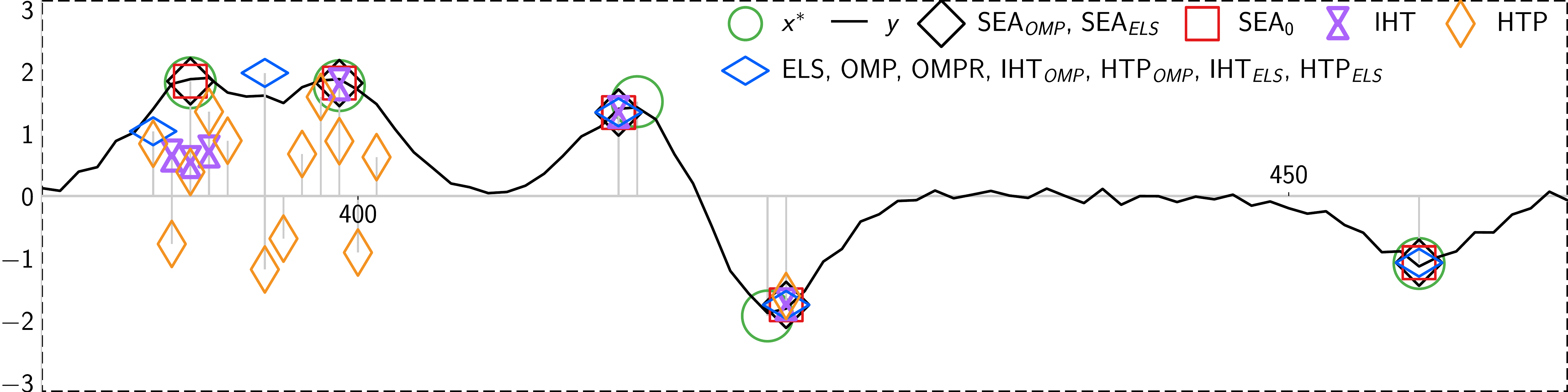}
    \caption{Crop from the dashed area in \cref{fig:deconv:noisy_full}, matching the location of \cref{fig:dcv_prec} with results from all analyzed algorithms. This region corresponds to the most densely populated area within the signal.}
    \label{fig:deconv:noisy_crop}
\end{figure}

\subsection{Deconvolution: Loss along the Iterative Process}
\label{app:deconv-precise}

\cref{fig:sea_0,fig:sea_0_n_sup} illustrate the behavior of HTP, IHT, ELS, \SEAZ, \SEAOMP and \SEAELS, for the same $20$-sparse vector $\sol$ used in \cref{fig:dcv_prec} (\cref{deconv-sec}) and \cref{fig:deconv:noisy_full,fig:deconv:noisy_crop} (\cref{app:dcv_prec}), throughout the iterative process.

\begin{figure}[!htb]
    \centering
    \includegraphics[width=\linewidth]{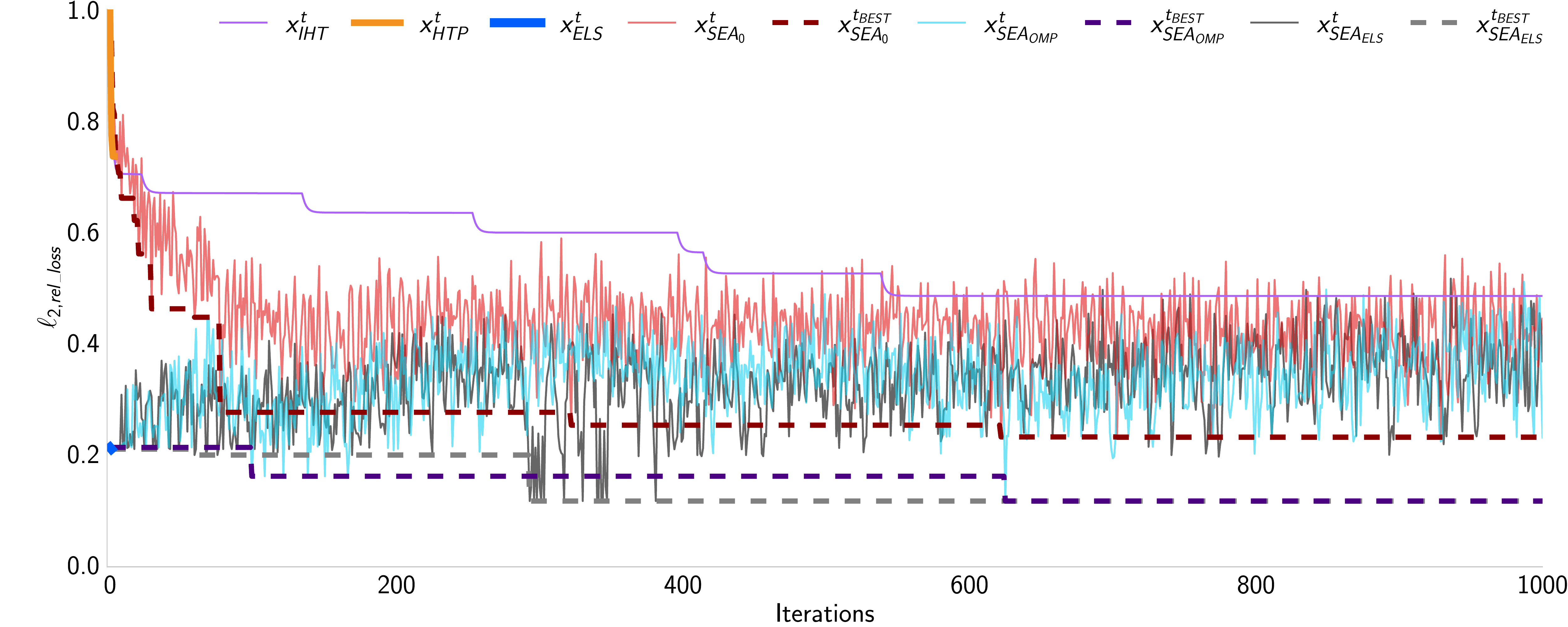}
    \caption{Representation of $\ell_{2, \text{rel}\_\text{loss}}(\sigst)$ (solid lines) and $\ell_{2, \text{rel}\_\text{loss}}(\sigs^{\iterb(t)})$ (dashed lines) for each iteration of several algorithms, for the experiment of \cref{fig:dcv_prec}.}
    \label{fig:sea_0}
\end{figure}

\begin{figure}[!htb]
    \centering
    \includegraphics[width=\linewidth]{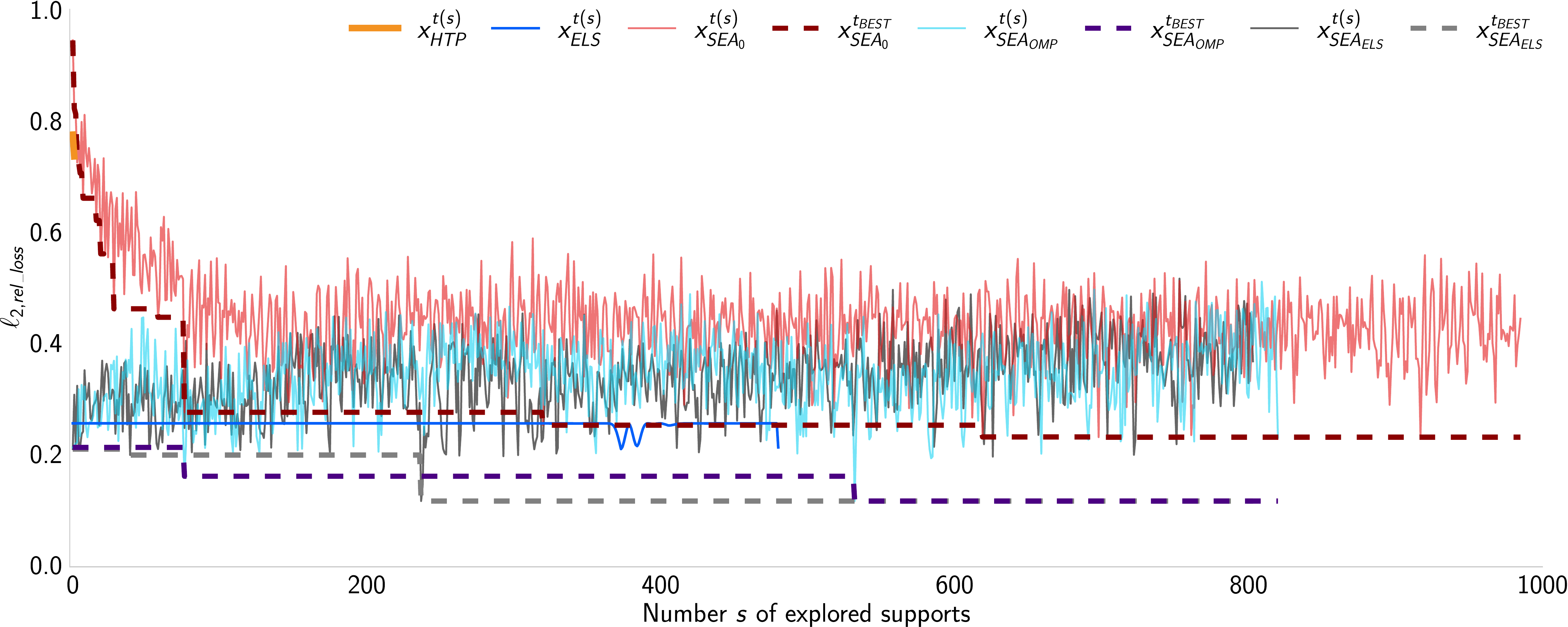}
    \caption{Representation of $\ell_{2, \text{rel}\_\text{loss}}(x^{t(s)})$ (solid lines) and $\ell_{2, \text{rel}\_\text{loss}}(x^{\iterb(\iter(s))})$ (dashed lines) for each new explored support of several algorithms, for the experiment of \cref{fig:dcv_prec}.}
    \label{fig:sea_0_n_sup}
\end{figure}

More precisely, in \cref{fig:sea_0} the solid curves represent $\ell_{2, \text{rel}\_\text{loss}}(\sigst)$ when $\iter$ varies in $\intint{0,1000}$, where $\ell_{2, \text{rel}\_\text{loss}}$ is defined by 
\begin{equation}
    \ell_{2, \text{rel}\_\text{loss}}(x) = \frac{\normd{\mat\sig - \obs}}{\normd{\obs}}.
    \label{eq:l2rel_loss}
\end{equation}
The dashed lines represent $\ell_{2, \text{rel}\_\text{loss}}(\sigs^{\iterb(t)})$ where $\iterb(t) = \underset{\itero \in \intint{0, \iter}}{\argmin} \normdd{\mat\sigs^\itero - \obs}$ and  $\iter$ varies in $\intint{0,1000}$.

Overall, no algorithm succeeds in reaching zero error. 
ELS performs only one iteration before stopping in a local minimum ($\ell_{2, \text{rel}\_\text{loss}}(x^{1}) \approx 0.2$). 
HTP completes a few iterations before stopping. 
IHT outperforms HTP by exploring a bit more. 
One can observe that, due to the exploratory nature of SEA, $\ell_{2, \text{rel}\_\text{loss}}(\sigs^{\iter})$ oscillates for both versions of SEA. 
This exploration enables \SEAELS to refine the ELS estimate within 300 iterations. Despite faster decay around the 100$^{th}$ iteration, \SEAOMP finally reaches \SEAELS after 620 iterations.

We observe that HTP and IHT exhibit poor performance due to the high coherence of $\mat$. 
As demonstrated in \cref{app:dcv_prec}, these algorithms initially make the mistake of erroneously assigning several neighboring atoms to represent the same large bump and fail to correct this error during the iterative process.

\cref{fig:sea_0_n_sup} illustrates the same iterative process as \cref{fig:sea_0}, focusing on support exploration rather than the iteration count for each algorithm.
Here, the solid curves represent $\ell_{2, \text{rel}\_\text{loss}}(x^{t(s)})$ when $s$ varies from $0$ to the number of explored supports, where $t(s)$ is the iteration associated to the $s^{th}$ explored support (without redundancy).
As in the previous figure, the dashed lines represent $\ell_{2, \text{rel}\_\text{loss}}(\sigs^{\iterb(t(s))})$ where $\iterb(t(s)) = \underset{\itero \in \intint{0, \iter(s)}}{\argmin} \normdd{\mat\sigs^\itero - \obs}$.
This is the loss associated to the best estimate found while exploring the $s^{th}$ supports.

We observe that HTP explores very little before stopping in a local minimum.
Despite performing only one iteration, ELS explores $500$ supports within the neighborhood of its OMP initialization for a slight improvement. 
Here, \SEAZ explores one new support at each iteration, while \SEAELS explores fewer, improving upon ELS by exploring less than 250 supports. Again, despite faster decay at the beginning, \SEAOMP finally reaches \SEAELS after exploring around $520$ supports.
This reveals how efficient each algorithm is at finding relevant supports.

\subsection{Deconvolution: Number of Explored Supports}\label{app:deconv-n_support}

As discussed in \cref{complexity-sec}, the overall cost of the algorithms depends on the number of explored supports.
In \cref{fig:n_supports:noisy}, we illustrate the number of explored supports in two different ways. 
First, in \cref{fig:n_supports:noisy} (left), we present the average number of explored supports for the entire problem resolution — representing the overall cost.
This includes supports explored before initialization. For instance, \SEAOMP includes both the supports explored by OMP for its initialization and those explored subsequently in the SEA procedure.
Then, in \cref{fig:n_supports:noisy} (right), we present the average number of explored supports that actually required computation after initialization.
These curves reveal the cost of the algorithms after initialization, where supports seen before the initialization (e.g., those of OMP for \SEAOMP) are not included as they do not incur additional computing time.

\begin{figure}[!htb]
    \centering
    \includegraphics[width=\linewidth]{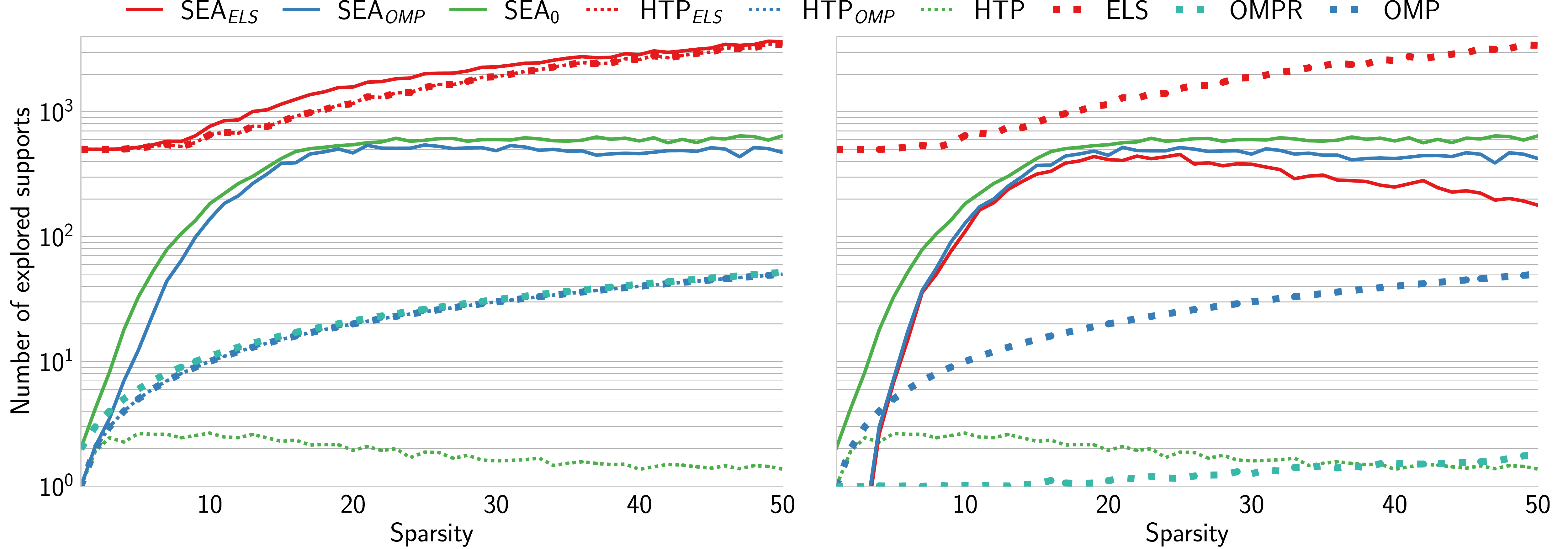}
    \caption{Left: Average number of explored supports by algorithms solving the 200 problems in \cref{deconv-sec}, across sparsity levels $\spars\in\intint{1,50}$.
    Right: Average number of explored supports from algorithms initialization in the same setup.}
    \label{fig:n_supports:noisy}
\end{figure}

Examining the overall cost of the algorithms on the left, we observe three types of exploration profiles.
Some algorithms, such as OMP, OMPR, and HTP, exhibit minimal exploration.
Notably, as $k$ increases, HTP explores fewer supports. 
On the other hand, algorithms like ELS explore extensively.
SEA falls in between, exploring a few supports for small $k$ and more as $k$ increases until reaching a threshold.
Despite exploring at least two times fewer supports than ELS, SEA's more efficient exploration allows it to achieve better results, as demonstrated in \cref{fig:dcv_mass}.

From this figure, we observe that adding SEA to ELS (\SEAELS) does not significantly alter the order of magnitude of the cost.
Turning our attention to the cost after initialization on the right, we do not observe \HTPELS and \HTPOMP because they do not explore after their initialization, as shown in \cref{fig:sea_0_n_sup} for \HTPELS. 
OMP and \SEAZ curves remain unchanged because they do not have any initializing algorithms.

All SEA variants exhibit a similar order of magnitude of explored supports. 
However, we conclude that the stronger the initialization (with 0 $<$ OMP $<$ ELS), the more challenging the exploration becomes due to the high coherence of $\mat$ and the local minima in which OMP and ELS end up.

\subsection{Deconvolution: Average Loss and Wasserstein Distance when \texorpdfstring{$k$}{k} Varies}
\label{app:deconv-global}

In this section, we complement the analysis of the experiment described in \cref{deconv-sec}, the results of which are already depicted in \cref{fig:dcv_mass}.

In \cref{fig:dcv_mass_y}, we present the average -- over the $r=200$ problems -- of the relative $\ell_2$ loss ($\ell_{2, \text{rel}\_\text{loss}}$), defined in \eqref{eq:l2rel_loss}, for the outputs of all algorithms and for
$\spars\in\intint{1,50}$. 
We observe that all versions of SEA achieve the lowest errors for $\spars < 20$. 
The largest gap between SEA and its competitors is observed for $\spars$ between $9$ and $13$. 
For clarity, the curves for HTP and IHT are visible for small $\spars$ only.
Due to the high coherence of $\mat$ and their method of selecting multiple elements of the support estimate at once, both IHT and HTP attempt to reconstruct single peaks with multiple atoms, leading to much larger errors than those of the competitors.

\begin{figure}[!htb]
\centering
\begin{minipage}{.48\linewidth}
  \centering
    \includegraphics[width=\linewidth]{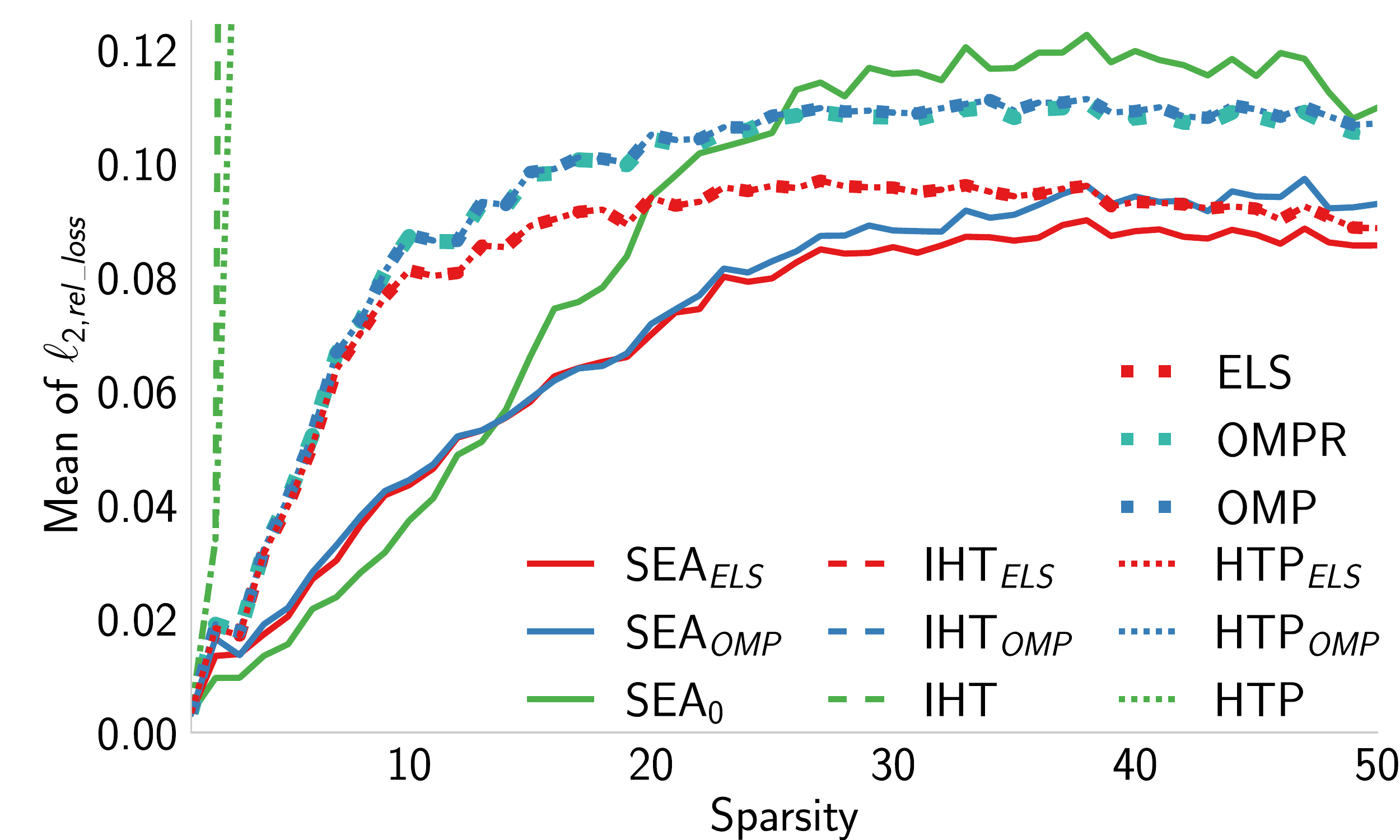}  
    \caption{Mean of $\ell_{2, \text{rel}\_\text{loss}}(x)$ -- defined in \eqref{eq:l2rel_loss} -- for the outputs of the algorithms on the $200$ problems of \cref{deconv-sec}, for each sparsity level $\spars\in\intint{1,50}$.}
    \label{fig:dcv_mass_y}
\end{minipage}%
\hspace{.02\linewidth}
\begin{minipage}{.48\linewidth}
  \centering
    \includegraphics[width=\linewidth]{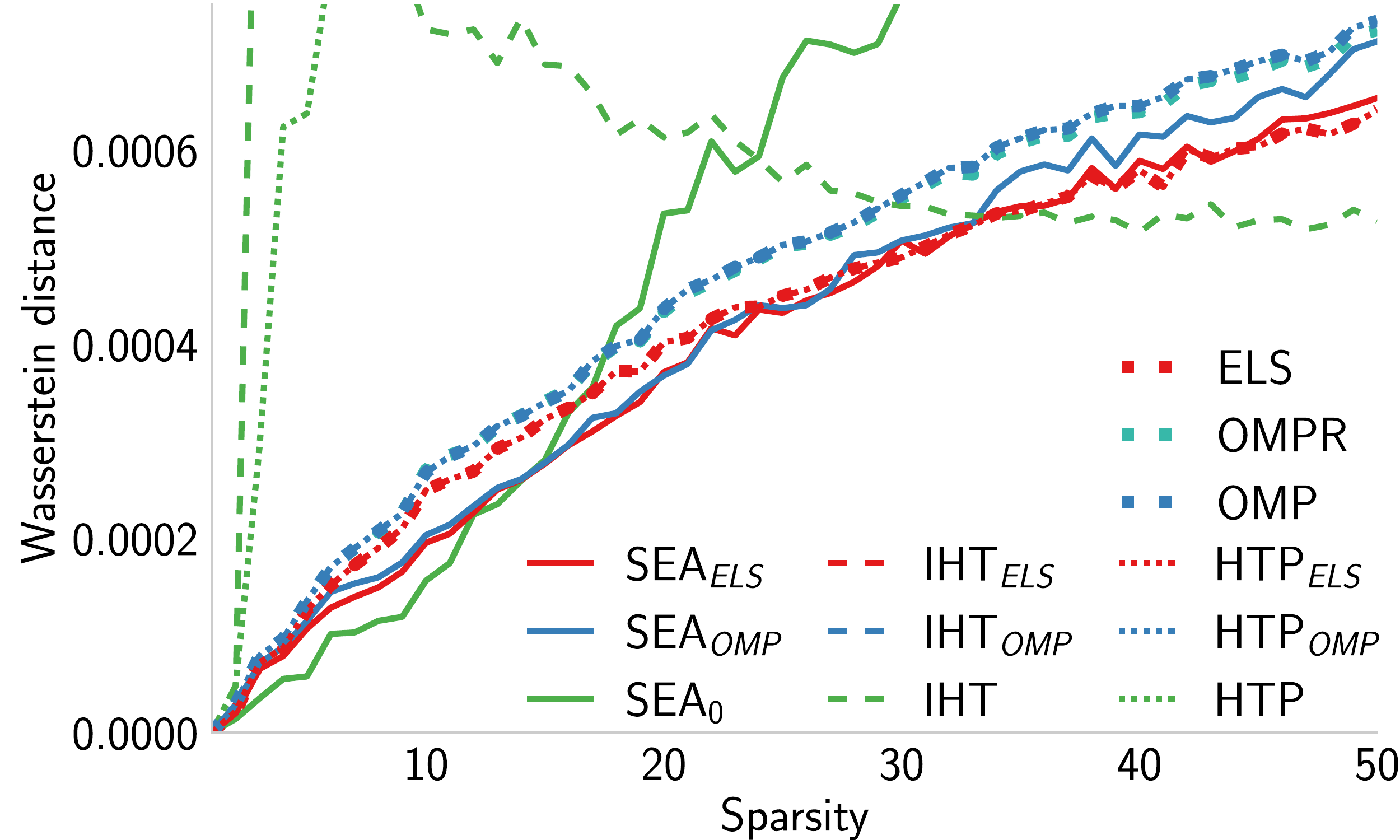}
    \caption{Mean of the Wassertstein distance between the outputs of the algorithms and the solutions $\sol$ of the $200$ problems of \cref{deconv-sec}, for each sparsity level $\spars\in\intint{1,50}$.}
  \label{fig:dcv_mass_y:wasserstein}
\end{minipage}
\end{figure}

In \cref{fig:dcv_mass_y:wasserstein}, we show the mean of the Wasserstein-1-distance (also called Earth mover's distance) over the same problems.
It illustrates how 'far' the chosen spikes are from the true ones.
Again, all the versions of SEA achieve the smallest distances for $k < 18$.
As $k$ increases, despite being the best at finding the exact position of the spikes (see \cref{fig:dcv_mass}), \SEAZ and, to a lesser extent, \SEAOMP and \SEAELS, choose spikes 'far' from the true ones when they are mistaken, while IHT improves for the highest $k$.

\subsection{Deconvolution: Results in the Noiseless Setup} \label{app-deconv-noiseless}
We consider the same experiment as in \cref{deconv-sec} but in a noiseless setting ($\noise = 0$).
Thus, we set again n = 500, a convolution matrix $\mat$ corresponding to
a Gaussian filter with a standard deviation equal to 3.
We tested every algorithm on $r=200$ noiseless problems, for different $k$-sparse $\sol$, with $k \in \intint{1, 50}$.

\subsubsection{Visualization of a Specific Instance}
\label{app:deconv:precise}

The counterparts of the curves in \cref{fig:deconv:noisy_full,fig:deconv:noisy_crop} from \cref{app:deconv-precise} in the noiseless case are shown in \cref{fig:deconv:noiseless_full,fig:deconv:noiseless_crop}. 
The algorithms behave in a similar way to the noisy case. 
However, with no perturbation in the signal, all versions of SEA successfully recover the exact positions of the spikes, whereas no other algorithm achieves such a performance.

\begin{figure}[!htb]
    \centering
    \includegraphics[width=\linewidth]{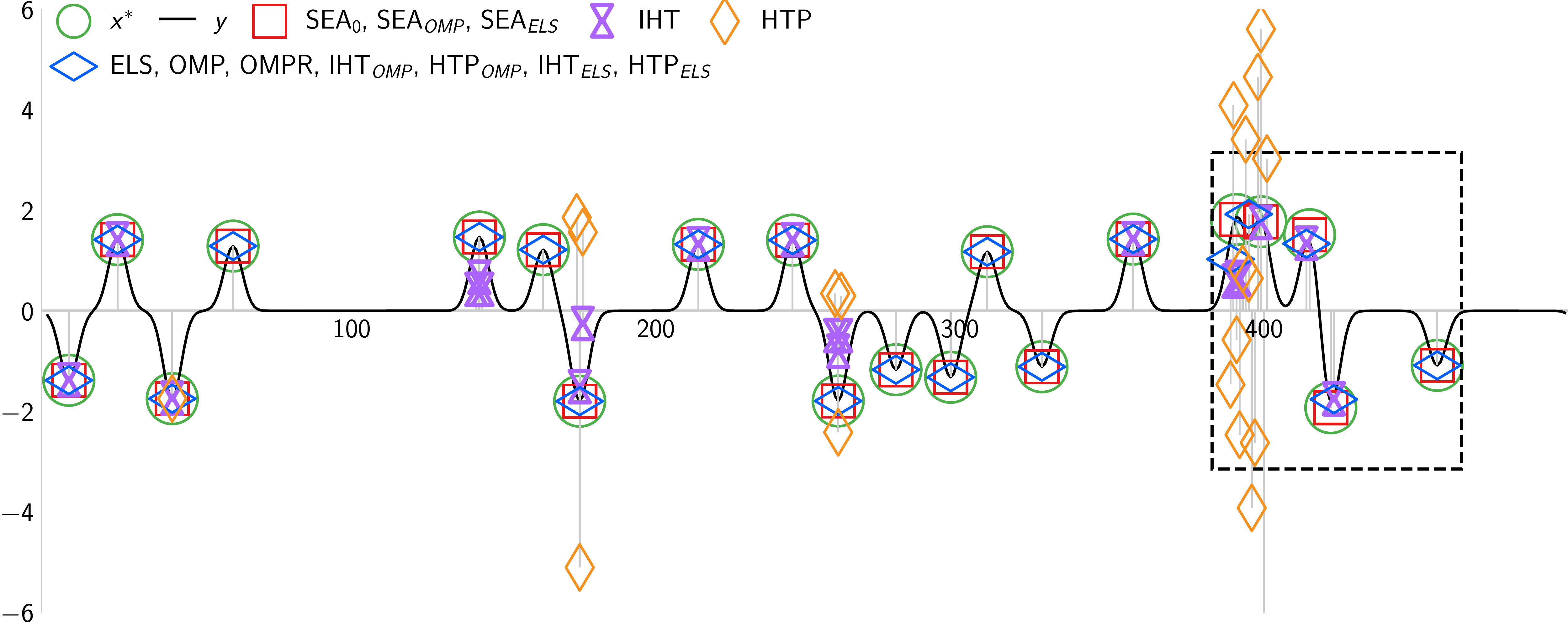}
    \caption{Representation of an instance of $\sol$ and $\obs$ with the solutions provided by the algorithms when $\spars = 20$ and $n=500$ in the noiseless case: Full signal.}
    \label{fig:deconv:noiseless_full}
\end{figure}

\begin{figure}[!htb]
    \centering
    \includegraphics[width=\linewidth]{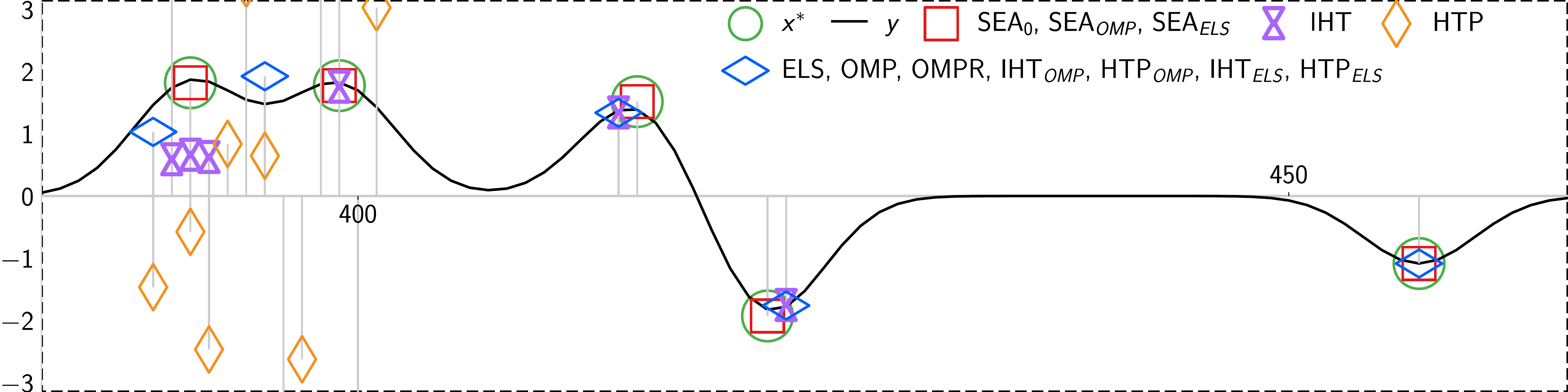}
    \caption{Crop from the dashed area in \cref{fig:deconv:noiseless_full}. This region corresponds to the most densely populated area within the signal.}
    \label{fig:deconv:noiseless_crop}
\end{figure}

\subsubsection{Loss along the Iterative Process}
\label{app:deconv:loss}

The analogs of \cref{fig:sea_0,fig:sea_0_n_sup} from \cref{app:deconv-precise} are respectively shown in \cref{fig:sea_0_noiseless,fig:sea_0_n_sup_noiseless}. 
The conclusions drawn here in the noiseless setting are similar to those in \cref{app:deconv-precise}.

Similarly to the noisy case, it can be observed in \cref{fig:sea_0_noiseless} that due to the exploratory nature of SEA,  $\ell_{2, \text{rel}\_\text{loss}}(\sigs^{\iter})$ oscillates for all versions of SEA. 
However, this does not prevent \SEAZ from finding a better approximation of $\suppsol$ than ELS in the first $80$ iterations and eventually recovering $\suppsol$ despite the high coherence of $\mat$.
Indeed, for all SEA versions, once $\suppsol$ is recovered in the noiseless setting, for $t$ sufficiently large, $x^t=x^*$, and therefore $Ax^t-y=0$.
Using the update rule of $\explo^t$ in line \ref{line:SEA:update_explo} of \cref{alg:SEA}, we observe that $\explo^t$ should no longer evolve, and no new support is explored. 
This behavior is evident not only from \cref{fig:sea_0_noiseless} but also from \cref{fig:sea_0_n_sup_noiseless}.
Furthermore, as can be seen in \cref{fig:deconv:noiseless_full}, ELS does not improve OMP thus leading to an identical initialization for \SEAOMP and \SEAELS. Consequently, these last two algorithms thus follow the same trajectory.
From \cref{fig:sea_0_n_sup_noiseless}, we however observe that\SEAELS and \SEAOMP must explore twice as many supports as \SEAZ.

\begin{figure}[tbhp]
    \centering
    \includegraphics[width=1\linewidth]{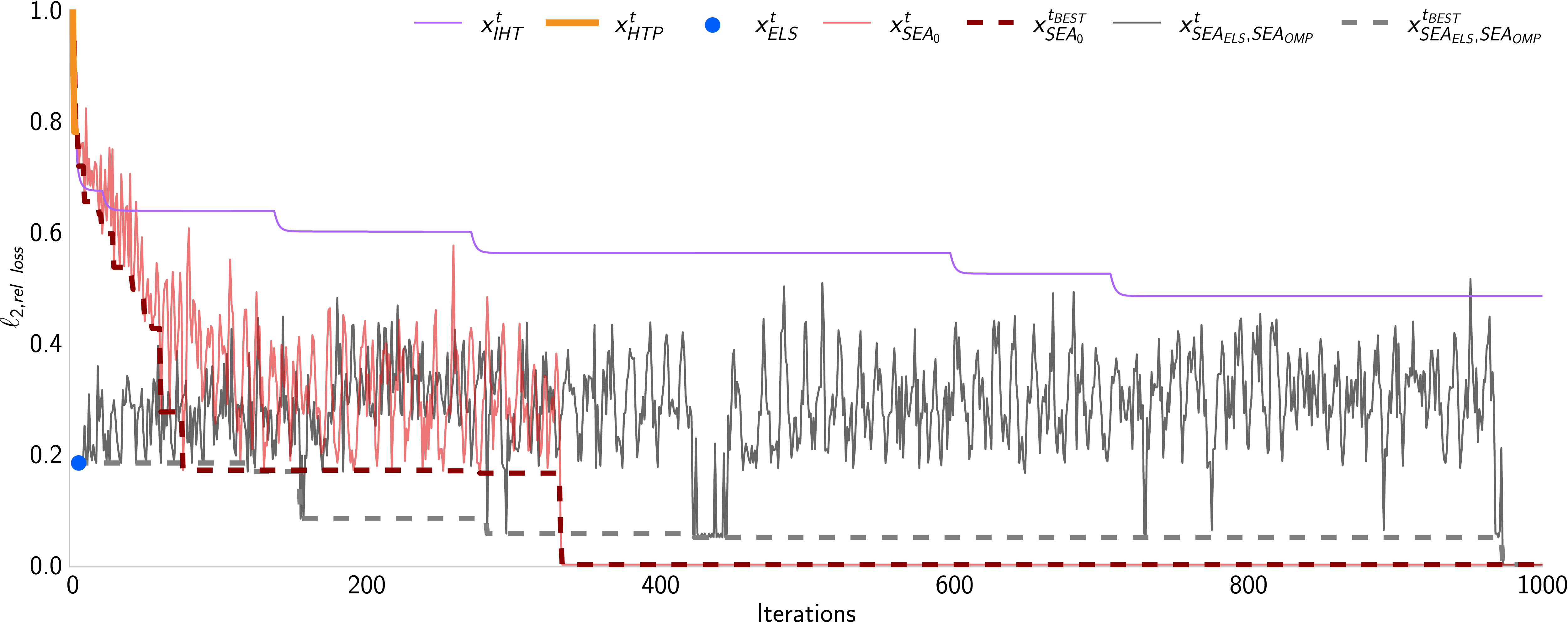}
    \caption{Representation of $\ell_{2, \text{rel}\_\text{loss}}(\sigst)$ (solid lines) and $\ell_{2, \text{rel}\_\text{loss}}(\sigs^{\iterb(t)})$ (dashed lines) for each iteration of several algorithms, for the noiseless experiment.}
    \label{fig:sea_0_noiseless}
\end{figure}

\begin{figure}[tbhp]
    \centering
    \includegraphics[width=1\linewidth]{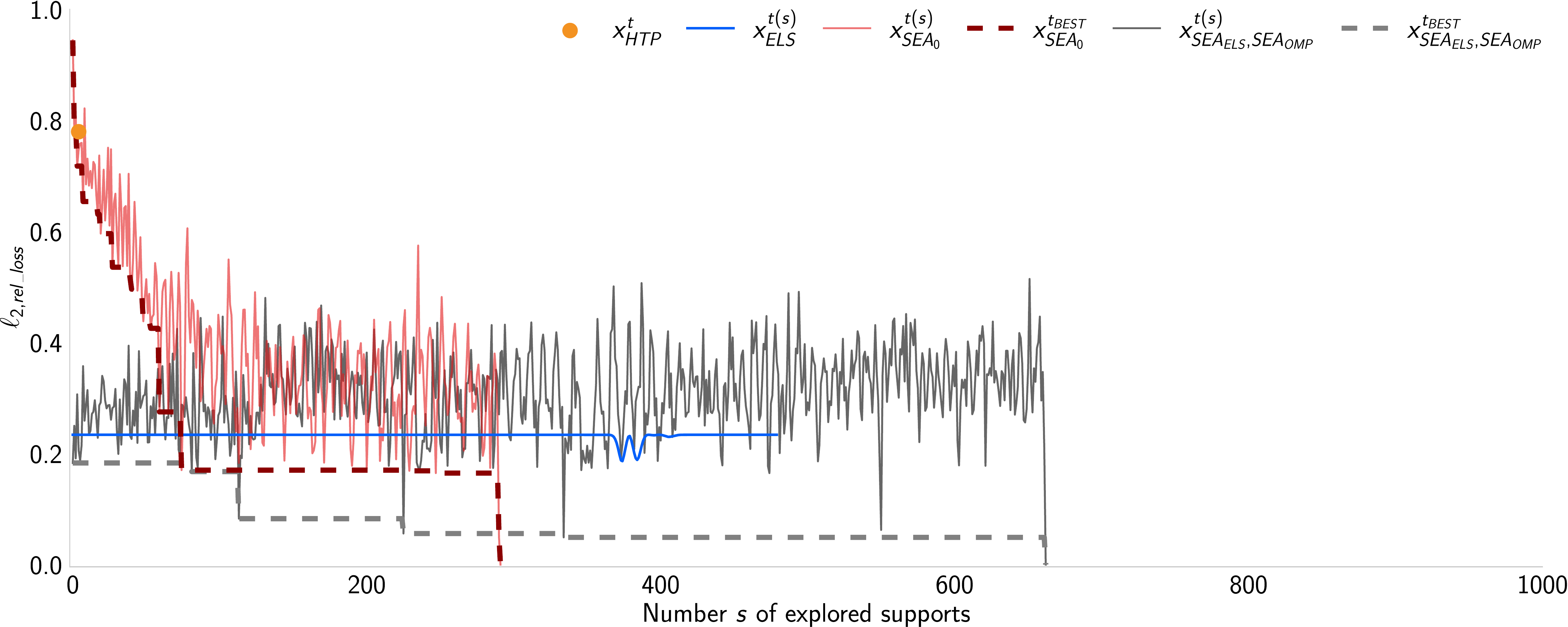}
    \caption{Representation of $\ell_{2, \text{rel}\_\text{loss}}(x^{t(s)})$ (solid lines) and $\ell_{2, \text{rel}\_\text{loss}}(x^{\iterb(\iter(s))})$ (dashed lines) for each new explored support of several algorithms, for the noiseless experiment.}
    \label{fig:sea_0_n_sup_noiseless}
\end{figure}

\subsubsection{Number of Explored Supports}

The analog of \cref{fig:n_supports:noisy} from \cref{app:deconv-n_support} is shown in \cref{fig:n_supports:noiseless}.
The conclusions drawn here in the noiseless case are similar to those in \cref{app:deconv-n_support}.

All the algorithms behave in the same way as in the noisy experiment. 
%Generally speaking, all the algorithms explore more supports in the noiseless case than in the noisy case, this wider exploration being slightly more visible for HTP and SEA.

\begin{figure}[!htb]
    \centering
    \includegraphics[width=\linewidth]{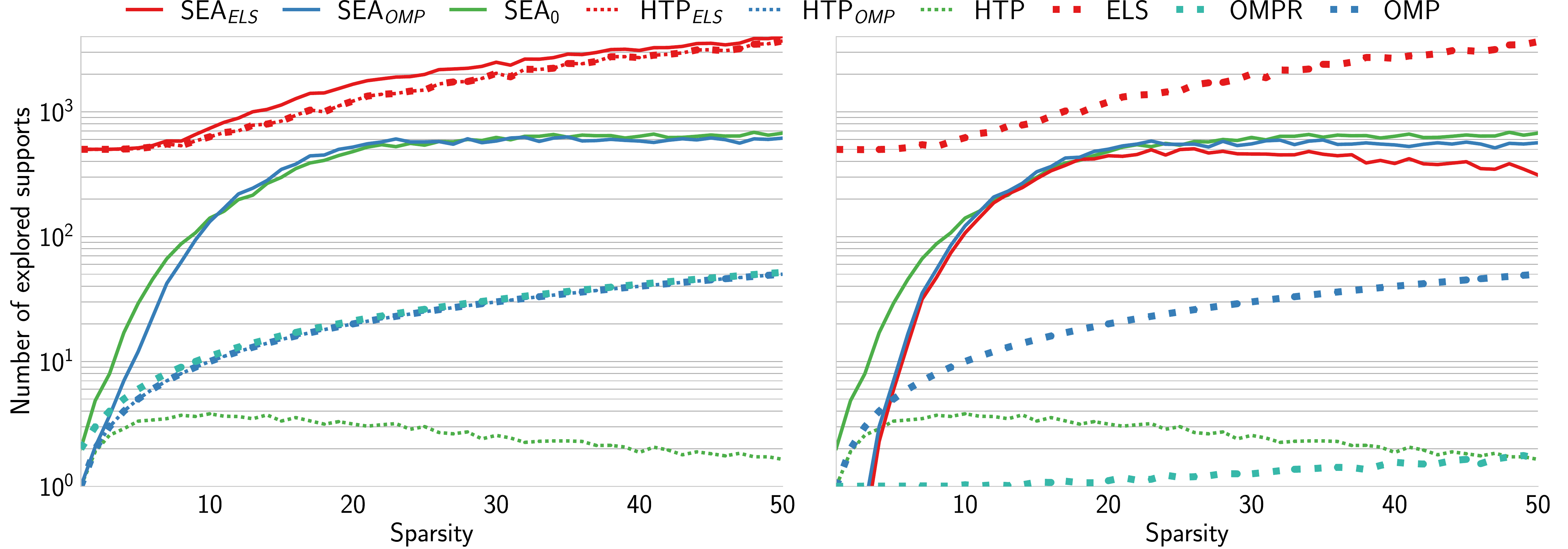}
    \caption{Left: Mean of the number of the explored supports by algorithms solving the 200 problems in the noiseless case, across sparsity levels $\spars\in\intint{1,50}$.
    Right: Mean of the number of the explored supports from algorithms initialization in the same setup.}
    \label{fig:n_supports:noiseless}
\end{figure}

\subsubsection{Average \texorpdfstring{$\text{dist}_{\text{supp}}$}{support distance}, Loss, and  Wasserstein Distance when \texorpdfstring{$k$}{k} Varies}

The analogs of \cref{fig:dcv_mass,fig:dcv_mass_y,fig:dcv_mass_y:wasserstein} are respectively displayed in \cref{fig:dcv_mass_noiseless,fig:dcv_mass_y_noiseless,fig:dcv_mass_y_noiseless:wasserstein}. 
The results in the noiseless setting closely mirror those in \cref{deconv-sec} and \cref{app:deconv-global} in the noisy setting.

In \cref{fig:dcv_mass_noiseless}, for sparsity levels $\spars < 30$, \SEAZ, \SEAOMP, and \SEAELS outperform the other algorithms. 
Across all studied sparsity levels, \SEAZ is reaching the best performances.

Moving to \cref{fig:dcv_mass_y_noiseless}, the absence of noise makes the problems easier to solve.
Despite overall improvement, \SEAZ still attains the lowest error for the smallest $k$, followed by \SEAELS.
Once again, HTP and IHT exhibit much larger errors than their competitors.

In \cref{fig:dcv_mass_y_noiseless:wasserstein}, as $k$ increases, \SEAZ, followed by \SEAELS and eventually IHT, exhibits the lowest Wasserstein distance.

\begin{figure}[tbhp]
    \centering
    \includegraphics[width=0.6\linewidth]{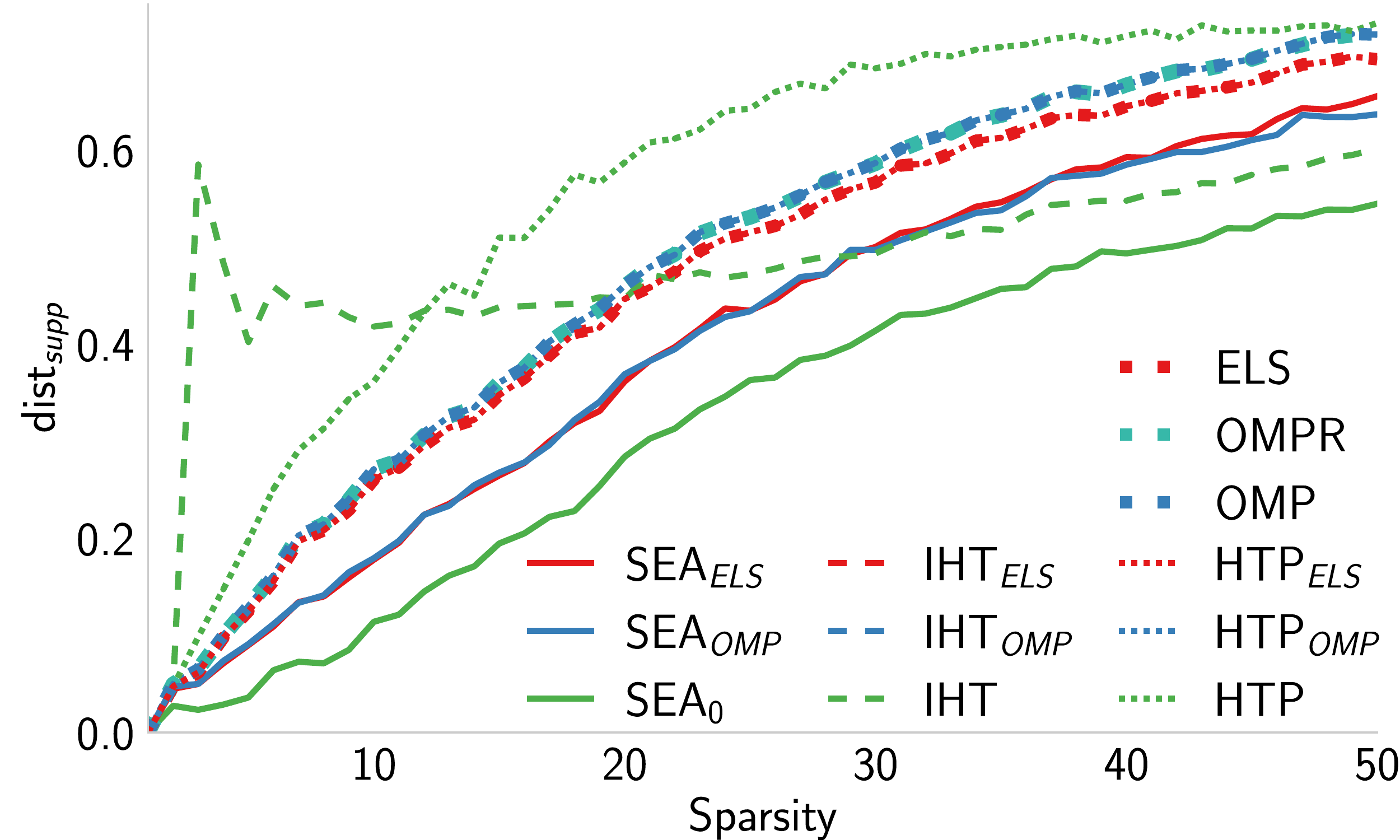}
    \caption{Mean of support distance $\text{dist}_{\text{supp}}$ (defined in \eqref{eq:supp_dist}) between $\suppsol$ and the support of the solutions provided by several algorithms as a function of the sparsity level $\spars$ in the noiseless setup.}
    \label{fig:dcv_mass_noiseless}
\end{figure}

\begin{figure}[!htb]
\centering
\begin{minipage}{.48\linewidth}
  \centering
    \includegraphics[width=\linewidth]{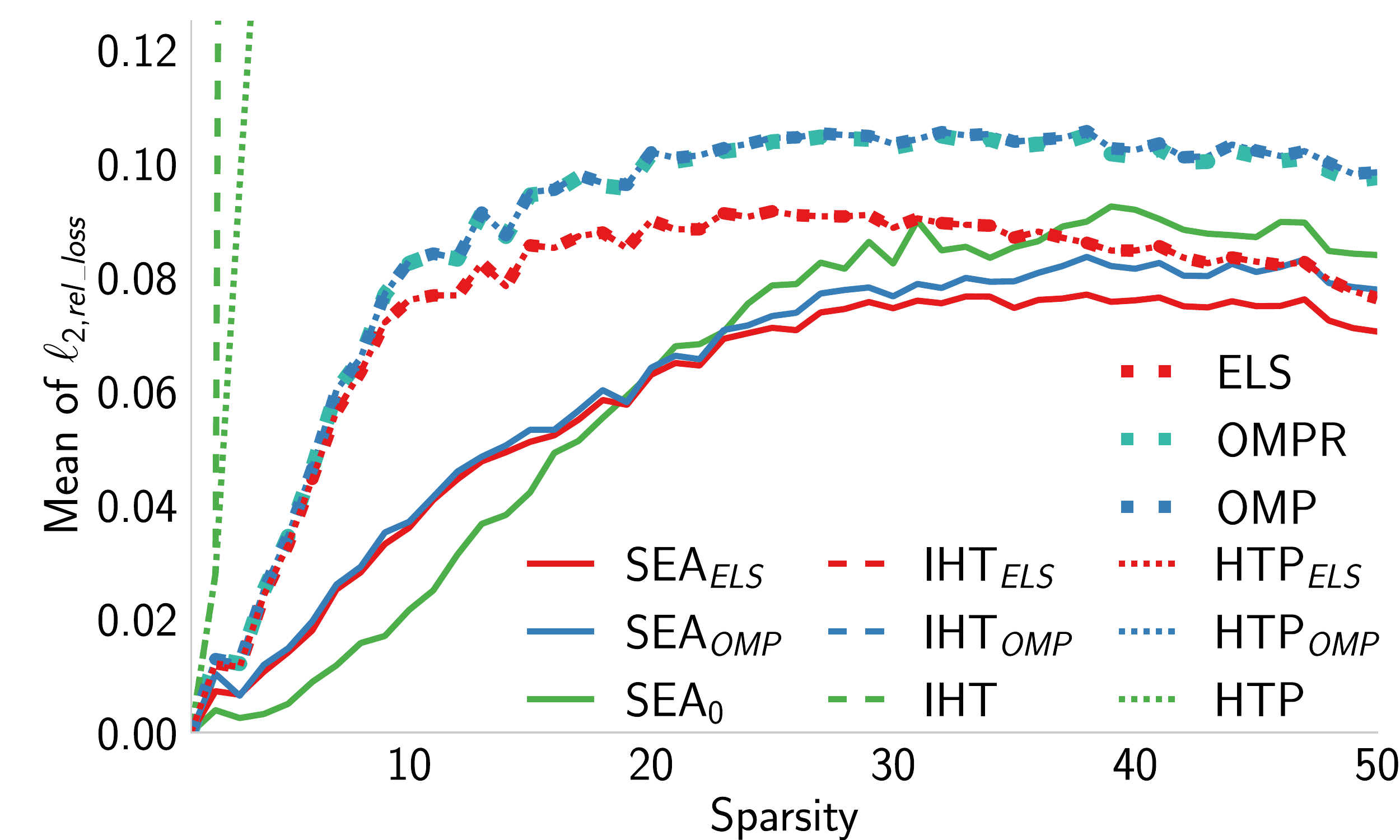}  \caption{Mean of $\ell_{2, \text{rel}\_\text{loss}}(x)$ -- defined in \eqref{eq:l2rel_loss} -- for the outputs of the algorithms on the $200$ problems of the noiseless setup.}
    \label{fig:dcv_mass_y_noiseless}
\end{minipage}%
\hspace{.02\linewidth}
\begin{minipage}{.48\linewidth}
  \centering
    \includegraphics[width=\linewidth]{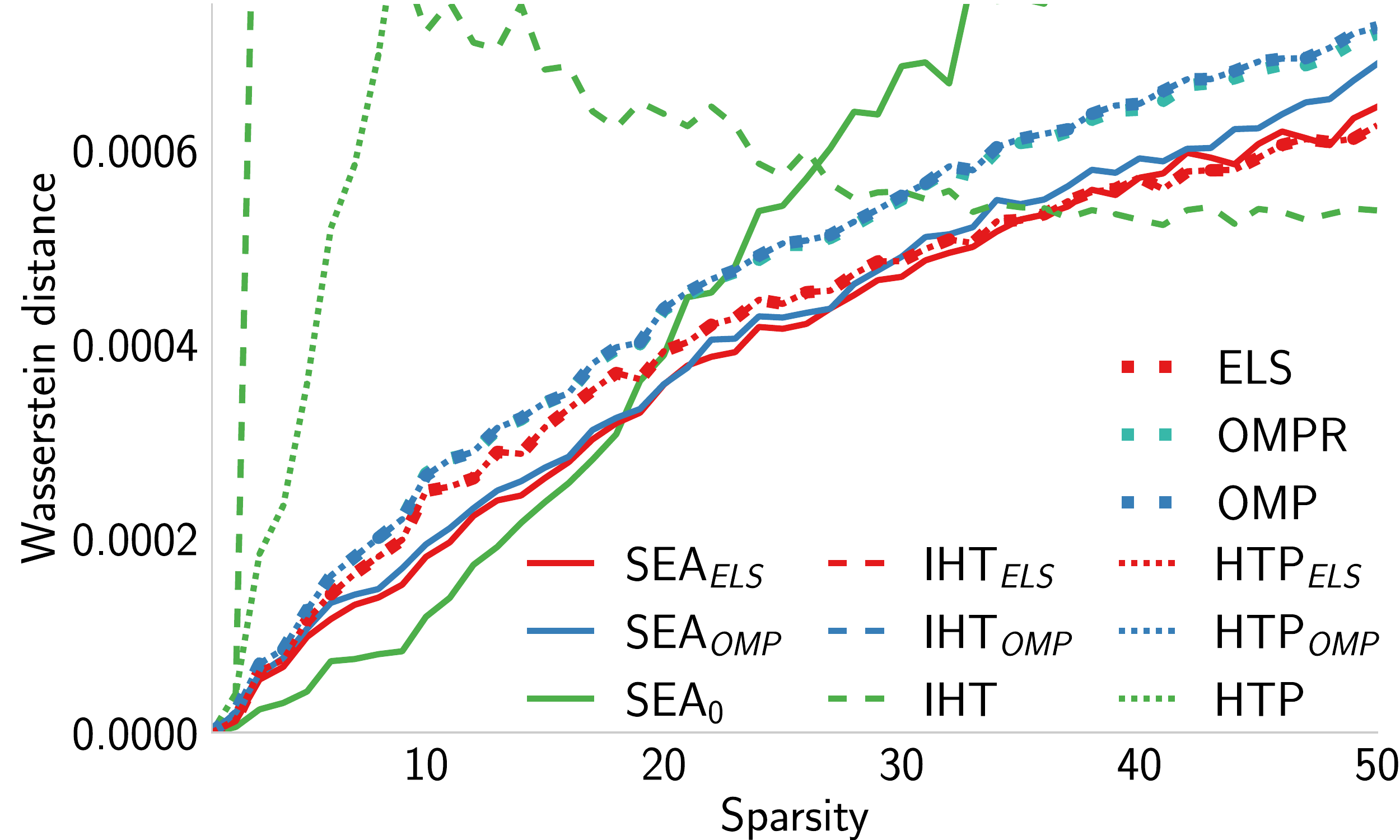}
    \caption{Mean of the Wassertstein distance between the outputs of the algorithms and the solutions $\sol$ of the $200$ problems of the noiseless setup.}
  \label{fig:dcv_mass_y_noiseless:wasserstein}
\end{minipage}
\end{figure}

\subsection{Deconvolution: Impact of an Erroneous Input Sparsity \texorpdfstring{$\spars$}{k}}
\label{app:deconv:wrong_k}

We consider the same experiment as in \cref{deconv-sec} but in the case where the sparsity $\spars$ is wrongly estimated.
Thus, we set $n = 500$ and use a convolution matrix $\mat$ corresponding to a Gaussian filter with a standard deviation of 3. Every algorithm is tested on $r = 200$ noisy problems, for different $k$-sparse $\sol$, with $\spars\in \{ 4p, \ p\in \intint{1, 12}\}$.
The difference here is that the algorithms are asked to recover a signal of sparsity $\spars' \neq \spars$.

The analogs of \cref{fig:dcv_mass} when $\spars'= 0.75\spars$ and $\spars'=1.25\spars$ are depicted in \cref{fig:dcv_mass:k075,fig:dcv_mass:k125}.
Thus, not all values $\spars\in\intint{1, 50}$ are tested because we wanted to keep a constant ratio $\frac{\spars'}{\spars}$ equal to $0.75$ or $1.25$.
To ensure fairness in the underestimated case, we introduce a slightly different metric from dist$_\text{supp}$, which allows algorithms to achieve a distance equal to $0$:

\begin{equation*}
    \text{dist}_{\text{supp},\spars'}(x)
    =
    \frac
    {\spars' - \abs{\suppsol \cap \ \SUPP{\sig}}}
    {\spars'}.
\end{equation*}

This metric shows how good is each algorithm at recovering only elements of $\suppsol$.

In \cref{fig:dcv_mass:k075}, we see that improving OMP and ELS is more difficult for \SEAOMP and \SEAELS, with SEA keeping the lowest $\text{dist}_{\text{supp}}$ as in \cref{fig:dcv_mass}.
In \cref{fig:dcv_mass:k075,fig:dcv_mass:k125}, we see that all algorithms reach a lower $\text{dist}_{\text{supp}}$ than in \cref{fig:dcv_mass}.
Thus, in these configurations, \SEAZ is still the algorithm reaching the lowest $\text{dist}_{\text{supp}}$. 
The better performance for $\frac{\spars'}{\spars}>1$ is expected according to the definition of the mean support distance. 
Additionally, when $\frac{\spars'}{\spars}<1$, the algorithms can focus on the largest entries of $\sol$ which are easier to recover.

\cref{fig:dcv_mass:k075_topk,fig:dcv_mass:k125_topk} complement the results presented in \cref{fig:dcv_mass:k075,fig:dcv_mass:k125} with another metric.
Here we illustrate how the performance degrades when evaluating the mean support distance based on the largest estimated entries.
Thus we introduce $K = \min \{\spars, \spars'\}$, and for any $\sig \in \RR^n$ and $\ii \in \intint{1, n}$

\begin{equation*}
    \text{dist}_\text{supp,largest}(\sig)
    =
    \frac
    {K - \abs{\SUPP{\sol_{\text{largest}_K}} \cap \SUPP{\sig_{\text{largest}_K}}}}
    {K}
    \quad \text{with} \
        (\sig_{\text{largest}_K})_\ii = 
    \begin{cases}
        \sig_\ii & \text{if } \ii \in \text{largest}_K(\sig) \\
        0 & \text{if } \ii \notin \text{largest}_K(\sig).
    \end{cases}
\end{equation*}

When $\spars$ is underestimated (\cref{fig:dcv_mass:k075_topk}), $\text{dist}_\text{supp,largest}$ depicts the capacity of each algorithm to recover $\SUPP{\sol_{\text{largest}_K}}$, the support of the largest entries of $\sol$.
When $\spars$ is overestimated (\cref{fig:dcv_mass:k125_topk}), $\text{dist}_\text{supp,largest}$ depicts the capacity of each algorithm to recover $\suppsol$ in the largest entries of the provided solution.
With this new metric, we observe in \cref{fig:dcv_mass:k075_topk} that the performance of all algorithms degrades.
We also see the same phenomenon in \cref{fig:dcv_mass:k125_topk} with \SEAZ being less affected than the other algorithms and OMPR showing the best performance for $\spars < 11$.

\begin{figure}[!htb]
\centering
\begin{minipage}{.48\linewidth}
  \centering
    \includegraphics[width=\linewidth]{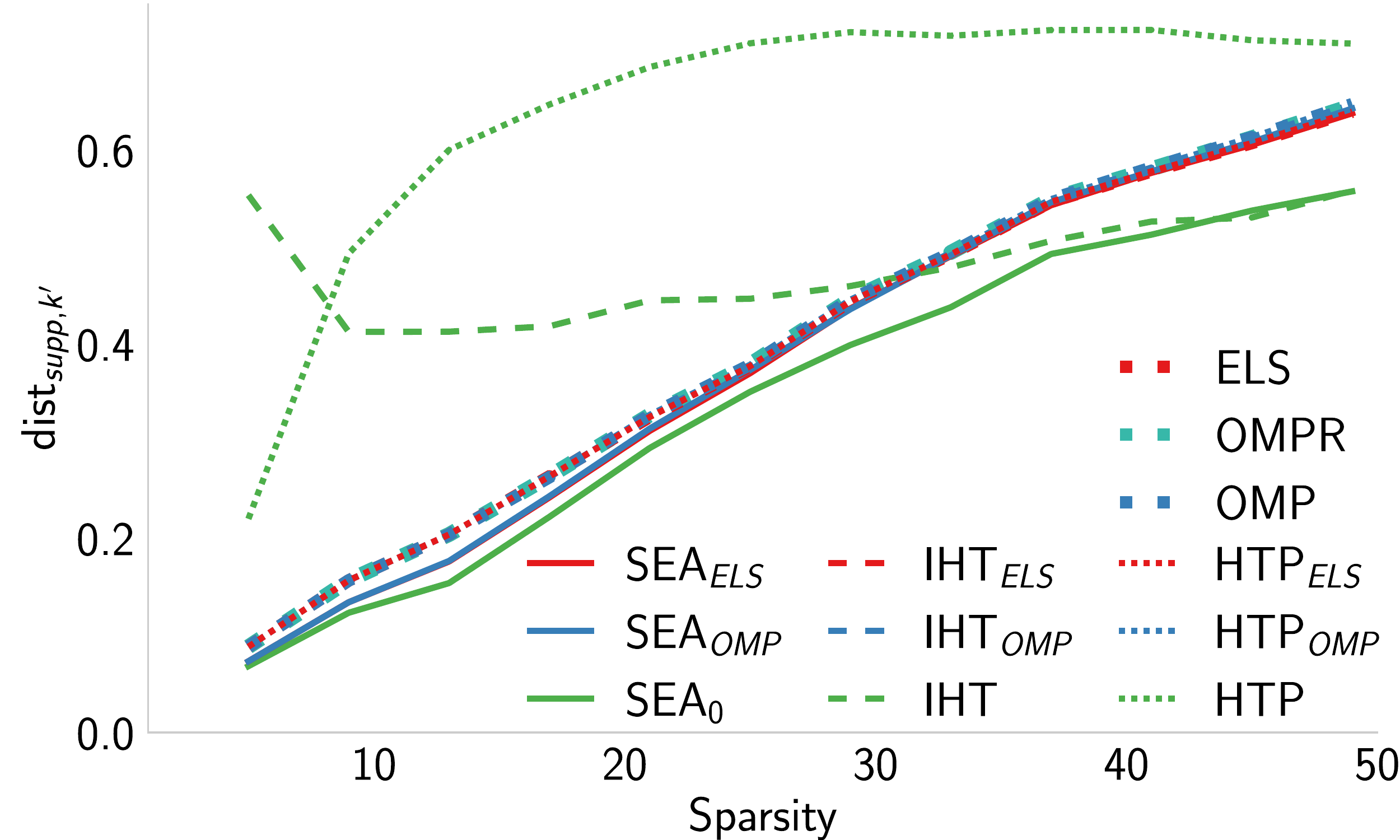}  
    \caption{Mean of support distance $\text{dist}_{\text{supp},\spars'}$ between $\suppsol$ and the support of the solutions provided by several algorithms as a function of the sparsity level $\spars$ when the sparsity is underestimated ($\spars' = 0.75\spars$).}
    \label{fig:dcv_mass:k075}
\end{minipage}%
\hspace{.02\linewidth}
\begin{minipage}{.48\linewidth}
  \centering
    \includegraphics[width=\linewidth]{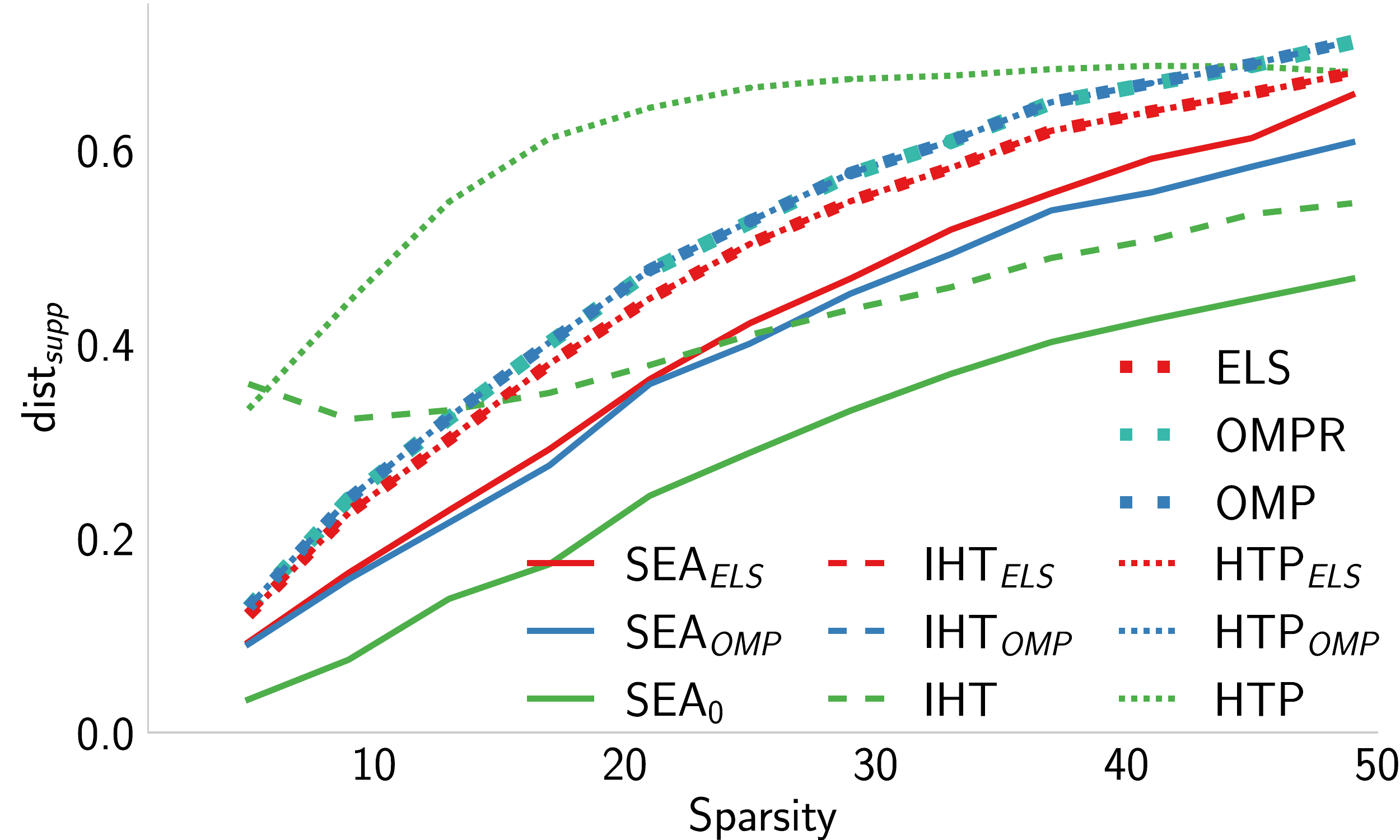}
    \caption{Mean of support distance $\text{dist}_{\text{supp}}$ (defined in \eqref{eq:supp_dist}) between $\suppsol$ and the support of the solutions provided by several algorithms as a function of the sparsity level $\spars$ when the sparsity is overestimated ($\spars' = 1.25\spars$).}
    \label{fig:dcv_mass:k125}
\end{minipage}
\end{figure}

\begin{figure}[!htb]
\centering
\begin{minipage}{.48\linewidth}
  \centering
    \includegraphics[width=\linewidth]{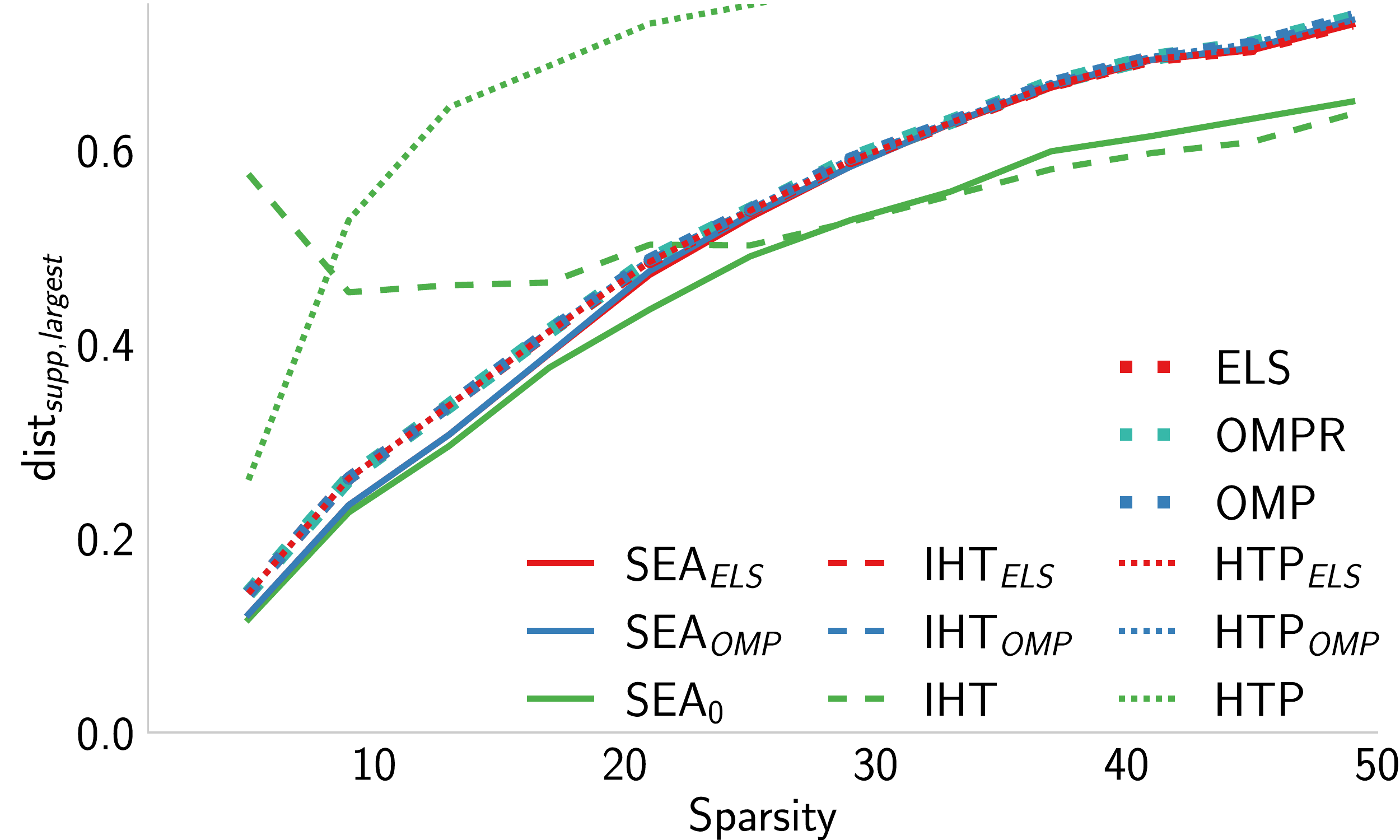}  
    \caption{Mean of dist$_\text{supp,largest}$ between $\suppsol$ and the support of the solutions provided by several algorithms as a function of the sparsity level $\spars$ when the sparsity is underestimated ($\spars' = 0.75\spars$).}
    \label{fig:dcv_mass:k075_topk}
\end{minipage}%
\hspace{.02\linewidth}
\begin{minipage}{.48\linewidth}
  \centering
    \includegraphics[width=\linewidth]{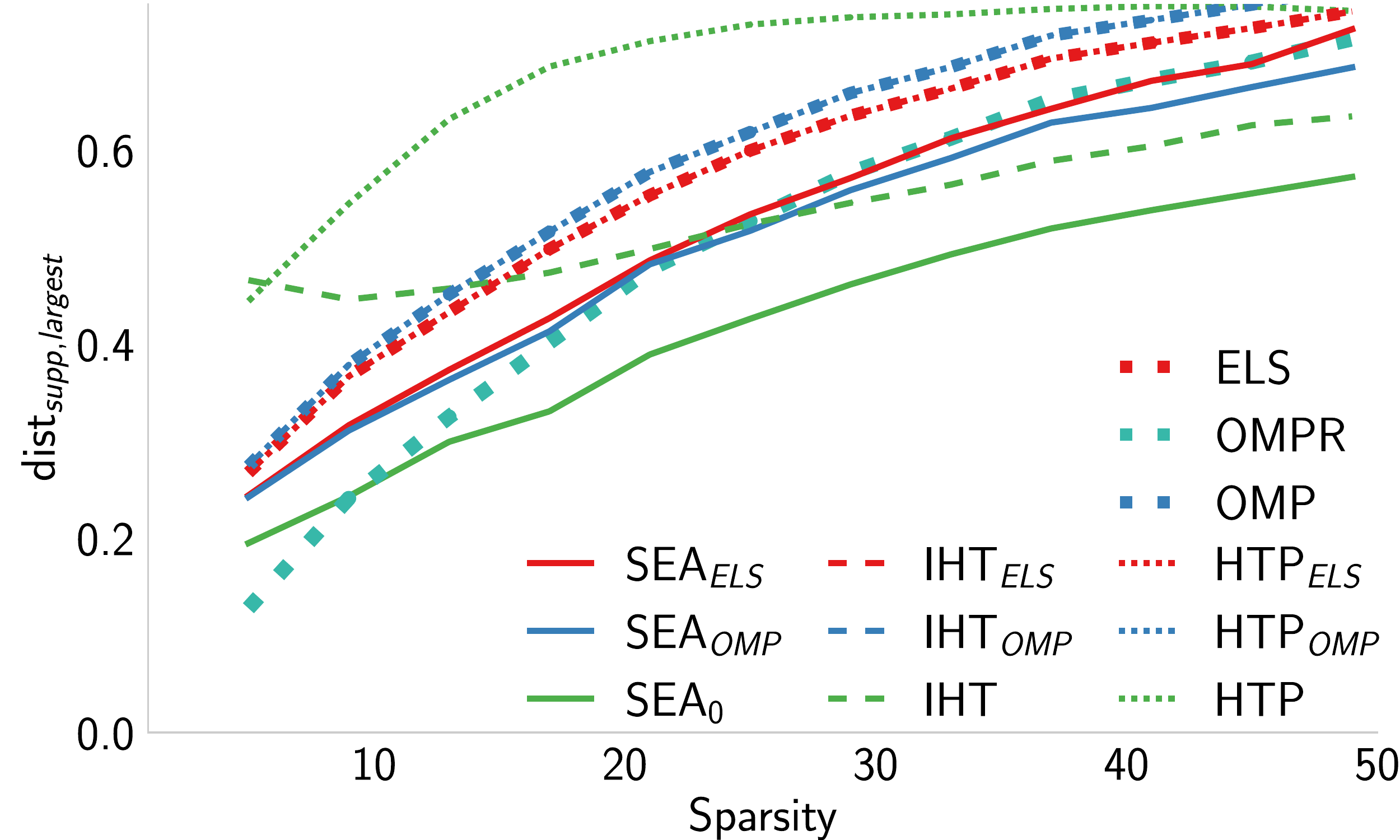}
    \caption{Mean of dist$_\text{supp,largest}$ between $\suppsol$ and the support of the solutions provided by several algorithms as a function of the sparsity level $\spars$ when the sparsity is overestimated ($\spars' = 1.25\spars$).}
    \label{fig:dcv_mass:k125_topk}
\end{minipage}
\end{figure}

\subsection{Deconvolution: Noise Robustness}
\label{app:deconv:noise_robustness}

We consider the same experiment as in \cref{deconv-sec} but using a noise $\noise$ uniformly drawn from the sphere of radius $\alpha\normd{\mat\sol}$ with $\alpha\in\{0.2, 0.3\}$ instead of $\alpha=0.1$.
Again, we set $n = 500$ and use a convolution matrix $\mat$ corresponding to a Gaussian filter with a standard deviation of 3. Every algorithm is tested on $r = 200$ noisy problems, for different $k$-sparse $\sol$, with $\spars\in\intint{1, 50}$.
The difference here is that we change the magnitude of the noise of the noisy problems.

The analogs of \cref{fig:dcv_mass} when $\alpha=0.2$ and $\alpha=0.3$ are depicted in \cref{fig:dcv_mass:noise:02,fig:dcv_mass:noise:03}.
As the noise increases, the performance of all algorithms degrades, without changing their ranking or the conclusions of the experiment.

\begin{figure}[!htb]
\centering
\begin{minipage}{.48\linewidth}
  \centering
    \includegraphics[width=\linewidth]{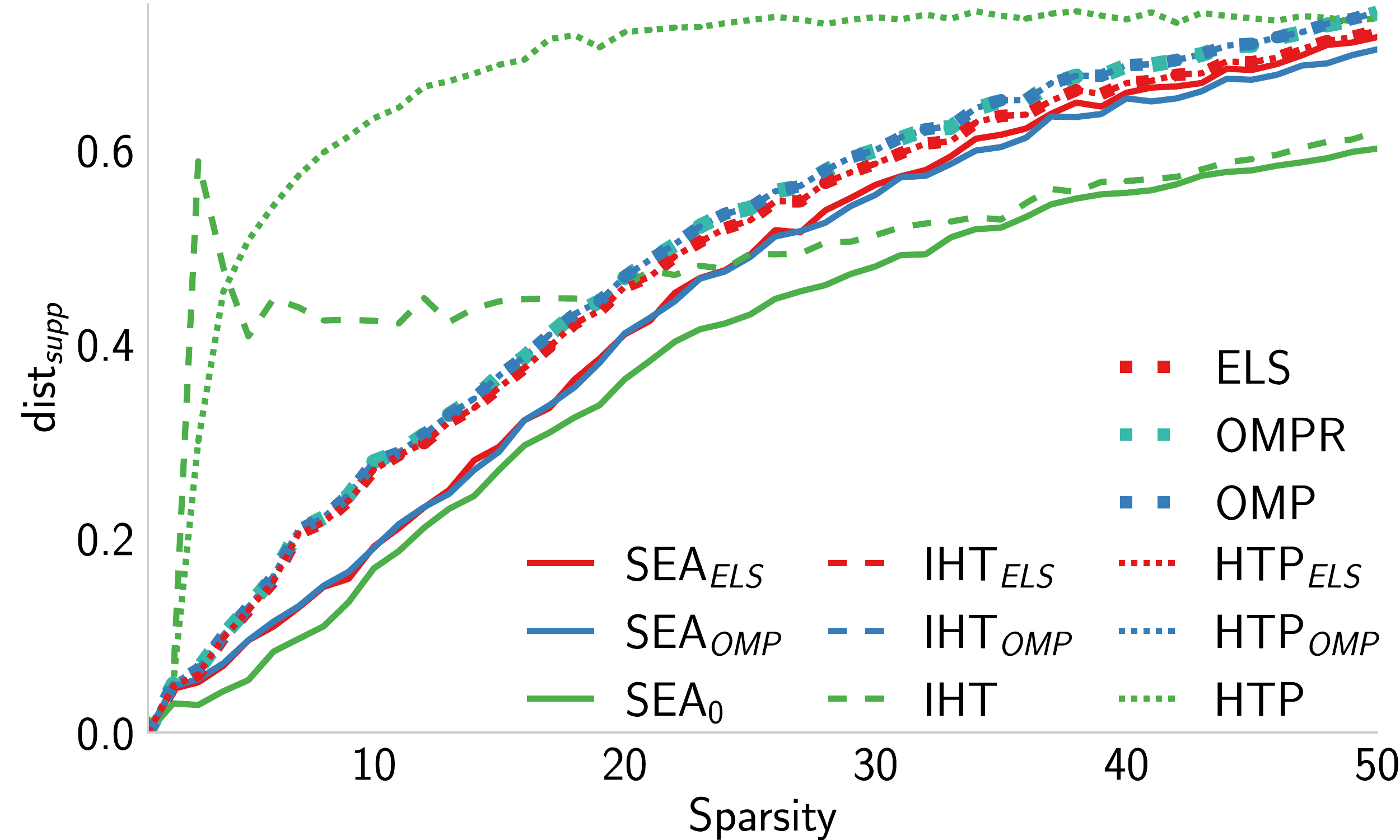}  
    \caption{Mean of support distance $\text{dist}_{\text{supp}}$ (defined in \eqref{eq:supp_dist}) between $\suppsol$ and the support of the solutions provided by several algorithms as a function of the sparsity level $\spars$ when the noise $e$ is uniformly drawn from the sphere of radius $0.2\normd{\mat\sol}$.}
    \label{fig:dcv_mass:noise:02}
\end{minipage}%
\hspace{.02\linewidth}
\begin{minipage}{.48\linewidth}
  \centering
    \includegraphics[width=\linewidth]{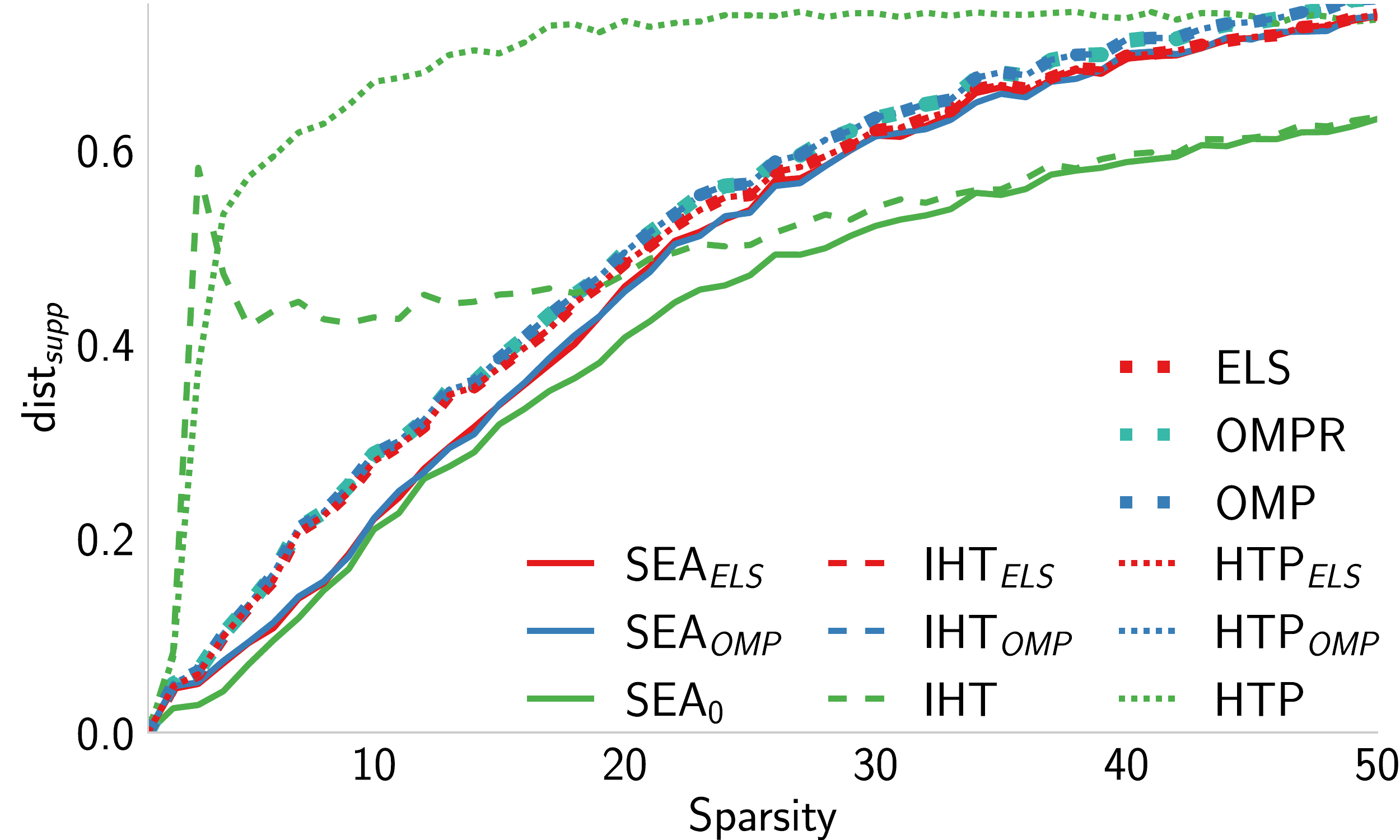}
    \caption{Mean of support distance $\text{dist}_{\text{supp}}$ (defined in \eqref{eq:supp_dist}) between $\suppsol$ and the support of the solutions provided by several algorithms as a function of the sparsity level $\spars$ when the noise $e$ is uniformly drawn from the sphere of radius $0.3\normd{\mat\sol}$.}
    \label{fig:dcv_mass:noise:03}
\end{minipage}
\end{figure}

\subsection{Deconvolution: Noise Before the Linear Transformation}
\label{app:deconv:noise_on_x}

We consider the same experiment as in \cref{deconv-sec} but add the noise $\noise$ differently.
Instead, of generating $\obs = \mat\sol+\noise$ with $\noise$ uniformly drawn from the sphere of radius $0.1\normd{\mat\sol}$, we generate $\obs=\mat(\sol+\noise)$, with $\noise$ uniformly drawn from the sphere of radius $0.1\normd{\sol}$.
Again, we set $n = 500$ and use a convolution matrix $\mat$ corresponding to a Gaussian filter with a standard deviation of 3. Every algorithm is tested on $r = 200$ noisy problems, for different $k$-sparse $\sol$, with $\spars\in\intint{1, 50}$.
The difference here is that we changed the way of adding the noise in the observations.

The analog of \cref{fig:dcv_mass} is depicted in \cref{fig:dcv_mass:noise:x}. 
The differences between these two figures are not significant.
Thus, in this configuration, \SEAZ is the algorithm reaching the lowest $\text{dist}_{\text{supp}}$ and we have the same conclusions as in \cref{deconv-sec}.

\subsection{Deconvolution: Increasing the Ratio \texorpdfstring{$\frac{\normd{\sol}}{\min_{\ii\in\suppsol}\abs{\soli}}$}{||x*||/min(|x*|)}}
\label{app:deconv:u_10}

The ratio $\frac{\normd{\sol}}{\min_{\ii\in\suppsol}\abs{\soli}}$ is of importance in conditions \ref{eq:HRIP} and \ref{eq:HSRIP}. We experimentally study whether it actually impacts the performance or whether it is an artifact of the proof.
We consider the same experiment as in \cref{deconv-sec} but with the non-zero entries of $\sol$ drawn uniformly in $\intint{-10, -1} \cup \intint{1, 10}$ instead of $\intint{-2, -1} \cup \intint{1, 2}$.
Again, we set $n = 500$ and use a convolution matrix $\mat$ corresponding to a Gaussian filter with a standard deviation of 3. Every algorithm is tested on $r = 200$ noisy problems, for different $k$-sparse $\sol$, with $\spars\in\intint{1, 50}$.

The analog of \cref{fig:dcv_mass} is depicted in \cref{fig:dcv_mass:noise:u10}. 
Here, OMP, OMPR, and ELS perform slightly better, while all the versions of SEA perform slightly worse. IHT and HTP also perform worse. However, the ranking of the algorithms remains the same. Thus, we can experimentally see that increasing the ratio$\frac{\normd{\sol}}{\min_{\ii\in\suppsol}\abs{\soli}}$ reduces SEA performance.

\begin{figure}[!htb]
\centering
\begin{minipage}{.48\linewidth}
  \centering
    \includegraphics[width=\linewidth]{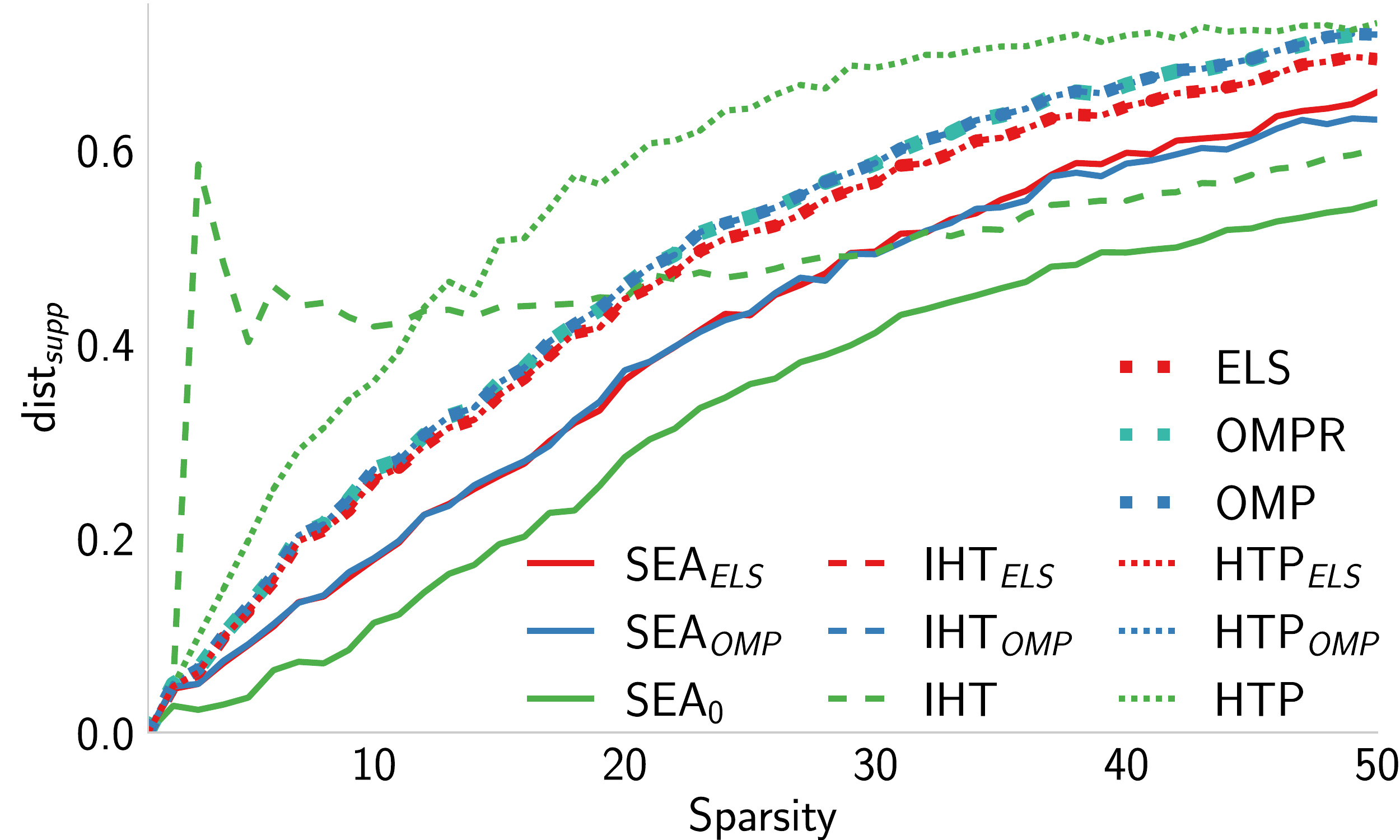}
    \caption{Mean of support distance $\text{dist}_{\text{supp}}$ (defined in \eqref{eq:supp_dist}) between $\suppsol$ and the support of the solutions provided by several algorithms as a function of the sparsity level $\spars$ when $\obs=\mat(\sol+\noise)$ with $e$ uniformly drawn from the sphere of radius $0.1\normd{\sol}$.}
    \label{fig:dcv_mass:noise:x}
\end{minipage}%
\hspace{.02\linewidth}
\begin{minipage}{.48\linewidth}
  \centering
      \includegraphics[width=\linewidth]{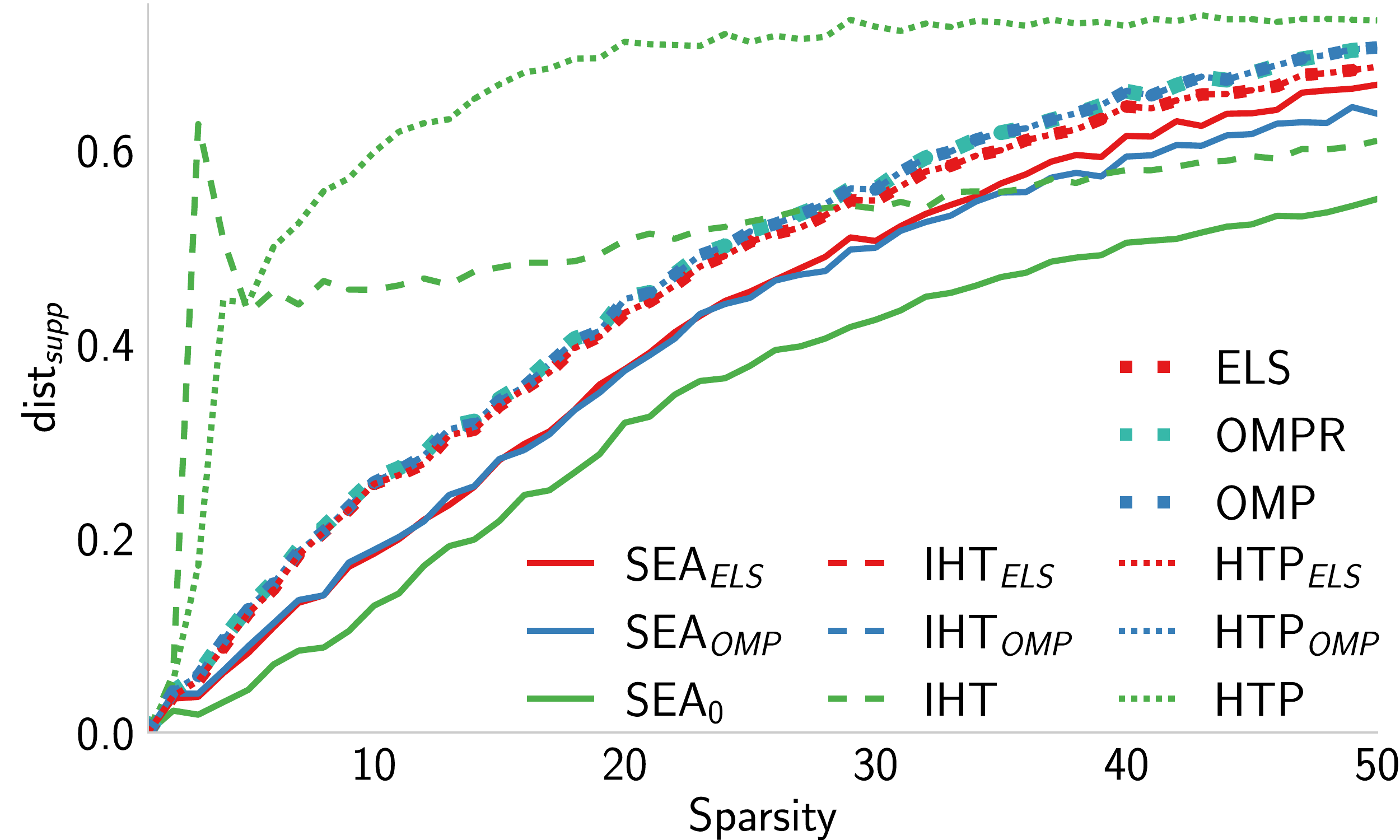}
    \caption{Mean of support distance $\text{dist}_{\text{supp}}$ (defined in \eqref{eq:supp_dist}) between $\suppsol$ and the support of the solutions provided by several algorithms as a function of the sparsity level $\spars$ when the non-zero entries of $\sol$ are drawn uniformly in $\intint{-10, -1} \cup \intint{1, 10}$.}
    \label{fig:dcv_mass:noise:u10}
\end{minipage}
\end{figure}

\subsection{Deconvolution: Step Size of IHT and HTP}
\label{app:deconv:step_size}

In \cref{expe-sec}, we arbitrarily fixed the step size $\gradstep=\frac{1.8}{L}$ where $L$ is the spectral radius of $\mat$.
In this section, we study the influence of the step size on HTP and IHT. 
We also consider the Normalized Iterative Thresholding (NIHT) \cite{blumensath2010normalized} algorithm which is based on IHT but includes an adaptive step size $\gradstep^\iter$ which depends on the iteration $\iter$.
The step size is chosen as:
\begin{equation*}
    \gradstep^\iter 
    = 
    \frac
    {\normdd{(\mat^T(\obs - \mat \sigst))_{\suppt}}}
    {\normdd{\mat((\mat^T(\obs - \mat \sigst))_{\suppt})}}.
\end{equation*}

If $\gradstep^t > 0.99\normdd{\sigs^\iter - \sig^{\iter+1}} / \normdd{\mat(\sigs^\iter - \sig^{\iter+1})}$, $\gradstep^t$ is halved until this inequality becomes unsatisfied, as explained in \cite{foucart2011hard}.

We consider the same experiment as in \cref{deconv-sec} for NIHT, IHT, and HTP with different step sizes. 
Again, we set $n = 500$ and use a convolution matrix $\mat$ corresponding to a Gaussian filter with a standard deviation of 3. Every algorithm is tested on $r = 200$ noisy problems, for different $k$-sparse $\sol$, with $\spars\in\intint{1, 50}$.
IHT and HTP were tested with a step size $\eta\in\{\frac{2^l}{L}~|~ \mbox{for }l=\intint{-3,+3}\}$.

The analog of \cref{fig:dcv_mass} is depicted in \cref{fig:dcv_mass:step_size}. 
For all considered step sizes, IHT and HTP cannot improve the solution provided by ELS and OMP when we consider the $\text{dist}_{\text{supp}}$ metric.
We reach the same conclusion for NIHT.
The ranking of HTP does not depend on the selected step size. 
The lowest $\text{dist}_{\text{supp}}$ is reached with a step size $\gradstep = \frac{8}{L}$.
Increasing the step size further did not improve the results and made HTP diverge.
IHT is more sensitive to step size variations than HTP.
The lowest $\text{dist}_{\text{supp}}$ is reached with a step size $\gradstep = \frac{2}{L}$.
Increasing the step size further made IHT diverge.
The dynamic step size of NIHT made NIHT better than \SEAOMP and \SEAELS for $\spars > 25$.
However, IHT with $\gradstep = \frac{2}{L}$ becomes better than NIHT for $\spars > 36$.
Thus, even by considering these variabilities, \SEAZ remains the best algorithm in this setting.

\begin{figure}[tbhp]
    \centering
    \includegraphics[width=0.6\linewidth]{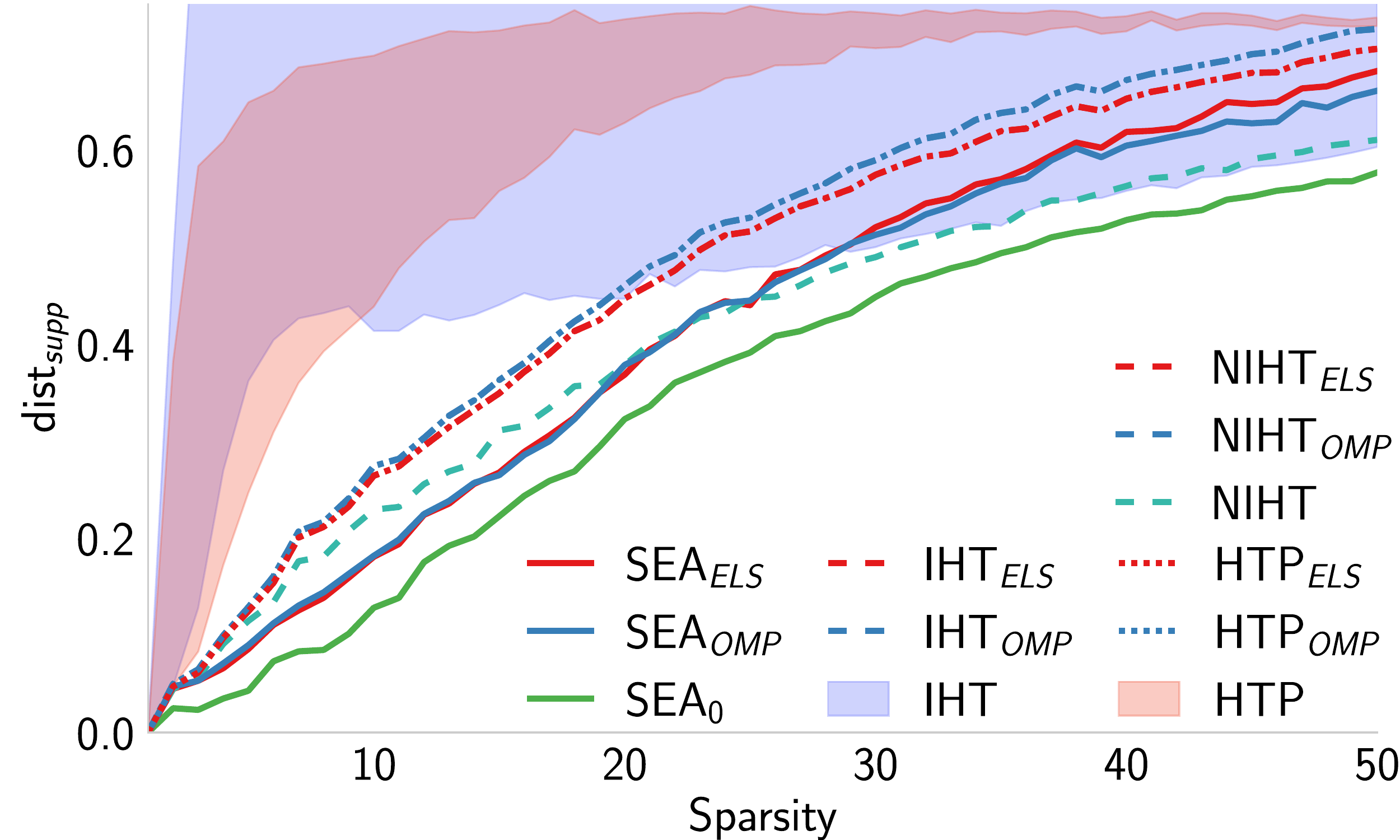}
    \caption{Mean of support distance $\text{dist}_{\text{supp}}$ (defined in \eqref{eq:supp_dist}) between $\suppsol$ and the support of the solutions provided by several algorithms as a function of the sparsity level $\spars$ with different step size values for IHT and HTP.
    The area between the highest and the lowest curve for different step sizes is displayed in blue and red for IHT and HTP.
    NIHT$_{OMP}$ and IHT$_{OMP}$ are superimposed.
    NIHT$_{ELS}$ and IHT$_{ELS}$ are superimposed.}
    \label{fig:dcv_mass:step_size}
\end{figure}

\subsection{Deconvolution: Random Search as a Pure Exploration Baseline.}
\label{app:deconv:random}

A random search to recover the correct support could be used as a baseline to assess the performance of SEA. 
However, we discarded this option because the probability of recovery is too small. 
For instance, if we draw $r=100000$ independent random supports from a uniform distribution, the probability of recovering the correct support of size $k=20$ in a signal of size $n=500$ is $1-\left(1-\frac{1}{\binom{n}{k}}\right)^r \sim 3.75 \times 10^{-31}.$ 
To establish a lower bound on support recovery performance, we conducted a random search for a number of supports equal to the maximal number of iterations ($1000$). 
The obtained results in this context were significantly  inferior to any other tested algorithm.

%% file: ML_experiments.tex
\section{Supervised Machine Learning Experiments}
\label{app:MLexp}

We describe here supervised learning experiments: they confirm that \SEAZ performs well in small dimensions, and performs better in high dimensions when combined with OMP or ELS. They also give evidence that SEA can perform very well in the presence of error/noise and when no perfect sparse vector fits the data.

\subsection{Context}\label{ann-intro-sec}

In a supervised learning setting, the rows of matrix $\mat \in \matspace$ (often denoted by $X$) are the $\sigsize$-dimensional feature vectors associated with the $\obssize$ training examples, while the related labels are in vector $\obs \in \obsspace$.
In the training phase, a sparse vector $\sigs$ (often denoted $\beta$ or $w$) is optimized to fit $\obs \approx \mat \sigs$ using an appropriate loss function.
In this context, support recovery is called model selection.

Based on the experimental setup of~\cite{pmlr-v119-axiotis20a}, we compare the training loss for different levels of sparsity, for all the algorithms, on linear regression and logistic regression tasks.
We use the preprocessed public datasets\footnote{\href{https://drive.google.com/file/d/1RDu2d46qGLI77AzliBQleSsB5WwF83TF/view}{https://drive.google.com/file/d/1RDu2d46qGLI77AzliBQleSsB5WwF83TF/view}} provided by~\cite{pmlr-v119-axiotis20a},
following the same preprocessing pipeline: we augment $\mat$ with an extra column equal to $1$ to allow a bias and normalize the columns of $\mat$.

We present results for regression problems in \cref{ann-reg-sec} and for classification problems in \cref{ann-class-sec}.

\subsection{Regression Datasets}\label{ann-reg-sec}

As we are working with real datasets without ground truth, we use the $\ell_2$ regression loss $\ell_2\_\text{loss}(\sig) = \frac{1}{2}\normdd{\mat\sig - \obs}$ for $\sig \in \sigspace$ for regression problems.

As shown in \cref{fig:cal_year}, \SEAZ, \SEAOMP and \SEAELS are at the same level as ELS on a regression dataset with $\sigsize$ small as in \texttt{comp-activ-harder}.
For the higher dimensional regression dataset as \texttt{year} (see \cref{fig:year}), \SEAZ performs poorly as $\spars$ increases, but \SEAELS can improve ELS performances and outperforms the other algorithms.

\begin{figure}[tbh]%
    \centering
    %\subfloat{{% [trim={left bottom right top},clip]
    \includegraphics[width=\textwidth]{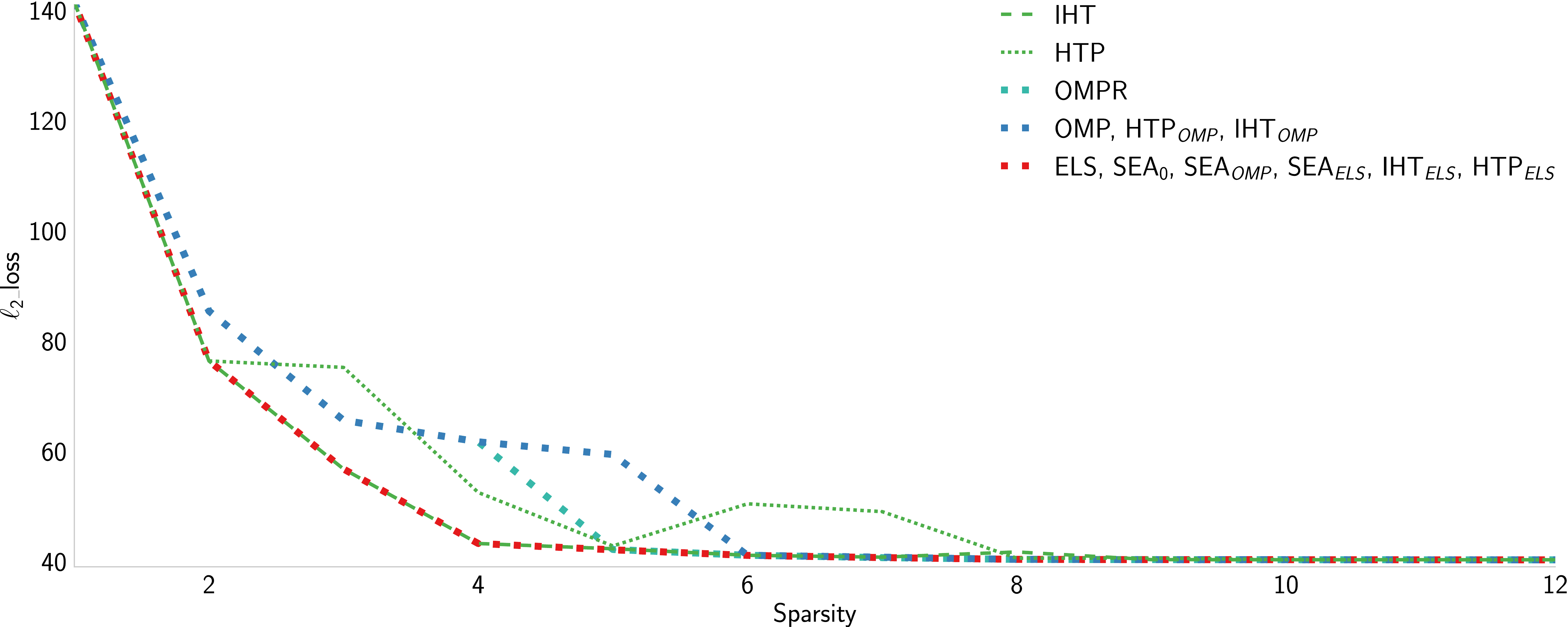} %}}%
    \caption{Performance on regression dataset \texttt{comp-activ-harder} ($\obssize=8191$ examples, $\sigsize=12$ features).}%
    \label{fig:cal_year}%
\end{figure}

\begin{figure}[tbh]%
    \centering
    %\subfloat{{% [trim={left bottom right top},clip]
    \includegraphics[width=\textwidth]{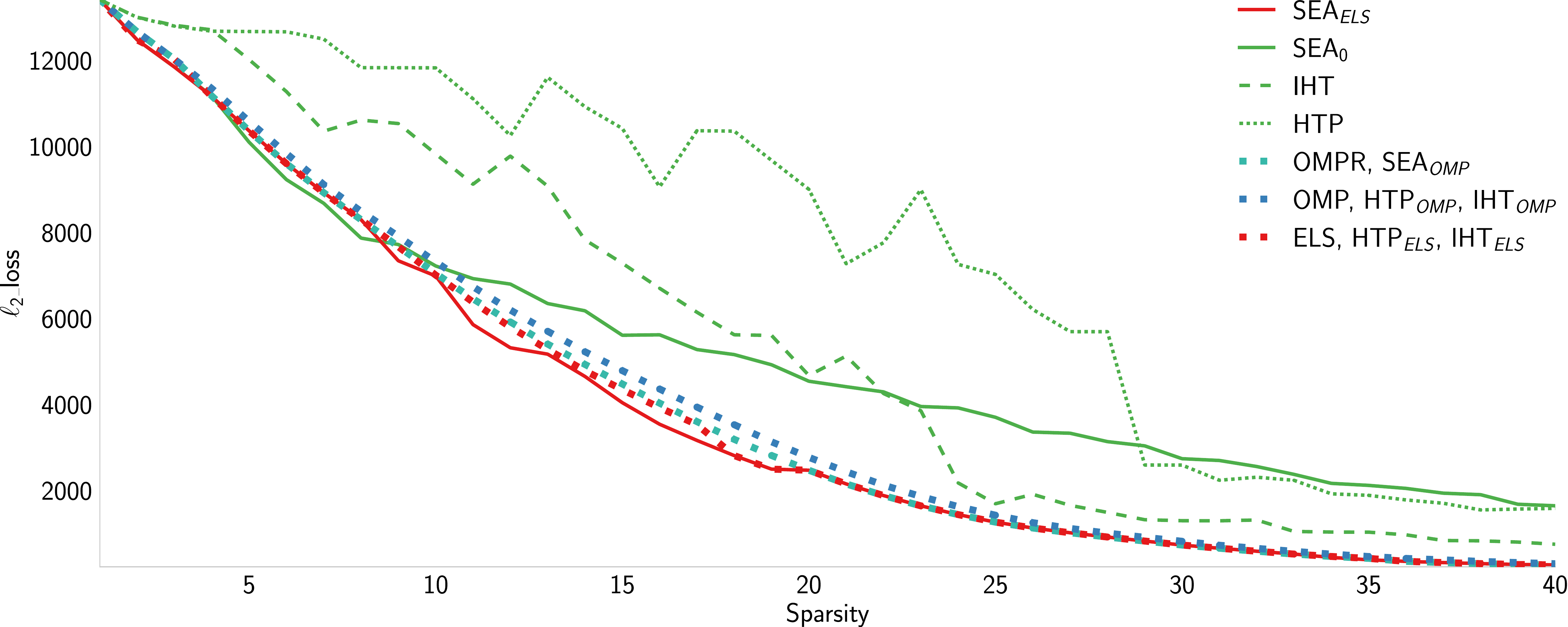} %}}%
    \caption{ Performance on regression dataset \texttt{year} ($\obssize=463715$ examples, $\sigsize=90$ features).}%
    \label{fig:year}%
\end{figure}

In a low-dimensional problem ($n$ is small) as \texttt{cal\_housing} dataset ($\obssize=20639$ examples, $\sigsize=8$ features) in \cref{fig:comp}, we see that \SEAOMP and \SEAELS perform better than ELS, and \SEAZ outperform them for a sparsity $k \in \intint{5, 8}$.
It is worth mentioning that HTP obtains good performances for this dataset.

The same experiment is reported on \cref{fig:slice}, but for the dataset \texttt{slice} ($\obssize=53500$ examples, $\sigsize=384$ features).  This is an intermediate-dimensional problem. \cref{fig:slice} shows that \SEAZ obtains slightly worse results than \SEAELS, \SEAOMP and ELS.
The non-decreasing curve of HTP comes from its support estimation technique.
Since the coherence of the \texttt{slice} dataset is 1, HTP selects highly correlated features and fails to correct this mistake.

\begin{figure}[ht!]
    \centering
    \includegraphics[width=\linewidth]{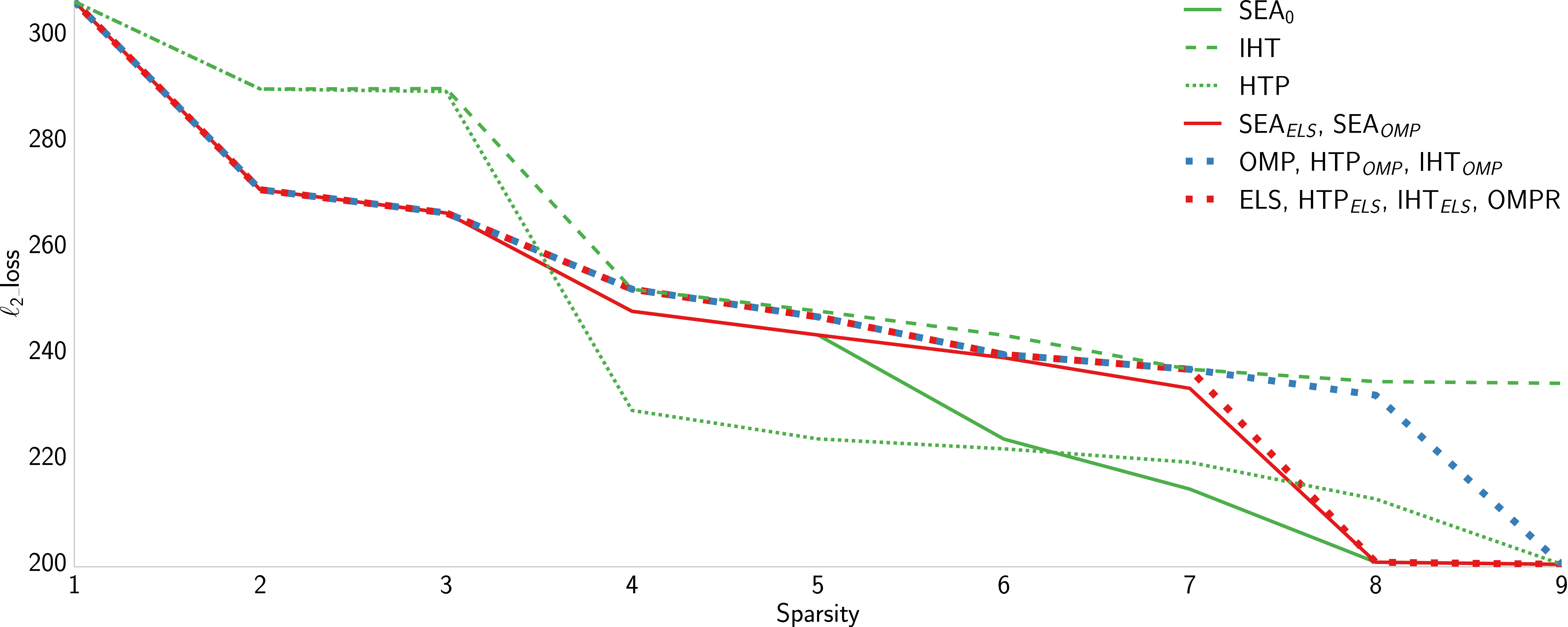}
    \caption{Performance on the regression dataset \texttt{cal\_housing} ($\obssize=20639$ examples, $\sigsize=8$ features).}
    \label{fig:comp}
\end{figure}

\begin{figure}[ht!]
    \centering
    \includegraphics[width=\linewidth]{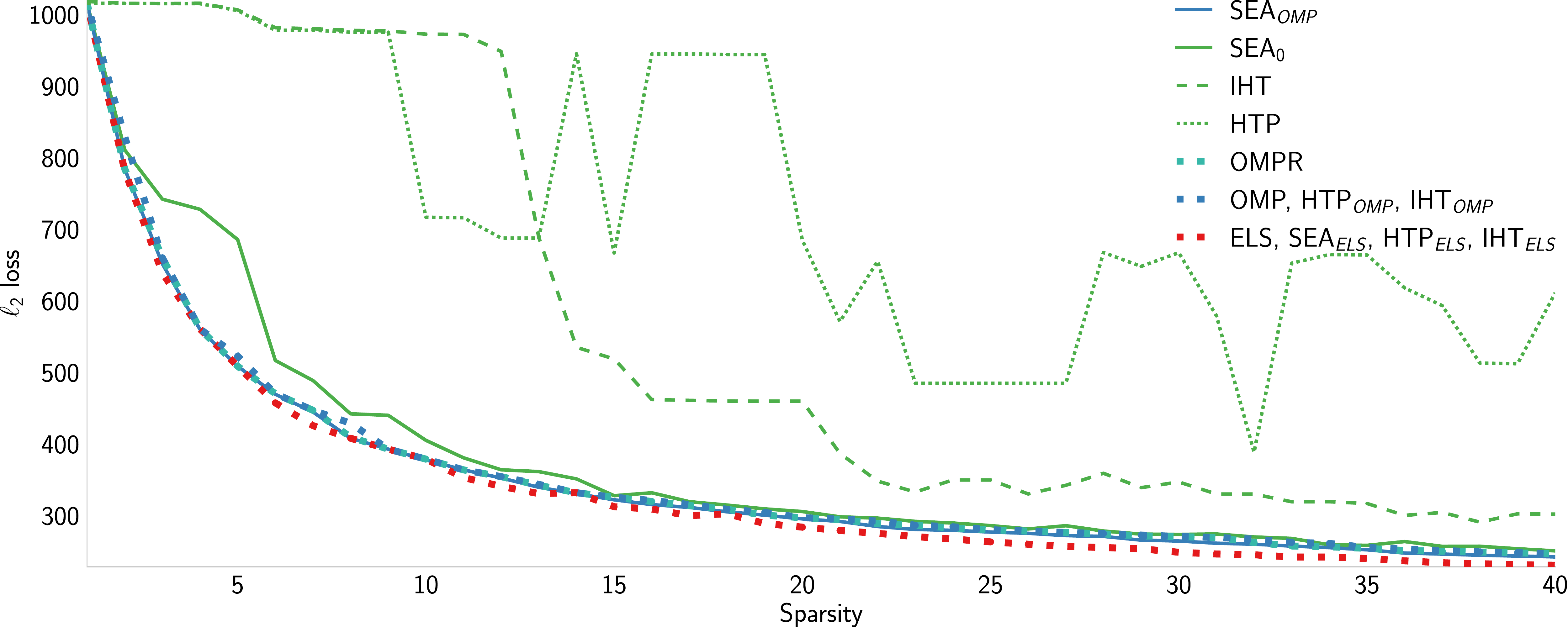}
    \caption{Performance on the regression dataset \texttt{slice} ($\obssize=53500$ examples, $\sigsize=384$ features).}
    \label{fig:slice}
\end{figure}

\subsection{Classification Datasets}\label{ann-class-sec}

In these experiments, we consider the logistic regression loss defined by $$\text{log}\_\text{loss}(\sig) = \sum_{\ii = 1}^{\obssize}\left(-\obs_\ii\log(\sigmoid((\mat\sig)_\ii)) - (1 - \obs_\ii)\log(1 - \sigmoid((\mat\sig)_\ii))\right),$$ 
where $\sigmoid(t) = \frac{1}{1 + e^{-t}}$ is the sigmoid function.

We need to adapt SEA to this new loss. In \cref{alg:SEA}, line \ref{line:SEA:update_x} is replaced by $\sigs^\iter = \underset{\substack{\sig \in \sigspace \\ \SUPP{\sig} \subseteq \supp^\iter }}{\argmin} \text{log}\_\text{loss}(\sig)$ and line \ref{line:SEA:update_explo} is replaced by $\explo^\iterp = \explo^\iter - \eta \nabla \text{log}\_\text{loss}(\sig^t)$. Similar adaptations are performed on the other algorithms.

The loss $\text{log}\_\text{loss}(\sig)$, for the \texttt{letter} dataset ($\obssize=20000$ examples, $\sigsize=16$ features), for all $k\in\intint{1,12}$ and for all algorithms is depicted in \cref{fig:letter}. We depict the same curves obtained for the \texttt{ijcnn1} dataset ($\obssize=24995$ examples,  $\sigsize=22$ features) in \cref{fig:ijcnn}.
These two last figures show that \SEAZ, \SEAOMP and \SEAELS achieve similar performances to ELS.

\begin{figure}[ht!]
    \centering
    \includegraphics[width=\linewidth]{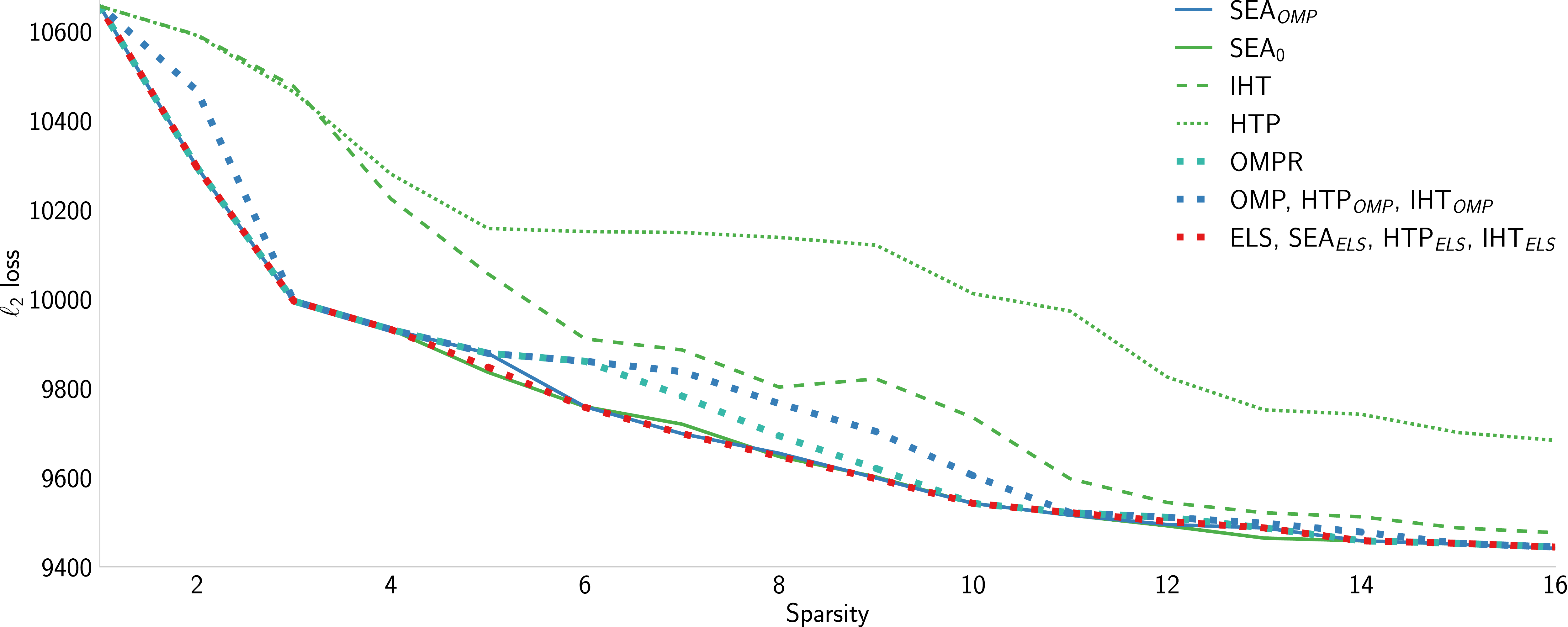}
    \caption{Performance on the classification dataset \texttt{letter} ($\obssize=20000$ examples, $\sigsize=16$ features).}
    \label{fig:letter}
\end{figure}
\vspace*{-0.4cm}
\begin{figure}[ht!]
    \centering
    \includegraphics[width=\linewidth]{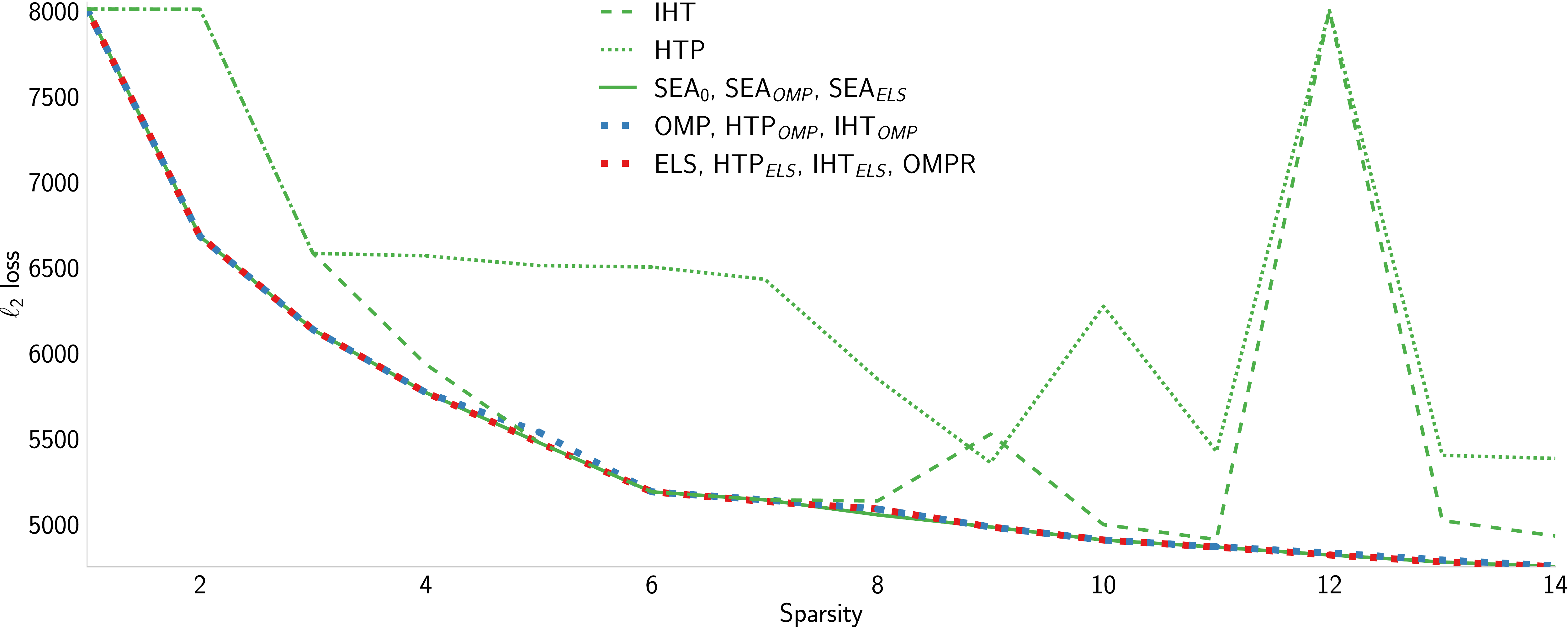}
    \caption{Performance on the classification dataset \texttt{ijcnn1} ($\obssize=24995$ examples,  $\sigsize=22$ features).}
    \label{fig:ijcnn}
\end{figure}